\documentclass[11pt,letterpaper,english]{article}
\usepackage{fullpage}
\usepackage{setspace}
\usepackage[T1]{fontenc}
\usepackage{lmodern}
\usepackage{newtxtext} 

\usepackage[small, compact]{titlesec}

\usepackage{abstract}

\usepackage{parskip} 

\usepackage{booktabs}
\usepackage{graphicx} 
\usepackage{caption}
\usepackage{wrapfig}
\usepackage{pgfplots}
\pgfplotsset{compat=1.17}
\usepackage{multirow}

\usepackage{enumerate}
\usepackage{comment}
\usepackage{enumitem}
\usepackage[makeroom]{cancel}
\usepackage{algorithm}
\usepackage{algpseudocode}
\algdef{SE}[DOWHILE]{Do}{doWhile}{\algorithmicdo}[1]{\algorithmicwhile\ #1}
\algtext*{EndIf}
\algtext*{EndFor}
\algtext*{EndWhile}
\algtext*{EndFunction}
\algtext*{EndProcedure}
\usepackage{amsmath}
\usepackage{amsthm}
\usepackage{amssymb}
\setitemize{itemsep=0pt}

\usepackage[style=authoryear, sorting=none, backend=biber, natbib=true]{biblatex}

\usepackage{bbm}
\usepackage{soul}
\usepackage{xcolor}
\usepackage{longtable}
\usepackage{array}
\newcolumntype{L}[1]{>{\raggedright\arraybackslash}p{#1}}

\usepackage{soul,xcolor}
\usepackage[edges]{forest}
\usepackage{graphicx}
\setstcolor{blue}
\usepackage{makecell}
\usepackage[bottom]{footmisc}

\usepackage{amsmath}
\usepackage{amsthm}
\usepackage{amssymb}
\usepackage{tikz}
\usepackage{xcolor}
\usetikzlibrary{arrows}
\usetikzlibrary{positioning}

\usepackage{hyperref}
\usepackage{todonotes}

\usepackage{hyperref}
\usepackage{bm}
\usepackage{subcaption}
\allowdisplaybreaks
\usepackage{makecell} 

\usepackage{graphicx}
\usepackage{pbox}

\def\TV{\mathrm{TV}}

\theoremstyle{plain}
\newtheorem{theorem}{Theorem}[section]
\newtheorem{proposition}[theorem]{Proposition}
\newtheorem{lemma}[theorem]{Lemma}
\newtheorem{corollary}[theorem]{Corollary}
\theoremstyle{definition}
\newtheorem{definition}{Definition}
\newtheorem{assumption}{Assumption}
\theoremstyle{remark}
\newtheorem{remark}{Remark}

\newenvironment{itemize*}%
{\begin{itemize}[leftmargin=*,topsep=0pt]%
		\setlength{\itemsep}{0pt}%
		\setlength{\parskip}{0pt}}%
	{\end{itemize}}

\newenvironment{enumerate*}%
{\begin{enumerate}[leftmargin=*,topsep=0pt]%
		\setlength{\itemsep}{0pt}%
		\setlength{\parskip}{0pt}}%
	{\end{enumerate}}

\newcommand{\negpar}[1][-1em]{%
  \ifvmode\else\par\fi
  {\parindent=#1\leavevmode}\ignorespaces
}

\title{Reasonably reasoning AI agents avoid game-theoretic failures in zero-shot, provably}

\author{Enoch Hyunwook Kang\thanks{ehwkang@uw.edu}
\\
Foster School of Business, University of Washington
}
\date{}  

\begin{document}
	
	\maketitle

\begin{abstract}
As autonomous AI agents increasingly mediate online platform markets, a fundamental question emerges: do these markets generate stable strategic outcomes? In repeated strategic environments, the Nash equilibrium provides a natural benchmark for this stability. However, empirical evidence on off-the-shelf LLM agents is mixed, leaving it unclear whether independently deployed agents can converge to equilibrium behavior without explicit strategic post-training. 
In this paper, we provide an affirmative answer. Extending the Bayesian learning literature in theoretical economics, we prove that AI agents, acting as Bayesian posterior samplers rather than expected utility maximizers, are guaranteed to eventually become weakly close to a Nash equilibrium in infinitely repeated games. We further extend this analysis to settings in which stage payoffs are unknown ex ante, and agents observe only their privately realized stochastic payoffs, and obtain the same convergence guarantees. Finally, we empirically evaluate these theoretical implications across five repeated-game environments, ranging from the Prisoner's Dilemma to marketing promotion games. Taken together, our findings suggest that strategic stability in AI-mediated markets can emerge from the intrinsic reasoning and learning properties of modern AI agents, without the need for unrealistic universal fine-tuning.
\end{abstract}

\newpage
\section{Introduction}

A fundamental transition is underway in online platform marketplaces: autonomous AI agents are beginning to act on behalf of consumers in search, evaluation, and purchasing decisions \citep{gaarlandt2025ai, shahidi2025coasean}. Rather than humans directly browsing rankings, comparing listings, and clicking through interfaces, AI agents can parse webpages or interact through APIs to evaluate products and transact. As a result, an increasing share of economically relevant platform activity may soon be mediated not by human attention directly, but by interaction among autonomous agents operating within platform-designed environments \citep{rothschild2025agentic}.

This shift makes a theoretical question newly urgent. In markets mediated by human participants, economists often analyze outcomes through the lens of stable strategic behavior emerging from repeated interaction. As autonomous AI agents increasingly act on behalf of users, it becomes unclear whether markets populated by such agents will exhibit the same kind of predictable strategic stability. The core issue, therefore, is not only whether an individual AI agent performs well, but whether many interacting AI agents collectively produce stable and predictable market outcomes. In digital markets, such interaction arises naturally in pricing, promotion, bidding, negotiation, matching, and recommendation environments, where each agent's payoff depends on how other agents behave over time \citep{bianchi2024well, guo2024large, lopez2025can, li2024aligning, zhu2025automated, bansal2025magentic}. In such settings, what matters for prediction is \textit{whether the system settles into a stable strategic pattern of play} \citep{carichon2025coming}.

Nash equilibrium in repeated games is a natural benchmark for such stability. In repeated strategic environments, Nash equilibrium provides a starting point for analyzing long-run behavior, incentives, and strategic predictability \citep{abreu1988structure, dal2011evolution, dal2018determinants, proto2019intelligence}. In this sense, asking the question of \textit{whether AI agents converge toward Nash-like play in repeated games} is the first step toward understanding strategic outcomes in AI-mediated markets; failure to approach this benchmark can be regarded as a \emph{game-theoretic failure}. That is, if agents do not reach behavior that is at least immune to profitable unilateral deviations in repeated games, then the resulting play lacks even the most basic form of strategic stability and is difficult to interpret as strategically coherent. This makes Nash convergence in repeated games a useful baseline question before considering richer models of dynamic interaction, such as Markov Perfect Equilibrium \citep{aguirregabiria2021dynamic}.

This question is not merely theoretical. Recent work by \citet{calvano2020artificial} and \citet{fish2024algorithmic}, together with related empirical studies of algorithmic interaction, suggests that autonomous algorithmic and AI systems can generate strategically consequential repeated-game behavior in economically important environments. Pricing algorithms can sustain supra-competitive outcomes without explicit communication, rapid reactive pricing technologies can elevate prices even in competitive equilibrium, and real-world adoption of repeated algorithmic pricing has been associated with higher margins in concentrated markets \citep{assad2024algorithmic}. 

On the other hand, empirical evaluations of LLMs reveal that widely used, off-the-shelf AI models (e.g., GPT, Claude, Gemini, Kimi, DeepSeek) as AI agents frequently fail to exhibit predicted equilibrium behavior in strategic interactions and often resort to brittle heuristics or produce inconsistent policies \citep{guo2024embodiedllmagentslearn,huang2024far, hua2024game, buscemi2025fairgame}. In practice, simply prompting standard AI models to engage in repeated games often yields strategies that diverge significantly from rational, equilibrium-based play predicted by classical game theory, although some successes have been reported \citep{akata2025playing}. Such brittleness and inconsistency raise concerns about deploying AI agents in societally crucial domains that require reliable strategic decision-making. Accordingly, whether off-the-shelf reasoning LLM agents can be guaranteed to converge to a Nash equilibrium in repeated strategic interaction stably remains an open problem.

One prominent approach to addressing this problem is the use of targeted, universal post-training procedures \citep{li2024aligning,duque2024advantage}, i.e., requiring all AI agents to undergo the same additional fine-tuning alignment steps to instill strategically desirable behavior.  However, relying on the universal deployment of such fine-tuning approaches across diverse, independently developed AI agents competing with one another is impractical. Consequently, there is a compelling need for assurance that off-the-shelf AI agents with some ``reasonable'' reasoning capabilities autonomously adapt their strategies and reach a stable equilibrium.
This critical observation motivates the central research question explored in this paper:
\begin{center}
\textit{Can off-the-shelf reasoning AI agents stably converge to Nash equilibrium in repeated strategic interaction without post-training?}
\end{center}

In this paper, we show that reasoning LLM-based AI agents are guaranteed to evolve toward Nash continuation play along realized play paths, without relying on specialized explicit post-training procedures.

At a high level, our argument builds on a simple idea from the economics of repeated interaction: if agents can learn from observed play what their opponents are likely to do and respond well to those learned forecasts, their behavior should move toward Nash-like play over time. The classical \textit{Bayesian learning} literature in theoretical economics formalizes this logic by showing that when agents Bayesian-update beliefs from observed histories and \emph{exactly} best respond to those beliefs, equilibrium behavior eventually emerges along realized play paths \citep{kalai1993rational, norman2022possibility}. 

We take this logic as our starting point but adapt it to the setting of off-the-shelf reasoning LLM agents. The central difficulty is that such agents are not expected-utility maximizers \citep{yamin2026llms, ge2026mind}. Rather than best-responding, i.e., deterministically choosing the exact optimal action against their current beliefs, they behave more naturally as stochastic posterior samplers \citep{arumugam2025toward}. Our first theoretical contribution is therefore to show that, under mild and realistic assumptions, posterior-sampling LLM agents nevertheless achieve \emph{asymptotic} best-response learning along the realized play path. Our second contribution is to show that the classical Bayesian learning literature's logic extends to this weaker asymptotic notion, so that exact best responses need not be assumed to recover eventual proximity to a Nash equilibrium of the continuation game.

To formalize this argument, we isolate two reasoning capabilities and call agents satisfying them ``reasonably reasoning'' agents: \textit{Bayesian updating} and \textit{asymptotic best-response learning}. By Bayesian updating, we mean the capacity to learn opponents' strategies from observed interaction histories and thereby form increasingly accurate beliefs about future play. By asymptotic best-response learning, we mean that, relative to those inferred beliefs, the agent's continuation behavior eventually becomes approximately optimal along the realized path. Combined with recent findings that LLMs behave as Bayesian in-context learners under stationary, repeated settings \citep{coda2023meta, lu2024emergent, cahyawijaya2024llms, xie2021explanation, wang2023large, wakayama2025context, falck2024context}, this yields our conclusion that reasoning LLM agents can satisfy the `reasonably reasoning' conditions and therefore eventually exhibit Nash equilibrium along realized paths in infinitely repeated interactions.

We also develop a weaker but practically important benchmark for \emph{myopic} reasoning. If an agent only predicts opponents' \emph{next} actions and best responds period by period, then one should not expect convergence to rich repeated-game equilibria supported by continuation incentives. Nevertheless, we show that such one-step predict--then--act reasoning is still sufficient for eventual convergence to a \emph{stage-game} $\varepsilon$-Nash equilibrium along the realized path. This result provides a useful theoretical benchmark for simpler prompting-based procedures, such as SCoT \citep{akata2025playing}, and clarifies what myopic reasoning can and cannot achieve relative to continuation-level planning.

Beyond the benchmark with common-knowledge stage payoffs, we also consider the practically relevant case in which payoffs are not known to agents ex ante and each agent observes only its own privately realized stochastic payoffs. To model this, we can modify the agents' posterior sampling to not only sample an opponent-strategy hypothesis, but also sample a hypothesis for the agent's own mean payoff matrix (equivalently, its own payoff kernel within the known noise family). Under the analogous learning conditions, posterior sampling recovers the same asymptotic on-path $\varepsilon$-best-response property and therefore inherits the zero-shot Nash convergence guarantees.

These theoretical results generate three distinct empirical predictions. First, because one-step predict--then--act reasoning is sufficient only for \emph{stage-game} equilibrium convergence, simple myopic procedures (e.g., SCoT) should often succeed when the objective is merely to reach some stage-game Nash action. Second, because nontrivial repeated-game equilibria depend on continuation values, myopic reasoning should generally fail to sustain such paths, whereas agents that infer opponent strategies and evaluate continuation plans should succeed. Third, when payoffs are not known ex ante and must be learned from noisy private observations, the same separation should persist, although under a more demanding informational problem. 

To examine whether these implications arise in practice for a concrete off-the-shelf model, we instantiate Qwen 3.5-27B \citep{qwen3.5}, a small open reasoning model, with three different decision rules: the \textit{Base} model agent, the myopic \textit{SCoT} agent, and the reasonably reasoning (RR) agent implemented via posterior-sampling best response (PS-BR). We then study their behavior in symmetric self-play across five repeated-game environments, ranging from the Prisoner's Dilemma to marketing promotion games. The resulting simulations confirm the predictions of the theory developed above.

This paper is structured as follows. Section~\ref{sec:related} discusses related work. Section~\ref{sec:setup} introduces the repeated-game setup and the operational belief notation used in the main text; the more formal predictive-representative construction is deferred to Appendix~\ref{app:belief_repr}. Section~\ref{sec:rr} defines reasonably reasoning agents and relates their Bayesian and best-response learning properties to in-context and test-time inference in language models. Section~\ref{sec:zero_shot} presents the main zero-shot Nash convergence results. Section~\ref{sec:unknown_payoffs} extends the analysis to unknown, stochastic payoffs. Section~\ref{sec:experiments} provides empirical evidence for the theory.

\section{Related works}\label{sec:related}

\paragraph{Bayesian Learning.} The theoretical analysis of reasonably reasoning agents is based largely on the Bayesian learning literature. Bayesian learning in repeated games is defined by a fundamental tension between the ability to logically learn opponents' strategies and the ability to respond to them optimally. The foundational possibility result in \cite{kalai1993rational} showed that if players' prior beliefs contain a "grain of truth" (absolute continuity) regarding the true distribution of play, then standard Bayesian updating guarantees that their predictions will eventually converge to the truth, thereby naturally culminating in a Nash equilibrium. However,  \citet{nachbar1997prediction, nachbar2005beliefs} subsequently proved a negative result: requiring players to simultaneously maintain this grain of truth and perfectly best-respond across all possible counterfactual game histories leads to a mathematical contradiction, as the infinite sets of learnable strategies and optimizing strategies are often mutually singular. \cite{norman2022possibility} resolved this tension by introducing ``optimizing learnability'', the crucial insight that agents do not need to perfectly learn unreached counterfactuals; they only need to accurately predict and best-respond along the realized path of play. Nonetheless, Norman identified that a stubborn impossibility persists in a specific class of games called MM* games, where adversarial payoff geometries prevent learning and optimization from coexisting even on-path.  

This paper systematically navigates these classic boundaries to guarantee zero-shot Nash convergence for LLM agents. We actively employ \cite{kalai1993rational} grain of truth (Assumption \ref{ass:grain_of_truth}) to guarantee predictive accuracy via the classic merging of opinions, and avoid \citet{nachbar1997prediction, nachbar2005beliefs}'s impossibility by formally adopting the on-path relaxation and non-MM* in \citet{norman2022possibility}.

\paragraph{Strategic capabilities of LLM agents.}
As LLMs are increasingly deployed as interactive agents, a growing literature studies whether LLMs behave strategically in canonical games, emphasizing preference representation, belief formation, and (approximate) best responses rather than taking equilibrium play for granted \citep{sun2025game, jia2025llm}. In one-shot normal-form, bargaining, and negotiation tasks, off-the-shelf models often follow plausible but context-sensitive heuristics: behavior can depart from equilibrium predictions and change markedly under small framing or instruction variations \citep{guo2023gptgame, fan2023rational, hua2024game}. Strategic performance can improve with model scale and reasoning scaffolds, but the remaining variance across prompts and settings is substantial \citep{kader2024emergence}.

These issues become more acute under repeated games, where payoffs depend on stable, history-contingent policies. Multi-agent evaluation benchmarks report large cross-model and cross-game heterogeneity and frequent non-equilibrium dynamics, especially in coordination and social-dilemma regimes \citep{mao2023alympics, duan2024gtbench, huang2024far}. Controlled repeated-game experiments similarly find that cooperation/reciprocity can emerge, but is fragile to opponent choice and to seemingly minor prompt or protocol changes \citep{akata2025playing, palatsi2024nicer, willis2025systems}. In market-style repeated settings, recent work further documents collusive or supra-competitive outcomes among LLM agents and highlights sensitivity to communication opportunities and wording choices \citep{fish2024algorithmic, agrawal2025doubleauction}.

Overall, existing results demonstrate meaningful strategic adaptation but do not provide general, zero-shot guarantees that heterogeneous, independently deployed off-the-shelf agents will converge to predictable equilibrium behavior. Our paper targets this gap by identifying two basic theory-of-mind capabilities, Bayesian updating of opponent strategies and asymptotic best-response learning, and proving that, under mild conditions, they imply Nash continuation play along realized paths in repeated games, without requiring explicit post-training or cross-agent coordination.

\paragraph{LLM agents as Bayesian in-context learners.}
A growing body of work links \emph{in-context learning} (ICL), i.e., test-time adaptation that conditions prior history on a prompt without parameter updates, to Bayesian inference over latent task hypotheses. 
In stylized transformer meta-learning settings, \citet{xie2021explanation} argue that transformers trained over a task distribution can implement an implicit Bayesian update and produce posterior-predictive behavior from in-context data; related analyses formalize ICL as (approximate) Bayesian model averaging and study how this view depends on model parameterization and drives generalization \citep{zhang2023and}. 
Moving beyond specific constructions, \citet{falck2024context} propose a martingale-based perspective that yields diagnostics and theoretical criteria for when an in-context learner’s predictive sequence is consistent with Bayesian updating, while \citet{wakayama2025context} provide a broader meta-learning theory in which ICL is provably equivalent to Bayesian inference with accompanying generalization guarantees. 
Empirically, LLMs also exhibit \emph{meta}-adaptation across tasks presented in-context \citep{coda2023meta}, and several abilities that appear “emergent” under scaling can be substantially attributed to improved ICL mechanisms \citep{lu2024emergent}. 
Complementing these viewpoints, \citet{wang2023large} model LLM ICL through a latent-variable lens, where demonstrations act as evidence about an unobserved task variable—clarifying why behavior can be highly sensitive to the specific examples and their ordering—and related results document few-shot in-context adaptation even in low-resource language learning regimes \citep{cahyawijaya2024llms}. 
For agentic and repeated-interaction settings, these Bayesian-ICL perspectives motivate modeling an LLM agent’s use of the interaction transcript as maintaining and updating a posterior over opponent strategies/types; autoregressive generation can then be interpreted as sampling-based decision-making from the induced posterior \citep{zhang2024posterior,welleck2024decoding}, providing a concrete bridge between in-context learning and belief-based strategic behavior.

\paragraph{Expected utility maximization and best response.}
Standard learning-in-games analyses often assume agents compute an exact best response to their posterior at every history \citep{kalai1993rational,norman2022possibility}. This is a poor behavioral model for off-the-shelf LLM agents, whose actions are induced by stochastic decoding and thus implement a distribution over choices rather than a deterministic maximization of expected utility. In probabilistic decision tasks, \citet{yamin2026llms} find systematic belief--decision incoherence, suggesting that elicited probabilities should not be treated as beliefs that the model then perfectly best-responds to. In risky-choice experiments, \citet{ge2026mind} similarly document substantial departures from expected-utility maximization and large sensitivity to prompting/model type, with behavior better described as noisy sampling. \citet{arumugam2025toward} argues that LLMs naturally implement posterior sampling. These results motivate replacing exact best response with a weaker, sampling-compatible notion, e.g., posterior-sampling policies, which are shown to achieve \emph{asymptotic} best-response performance along the realized path.

\section{Setup}\label{sec:setup}
\subsection{Infinitely repeated game}

We study interaction among a finite set of agents $I=\{1,2, \ldots, N\}$ in an
infinitely repeated (discounted) game with perfect monitoring of actions and
common-knowledge stage payoffs. We define the game as the tuple
\[
\mathcal{G}=\left(I,\left\{A_i\right\}_{i \in I},\left\{u_i\right\}_{i \in I},\left\{\lambda_i\right\}_{i \in I}\right)
\]
where:
\begin{itemize}[leftmargin=*]
 \setlength\itemsep{0em}
    \item $I$ is the finite set of AI agents
    \item $A_i$ is the finite set of actions available to agent $i$
    \item $A=\prod_{i \in I} A_i$ is the joint action space, where a joint action profile at round $t$ is denoted
    $a^t=\left(a_1^t, \ldots, a_{|I|}^t\right) \in A$. ($a_i^t$ indicates the action of agent $i$ at round $t$)
    \item $u_i: A \rightarrow [0,1]$ is agent $i$'s (known) stage-game payoff function
    \item $\lambda_i \in(0,1)$ is the private discount factor used by agent $i$ to value future payoffs.
\end{itemize}

At each round $t=1,2,\dots$, each agent $i$ simultaneously chooses an action $a_i^t \in A_i$,
forming a joint action profile $a^t \in A$, which is publicly observed. Agent $i$ then receives
the stage payoff 
\begin{equation}
u_i(a^t) \in [0,1] \label{eq:true_payoffs}
\end{equation}
These stage payoffs induce a standard infinitely repeated game with perfect monitoring of actions.

In defining the payoffs $\left\{u_i\right\}_{i \in I}$, we restrict attention to the following standard assumption from the Bayesian learning literature \citep{norman2022possibility}. Intuitively, this excludes games without a pure-strategy equilibrium, e.g., rock-scissors-paper; more importantly for our purposes, it excludes the payoff environments in which realized play need not admit a nearby Nash interpretation. Since the object of the paper is equilibrium prediction and analysis of realized play, this restriction isolates the part of the game space where that question is well posed.

\begin{assumption}[Non-MM$^\star$ game \citep{norman2022possibility}]\label{ass:nonmmstar}
Consider the infinitely repeated game induced by the true stage payoffs $\left\{u_i\right\}_{i \in I}$ in equation \eqref{eq:true_payoffs}. For each player $i$, define the stage-game minmax payoff and pure-action maxmin payoff as

$$
\varphi_i:=\min _{\sigma_{-i} \in \Delta\left(A_{-i}\right)} \max _{\sigma_i \in \Delta\left(A_i\right)} u_i\left(\sigma_i, \sigma_{-i}\right), \quad \Phi_i^{\star}:=\max _{a_i \in A_i} \min _{a_{-i} \in \mathrm{BR}_{-i}\left(a_i\right)} u_i\left(a_i, a_{-i}\right),
$$
where $\mathrm{BR}_{-i}\left(a_i\right)$ denotes the set of opponents' (joint) best responses to $a_i$ in the stage game. We call that the stage game is $\mathrm{MM}^{\star}$ if $\Phi_i^{\star}<\varphi_i$ for every $i$. We assume the stage game is not $\mathrm{MM}^{\star}$ (equivalently, $\Phi_i^{\star} \geq \varphi_i$ holds for some $i$).
\end{assumption}

\subsection{Strategy}

We define the joint action history at round $t$ as
$
h^t=\left(a^1, a^2, \ldots, a^{t-1}\right),
$
and
$$
H^t=\left\{\left(a^1, a^2, \ldots, a^{t-1}\right): a^s \in A \text { for } s \le t-1\right\}.
$$
Let $H^0:=\{\emptyset\}$ denote the empty history.
Denote the complete set of possible histories as $H=\bigcup_{t \geq 0} H^t$. (Throughout this paper, we allow AI agents' strategies to have bounded memory; See Appendix \ref{sec:bounded_memory}.)

\begin{definition}[Strategy]
A strategy for agent $i$ is a function
$$
f_i: H \rightarrow \Delta\left(A_i\right),
$$
which maps every joint action history to a distribution over agent $i$'s actions $A_i$. 
\end{definition}

Let $\mathcal{F}_i$ denote the space of all strategies of agent $i$. A strategy profile is a tuple $f=\left(f_1, \ldots, f_N\right) \in \mathcal{F}=\prod_{i \in I} \mathcal{F}_i$. 
Let $H^{\infty}$ denote the space of infinite play paths, i.e.,
$$
H^{\infty}=\left\{\left(a^1, a^2, \ldots\right): a^t \in A \text { for all } t \in \mathbb{N}\right\}.
$$
\begin{definition}[Play-path distribution]\label{def:path_distribution}
A strategy profile $f=(f_1,\ldots,f_N)\in\mathcal F$ induces a unique probability distribution
$\mu^f$ over $H^{\infty}$ (the \emph{play-path distribution}), defined on cylinder sets by
\[
\mu^f\left(C\left(a^1, \ldots, a^t\right)\right):=\prod_{s=1}^t \prod_{i \in I} f_i\left(h^s\right)\left(a_i^s\right),
\]
where $h^s=(a^1,\ldots,a^{s-1})$ and $C(h):=\{z\in H^\infty: z=(h,\ldots)\}$.
By Kolmogorov's extension theorem \citep{durrett2019probability}, these finite-dimensional probabilities define a unique
probability measure $\mu^f$ on $(H^\infty,\mathcal B)$, where $\mathcal B$ is the product
$\sigma$-algebra.
\end{definition}

For the upcoming discussions, we fix some notations. Given that we fix a history $h^t$, for any continuation profile $g$ (i.e., a profile that specifies play after histories
extending $h^t$), let $\mu^{g}_{h^t}$ denote the induced distribution on $H^\infty$ over the future
joint-action sequence $(a^t,a^{t+1},\ldots)$ when play starts at history $h^t$ and follows $g$ thereafter.
Formally, we identify the tail $(a^t,a^{t+1},\ldots)$ with $y \in H^\infty$ by setting $y^1=a^t$,
$y^2=a^{t+1}$, and so on, and regard $\mu^{g}_{h^t}$ as a measure on this reindexed space.
For a full profile $g \in \mathcal{F}$, we write $\mu^{g}_{h^t}$ for the continuation distribution induced
by its restriction $g|_{h^t}$. If $\mu^g(C(h^t))>0$, then $\mu^{g}_{h^t}$ coincides with the conditional
distribution $\mu^g(\cdot\mid h^t)$.

\subsection{Beliefs}\label{sec:beliefs}

Each agent $i$ holds a prior $\mu_i^0$ over opponents' strategy profiles $\mathcal F_{-i}$ and updates it by Bayes' rule as public histories are observed. For the main text, the relevant object is the induced \emph{posterior predictive continuation law}: after history $h^t$, the agent needs a forecast of future opponents' play, not a full existence argument for how that forecast is represented.

Given any own strategy $g_i\in\mathcal F_i$ and any belief $\mu_i$ over opponents' strategies, let
\[
P_i^{\mu_i,g_i}(E):= \int_{\mathcal F_{-i}} \mu^{(g_i,f_{-i})}(E)\, d\mu_i(f_{-i}),
\qquad\text{for measurable }E\subseteq H^\infty.
\]
denote the predictive play-path distribution induced by $\mu_i$. We write $P_i^{0,g_i}:=P_i^{\mu_i^0,g_i}$ for the prior predictive distribution. At any history $h^t$ where Bayes' rule is defined, the posterior $\mu_i^t(\cdot\mid h^t)$ induces an analogous posterior predictive continuation law.

To keep the main text operational, we write $f_{-i}^{i,t}$ for any continuation model that reproduces player $i$'s posterior predictive continuation law for the purpose of continuation-value calculations, and $f_{-i}^{i}$ for the analogous model associated with the prior predictive law. The standard existence and selection details for these representative continuation models are deferred to Appendix~\ref{app:belief_repr}.

\subsection{Subjective utility and Nash equilibrium}

\paragraph{Subjective Expected Utility.}
An agent evaluates the optimality of a continuation strategy based on their subjective beliefs at a given history. Fix a history $h^t$ and let $\sigma_i \in \mathcal{F}_i(h^t)$ be a continuation strategy for agent $i$ from $h^t$ onward. For any opponents' continuation profile $g_{-i}$, denote by $\mu^{(\sigma_i, g_{-i})}_{h^t}$ the induced distribution over future play paths when play starts at $h^t$ and follows $(\sigma_i,g_{-i})$ thereafter.

Following the standard literature \citep{kalai1993subjective}, we define the belief-explicit \textit{subjective expected utility} of playing $\sigma_i$ starting at $h^t$ as
\begin{equation}
V_i(\sigma_i \mid h^t; g_{-i})
=
\mathbb{E}_{y \sim \mu^{(\sigma_i, g_{-i})}_{h^t}}
\left[ (1-\lambda_i) \sum_{k=0}^{\infty} \lambda_i^k u_i(y^{k+1}) \right],
\label{eq:belief_explicit_value}
\end{equation}
where $y = (y^1, y^2, \dots)$ represents the future path of joint actions relative to time $t$, with $y^{k+1}$ denoting the joint action at step $k+1$ of this future path (i.e., at absolute time $t+k$).

When $g_{-i}=f_{-i}^{i,t}$, we write
\begin{equation}
V_i(\sigma_i \mid h^t)
:= V_i(\sigma_i \mid h^t; f_{-i}^{i,t}).
\label{eq:subjective_value}
\end{equation}

For any belief about opponents' continuation play $g_{-i}$ at history $h^t$, we define the set of $\varepsilon$-best-response continuation strategies for agent $i$ at $h^t$ as
\begin{align}
&\mathrm{BR}_i^\varepsilon(g_{-i} \mid h^t) \notag
=
\left\{
\sigma_i \in \mathcal{F}_i(h^t) :
V_i(\sigma_i \mid h^t; g_{-i})
\ge
\sup_{\sigma_i' \in \mathcal{F}_i(h^t)} V_i(\sigma_i' \mid h^t; g_{-i}) - \varepsilon
\right\}.    
\end{align}

\paragraph{Nash equilibrium.}

The true performance of a strategy profile $f \in \mathcal{F}$ for agent $i$ is given by:
$$
U_i(f)=\mathbb{E}_{z \sim \mu^f}\left[\left(1-\lambda_i\right) \sum_{t=1}^{\infty} \lambda_i^{t-1} u_i\left(z^t\right)\right],
$$
where $z^t \in A$ is the joint action at round $t$, and $\lambda_i \in(0,1)$ is agent $i$'s discount factor. The factor $(1-\lambda_i)$ is a normalization ensuring that $U_i(f) \in[0,1]$ whenever $u_i(a) \in[0,1]$ for all $a \in A$.

\begin{definition}[$\varepsilon$-Nash equilibrium]
    A strategy profile $f=\left(f_1, \ldots, f_N\right) \in \mathcal{F}$ is an $\varepsilon$-Nash equilibrium if, for every agent $i \in I$,
$$
U_i(f) \geq \sup _{f_i^{\prime} \in \mathcal{F}_i} U_i\left(f_i^{\prime}, f_{-i}\right)-\varepsilon .
$$
\end{definition}

\section{Reasonably Reasoning Agents}\label{sec:rr}

As discussed earlier, one of the key ideas of this work is that reasoning LLM-based AI agents are fundamentally ``reasonably reasoning'' agents. In this section, we formally define the class of reasonably reasoning agents, and then demonstrate why reasoning-LLM agents are naturally reasonably reasoning agents. The definition isolates two ingredients:
(i) Bayesian updating and
(ii) an \emph{on-path, asymptotic} notion of $\varepsilon$-consistency.

\begin{definition}[Reasonably Reasoning Agent]\label{def:rr}
Fix a repeated game and a strategy profile $f=(f_i)_{i\in I}$ generating the objective play-path distribution $\mu^f$ (Definition~\ref{def:path_distribution}).
Player $i$ is a \emph{Reasonably Reasoning} (RR) agent if the following hold.
\begin{itemize}[leftmargin=*]
\item \textbf{Bayesian updating:} Player $i$ has a prior $\mu_i^0$ over opponents' strategy profiles $\mathcal F_{-i}$ and forms posteriors $(\mu_i^t)_{t\ge 0}$ by Bayes' rule. Let $f_{-i}^{i,t}$ denote any continuation model reproducing player $i$'s posterior predictive continuation law at history $h^t$ (Section~\ref{sec:beliefs}), so that for every continuation strategy $\sigma_i$,
\[
V_i(\sigma_i\mid h^t)=V_i(\sigma_i\mid h^t; f_{-i}^{i,t}).
\]

\item \textbf{Asymptotic $\varepsilon$-consistency on-path:} For every $\varepsilon>0$,
\[
\mu^f\!\left(\left\{z:\exists\,T_i(z,\varepsilon)<\infty\ \text{s.t.}\ \forall\,t\ge T_i(z,\varepsilon),\ 
f_i\big|_{h^t(z)}\in \mathrm{BR}_i^\varepsilon\!\big(f_{-i}^{i,t}\big|_{h^t(z)}\mid h^t(z)\big)\right\}\right)=1.
\]
\end{itemize}
\end{definition}

Intuitively, the ``Bayesian updating'' condition ensures that agents update their strategic beliefs coherently given observations. The ``asymptotic $\varepsilon$-consistency'' condition captures the idea that after a possibly long initial stumbling phase, agents eventually learn to play (approximately) optimal continuation strategies relative to their own beliefs along the realized path of play. It generalizes Norman's $\varepsilon$-consistency \citep{norman2022possibility}, which requires $\varepsilon$-best responding at \emph{all} times (not only eventually) on a full-measure set of paths. This generalization is critical, as LLM-based AI agents are not expected-utility maximizers but rather posterior belief samplers \citep{arumugam2025toward, yamin2026llms, ge2026mind}.

The Bayesian-learning component of Definition~\ref{def:rr} is intentionally predictive. In repeated interaction, the object an agent needs for planning is the posterior predictive law over future play conditional on the realized history, not a perfectly identified label for the opponent's entire strategy. The notation $f_{-i}^{i,t}$ from Section~\ref{sec:beliefs} is just a convenient continuation-model wrapper for that predictive law.

This distinction matters because repeated-game strategies are history-contingent reaction rules. Realized actions change over time, but learning is about refining uncertainty over the underlying rule and, operationally, over its implications for continuation play. Continuation values are computed by integrating payoffs against the induced posterior predictive law. The formal representative construction is standard and recorded in Appendix~\ref{app:belief_repr}; the main text only needs the predictive object itself.

To guarantee that Bayesian updating is well-defined and that predictive beliefs can converge to the truth on-path, we impose the standard grain-of-truth condition.

\begin{assumption}[Grain of truth \citep{kalai1993rational}]\label{ass:grain_of_truth}
For each player $i$, the objective play-path distribution $\mu^f$ is absolutely continuous with respect to $i$'s prior predictive distribution under $f_i$, i.e.\ $\mu^f \ll P_i^{0,f_i}$.
Equivalently, any event that player $i$ assigns zero probability under their prior predictive model has zero probability under the true play distribution induced by $f$.
\end{assumption}

Under Assumption~\ref{ass:grain_of_truth}, classical merging-of-opinions results
\citep{blackwell1962merging} imply that player $i$'s posterior predictive continuation beliefs
become accurate along $\mu^f$-almost every realized play path.
We formalize this later by showing that absolute continuity implies strong path prediction
(Lemma~\ref{lem:absolute_cont_predict}).

\subsection{LLM agents are Bayesian updating agents}

The Bayesian-learning abstraction above matches what we can operationally observe
from LLM agents: \emph{history-conditioned predictive distributions}.
An LLM, when prompted with the game rules and the realized interaction history, induces a
conditional distribution over next tokens, which can be arranged to correspond to a
distribution over a discrete label for an opponent strategy.

This ``as if Bayesian'' framing is appropriate for two reasons.
First, the technical apparatus already works at the level of history-conditioned predictive distributions. The notation $f_{-i}^{i,t}$ is only a representative continuation model used to evaluate continuation values under those forecasts, and the formal representative construction is deferred to Appendix~\ref{app:belief_repr}.
Second, recent theory and empirical evidence indicate that AI agents, most of which are auto-regressive LLM models, can implement Bayesian or
approximately Bayesian in-context learning in repeated, stationary environments
\citep{xie2021explanation, zhang2023and, falck2024context, wakayama2025context}.
Interpreting the prompt history as data and the model’s induced distribution as a posterior
predictive therefore provides a principled bridge between LLM behavior and Bayesian-learning
agents in repeated games.

Finally, Assumption~\ref{ass:grain_of_truth} should be understood as a modeling requirement on
the LLM agent’s \emph{support}: the agent’s predictive model should not rule out (assign zero
probability to) events that can actually occur under the true interaction induced by $f$.
In practice, this corresponds to ensuring that the agent’s elicited beliefs are sufficiently expressive so that the true on-path behavior is not ruled out.
Mild smoothing can be useful in implementation, but the finite-menu argument below does not require an
exogenous tremble.

\subsection{LLM agents achieve asymptotic $\varepsilon$-consistency}\label{sec:ps_asymptotic_consistency}

In LLM agents, actions are produced through \emph{stochastic decoding} rather than through a deterministic argmax computation. Empirically, this introduces substantial decision noise and breaks the literal expected-utility-maximization view of behavior \citep{yamin2026llms,ge2026mind}. A better approximation is that the agent samples a latent strategic hypothesis from its current posterior and then reasons conditional on that sample \citep{arumugam2025toward,cai2024active}. The goal of this subsection is to show that such sampling-based behavior can still satisfy the asymptotic $\varepsilon$-consistency part of Definition~\ref{def:rr}.

The big picture has three steps. First, we formalize the LLM decision rule as \emph{posterior-sampling best response} (PS-BR): infer one opponent-strategy hypothesis from the current posterior, then optimize against that sampled hypothesis. Second, we quantify the loss from sampling a single hypothesis rather than optimizing against the full posterior predictive continuation. Third, we impose a menu-level learnability condition saying that, along the realized path, the posterior over the finite retained menu eventually puts almost all mass on labels that are continuation-payoff-equivalent to the truth from the reached history onward. Once these three pieces are combined, the stochasticity of LLM decoding remains real at any fixed date, but it becomes asymptotically irrelevant for on-path optimality.

\paragraph{LLMs naturally induce posterior-sampling best response (PS-BR).}
Reasoning LLM-based AI agents are naturally scaffolded as ``infer, then respond'' systems \citep{zhou2023far, riemer2024position}. PS-BR isolates exactly that logic: sample one opponent hypothesis from the current posterior, then choose a best response to that sampled hypothesis. This preserves the sampling-based flavor of LLM behavior while still making the optimization step explicit.

\begin{definition}[Posterior sampling best response (PS-BR)]\label{def:psbr}
Fix player $i$ and a history $h^t$.
Given posterior $\mu_i^t(\cdot\mid h^t)$ over opponents' strategy profiles, PS-BR chooses a continuation strategy by:
\begin{enumerate}
\item sampling $\tilde f_{-i}\sim \mu_i^t(\cdot\mid h^t)$;
\item playing any best response $\sigma_i\in \mathrm{BR}_i(\tilde f_{-i}\mid h^t)$ in the continuation game after $h^t$.
\end{enumerate}
Denote the resulting (randomized) continuation strategy by $\sigma^{\mathrm{PS}}_{i,t}(\cdot\mid h^t)$.
\end{definition}


\begin{remark}
    
In our experiments, step 1 is implemented by querying the model to output one opponent-strategy label from an allowed finite menu, and step 2 is implemented by evaluating a finite set of candidate self-strategies against that sampled label via roll-out and selecting the value-maximizing candidate. See Appendix~\ref{app:impl_strategy_psbr} for the implementation details. The theoretical point is not the engineering scaffold itself, but the structure it captures: LLM agents need not act like deterministic expected-utility maximizers; they can act like posterior samplers whose draws are then optimized against.
\end{remark}

The first question is what is lost by best-responding to one sampled label rather than to the full posterior predictive continuation. The next lemma gives a simple worst-case bound: the gap is controlled by how often two independent posterior samples disagree.

\begin{lemma}[PS-BR is a $D_i^t$-best response]\label{lem:psbr_gap}
Fix player $i$ and a history $h^t$. Suppose $\mu_i^t(\cdot\mid h^t)$ is supported on a finite set $\mathcal S_{-i}$ and write
\[
p_t(g_{-i}) := \mu_i^t(g_{-i}\mid h^t),\qquad g_{-i}\in\mathcal S_{-i}.
\]
Define the posterior collision complement
\[
D_i^t(h^t)\ :=\ 1-\sum_{g_{-i}\in\mathcal S_{-i}} p_t(g_{-i})^2
\ =\ \Pr_{\tilde g,\tilde g'\,\sim\,\mu_i^t(\cdot\mid h^t)}\!\big[\tilde g\neq \tilde g'\big].
\]
Let $\sigma^{\mathrm{PS}}_{i,t}(\cdot\mid h^t)$ be PS-BR at $h^t$. Then
\[
V_i(\sigma^{\mathrm{PS}}_{i,t}\mid h^t)\ \ge\ \sup_{\sigma_i} V_i(\sigma_i\mid h^t)\ -\ D_i^t(h^t).
\]
Equivalently, $\sigma^{\mathrm{PS}}_{i,t}(\cdot\mid h^t)\in \mathrm{BR}_i^{D_i^t(h^t)}\!\big(f_{-i}^{i,t}\big|_{h^t}\mid h^t\big)$.
\end{lemma}

Lemma~\ref{lem:psbr_gap} gives the first key idea. PS-BR is already close to a best response whenever posterior dispersion over strategically different labels is small. If the posterior were concentrated on one label, the gap would be zero. But the exact posterior point concentration is stronger than we need. For Nash convergence, the only posterior disagreement that matters is disagreement that changes the continuation optimization problem from the reached history.

This observation is what lets us weaken the learning requirement. In general repeated games, full posterior concentration over an unrestricted strategy space is too much to ask and is closely related to the classic impossibility phenomena in \citet{nachbar1997prediction,nachbar2005beliefs}. We therefore ask only for strategy learnability \emph{up to continuation-payoff equivalence}. The key point is that PS-BR samples labels, but Nash convergence only cares about the continuation optimization problem induced by those labels. If two retained labels generate exactly the same continuation value for every continuation strategy available to player $i$ after a reached history, then PS-BR is strategically indifferent between them: sampling one or the other leaves the objective being optimized unchanged.

\begin{definition}[On-path continuation-payoff equivalence]
\label{def:cont_equiv}
Fix player $i$ and a history $h^t$. For opponents' continuation profiles
$g_{-i},g'_{-i}\in\mathcal F_{-i}$, write
\[
g_{-i}\sim_i^{h^t} g'_{-i}
\]
if
\[
V_i(\sigma_i\mid h^t; g_{-i})
=
V_i(\sigma_i\mid h^t; g'_{-i})
\qquad
\text{for every } \sigma_i\in \mathcal F_i(h^t).
\]
Equivalently, $g_{-i}$ and $g'_{-i}$ induce the same continuation optimization problem
for player $i$ after $h^t$.
\end{definition}

The relevant learning requirement is therefore that, if a player keeps following a single strategy, the opponent will eventually infer either that strategy itself or the strategically equivalent strategy class. 

\begin{assumption}[Strategy menu identifiability]
\label{ass:equiv_class_conc}
Fix player $i$.
Assume the support of $\mu_i^0$ is finite; write
$\mathcal S_{-i}:=\mathrm{supp}(\mu_i^0)\subseteq \mathcal F_{-i}$.
This assumption concerns only the finite retained strategy menu $\mathcal S_{-i}$ used by player $i$ for posterior sampling; it is not a restriction on the full repeated-game strategy space $\mathcal F_{-i}$.
Assume:
\begin{enumerate}[leftmargin=*]
\item \textit{(Menu grain of truth)} $f_{-i}\in\mathcal S_{-i}$ and $\mu_i^0(f_{-i})>0$.
\item \textit{(On-path identifiability up to equivalent strategy class)} If a player keeps following one strategy, then the opponent will eventually infer either that strategy itself or the strategically equivalent strategy class. Formally, for each history $h^t$, define the strategically equivalent class
\[
\mathcal E_i(h^t)
:=
\{g_{-i}\in\mathcal S_{-i}: g_{-i}\sim_i^{h^t} f_{-i}\}.
\]
Then
\[
\mu^f\!\left(
\left\{
z:
\mu_i^t(\mathcal E_i(h^t(z))\mid h^t(z))
\longrightarrow 1
\right\}
\right)=1.
\]
\end{enumerate}
\end{assumption}

\begin{remark}
For strategy menus used in our simulations in Section \ref{sec:experiments}, Assumption \ref{ass:equiv_class_conc} is easily justified: we use deterministic, finite, non-duplicative strategy menus, where wrong retained labels are hard-refuted once they predict a different observed action, while labels that would share the same cooperative path and punishment regime as the benchmark equilibrium automaton are excluded by construction. Appendix~\ref{app:retained_menu_identification} records this stronger hard-refutation route, and Appendix~\ref{appsec:game_menus} verifies it for the menus used in the experiment. 
\end{remark}

Once Assumption~\ref{ass:equiv_class_conc} holds, the posterior mass inside the true continuation-payoff equivalence class is harmless for PS-BR, because every label in that class induces the same continuation objective. Only posterior mass outside that class can create an optimization gap. The proposition below shows that when this strategically relevant posterior mass vanishes on the realized path, PS-BR delivers exactly the asymptotic on-path $\varepsilon$-consistency required in Definition~\ref{def:rr}.

\begin{proposition}[PS-BR-selected continuation plans are asymptotically $\varepsilon$-consistent on-path]\label{prop:ps_implies_asymptotic_consistency}
Fix player $i$.
Suppose player $i$ uses PS-BR at every history and
Assumption~\ref{ass:equiv_class_conc} holds for $i$.
Then for every $\varepsilon>0$,
\[
\mu^f\!\left(\left\{z:\exists\,T_i(z,\varepsilon)<\infty\ \text{s.t.}\ \forall\,t\ge T_i(z,\varepsilon),\ 
\sigma_{i,t}^{\mathrm{PS}}(\cdot\mid h^t(z))
\in \mathrm{BR}_i^\varepsilon\!\big(f_{-i}^{i,t}\big|_{h^t(z)}\mid h^t(z)\big)\right\}\right)=1.
\]
\end{proposition}

Proposition~\ref{prop:ps_implies_asymptotic_consistency} is the formal resolution of the sampling-versus-best-response tension in LLM agents at the level of the \emph{selected continuation plan}. Standard stochastic decoding prevents exact best responding at any fixed history. But the proposition shows that the continuation plan selected by PS-BR is eventually approximately optimal along the realized path. That is the only role of the assumption in this subsection; everything else follows from the fact that residual uncertainty inside the true continuation-payoff equivalence class is already harmless for PS-BR.

The proof of Lemma~\ref{lem:psbr_gap} and the proof of
Proposition~\ref{prop:ps_implies_asymptotic_consistency} are deferred to
Appendix~\ref{app:omitted_proofs}.

\section{Zero-shot Nash convergence}\label{sec:zero_shot}

We now show that the reasonably reasoning agents we defined in Section \ref{sec:rr}, together with a learnability condition on beliefs, generate play that is eventually weakly close to Nash equilibrium play along the realized path.
This argument follows the weak-subjective-equilibrium framework in \cite{norman2022possibility}, adapted to the LLM-agent setup discussed in Section~\ref{sec:rr}. 
\subsection{Weak subjective equilibrium}

We work with the standard weak distance on play-path distributions.
Let $\mathcal B^t$ be the $\sigma$-algebra generated by cylinder events of length $t$.

\begin{definition}[Weak distance]\label{def:weak_distance}
For probability measures $\mu,\nu$ over infinite play paths, define
\[
d(\mu,\nu)\ :=\ \sum_{t=1}^\infty 2^{-t}\ \sup_{E\in \mathcal B^t}\big|\mu(E)-\nu(E)\big|.
\]
For a history $h^t$ with $\mu(C(h^t))>0$ and $\nu(C(h^t))>0$, define the conditional (continuation) weak distance
\[
d_{h^t}(\mu,\nu)\ :=\ d(\mu(\cdot\mid C(h^t)),\ \nu(\cdot\mid C(h^t))).
\]
\end{definition}

We use weak distance to compare continuations of play after a realized history.

\begin{definition}[Weak similarity in continuation]\label{def:weak_similarity_cont}
Fix a history $h^t$. Two profiles $f$ and $g$ are \emph{$\eta$-weakly similar in continuation after $h^t$} if
\[
d_{h^t}(\mu^f,\mu^g)\ \le\ \eta.
\]
\end{definition}

Weak subjective equilibrium is Norman's key intermediate notion: players best respond (up to $\xi$) to their \emph{subjective} model, and their subjective model is weakly close (within $\eta$) to the objective continuation distribution.

\begin{definition}[Weak subjective equilibrium \citep{norman2022possibility}]\label{def:weak_subjective_eq}
Fix $\xi,\eta\ge 0$ and a history $h^t$.
A continuation profile $f\big|_{h^t}$ is a \emph{weak $\xi$-subjective $\eta$-equilibrium after $h^t$} if for every player $i$ there exists a supporting profile $f^i=(f_i,f_{-i}^i)$ such that:
\begin{enumerate}[leftmargin=*]
\item \textit{(Subjective best response)}\; $f_i\big|_{h^t}\in \mathrm{BR}_i^\xi\!\big(f_{-i}^i\big|_{h^t}\mid h^t\big)$, where payoffs are evaluated under $\mu^{f^i}$.
\item \textit{(Weak predictive accuracy)}\; $d_{h^t}(\mu^f,\mu^{f^i})\le \eta$.
\end{enumerate}
\end{definition}

\begin{definition}[Learns to predict the path of play (strong)]\label{def:learn_predict_path}
Player $i$ \emph{learns to predict the path of play under $f$} if for every $\eta>0$,
\[
\mu^f\!\left(\left\{z:\exists\,T_i(z,\eta)<\infty\ \text{s.t.}\ \forall\,t\ge T_i(z,\eta),\ 
d_{h^t(z)}(\mu^f,\mu^{f^i})\le \eta\right\}\right)=1,
\]
where $f^i=(f_i,f_{-i}^i)$ is a supporting predictive profile for player $i$ (Section~\ref{sec:beliefs}; formal construction in Appendix~\ref{app:belief_repr}).
\end{definition}

\begin{remark}[Connection to Optimizing Learnability]
A longstanding challenge in Bayesian learning in games is \citet{nachbar1997prediction, nachbar2005beliefs}'s inconsistency result, which shows that requiring an agent to learn and best-respond on \emph{all possible} continuation paths is often mathematically impossible. However, \citet{norman2022possibility} resolved this by introducing \emph{optimizing learnability}, the insight that agents only need to learn the continuation play along the realized paths generated by their optimizing choices. Our RR definition naturally instantiates Norman's insight: Definition~\ref{def:rr} and Definition~\ref{def:learn_predict_path} require $\varepsilon$-consistency and predictive accuracy strictly $\mu^f$-almost surely (i.e., strictly on the realized, optimizing play path). Therefore, the on-path merging of opinions guaranteed by \citet{blackwell1962merging} is entirely sufficient for zero-shot Nash convergence, bypassing Nachbar's impossibility.
\end{remark}

Crucially, the learning of the true path (strong path prediction) relies purely on the absolute continuity of beliefs. It does not require exact strategy identification, and it is logically distinct from the equivalence-class condition used for PS-BR. Strong path prediction can therefore be verified directly from Assumption~\ref{ass:grain_of_truth} via the classic merging of opinions result. The following Lemma \ref{lem:absolute_cont_predict} formalizes this idea.

\begin{lemma}[Absolute continuity implies strong path prediction]\label{lem:absolute_cont_predict}
Fix player $i$. Suppose the objective play-path distribution $\mu^f$ is absolutely continuous with respect to player $i$'s prior predictive distribution $P_i^{0, f_i}$ (Assumption~\ref{ass:grain_of_truth}). Then player $i$ learns to predict the path of play under $f$ in the sense of Definition~\ref{def:learn_predict_path}.
\end{lemma}

The proof is deferred to Appendix~\ref{app:omitted_proofs}.

\subsection{From learning to zero-shot Nash convergence}

We first show that asymptotic $\varepsilon$-consistency, together with strong prediction, implies that the realized continuation play is eventually a weak subjective equilibrium.

\begin{proposition}\label{prop:weak_subjective_from_learning}
Suppose each player $i$ satisfies the asymptotic best-response condition from Definition~\ref{def:rr} and learns to predict the path of play under $f$ (Definition~\ref{def:learn_predict_path}).
Then for any $\xi>0$ and $\eta>0$,
\[
\mu^f\!\left(\left\{z:\exists\,T(z)<\infty\ \text{s.t.}\ \forall\,t\ge T(z),\ 
f\big|_{h^t(z)}\ \text{is a weak $\xi$-subjective $\eta$-equilibrium after $h^t(z)$}\right\}\right)=1.
\]
\end{proposition}

Finally, we convert a weak subjective equilibrium into proximity to a Nash equilibrium.

\begin{theorem}[Zero-shot Nash convergence along realized play]\label{thm:zero_shot_nash_final}
Suppose every player satisfies the asymptotic best-response condition from Definition~\ref{def:rr} and learns to predict the path of play under $f$.
Then for every $\varepsilon>0$,
\begin{align}
\mu^f\!\bigl(\bigl\{z:\exists\, & T(z)<\infty\ \text{s.t.} \
\forall\,t\ge T(z),\ \exists \ \hat f^{\varepsilon,t,z}\ \text{an $\varepsilon$-Nash equilibrium} \notag
\\
& \text{of the continuation game after $h^t(z)$ with}\ 
d_{h^t(z)}(\mu^f,\mu^{\hat f^{\varepsilon,t,z}})\le \varepsilon\bigr\}\bigr)=1. \notag 
\end{align}

\end{theorem}

\begin{corollary}[Zero-shot Nash convergence for PS-BR]\label{cor:psbr_zero_shot_nash}
Assume that for every player $i$, Assumption~\ref{ass:equiv_class_conc} holds and player $i$ uses PS-BR (Definition~\ref{def:psbr}).
Assume moreover that, in the continuation game after each reached history, the continuation actually played is the PS-BR plan selected at that history.
Then the conclusion of Theorem~\ref{thm:zero_shot_nash_final} holds.
\end{corollary}

The proofs of Theorem~\ref{thm:zero_shot_nash_final} and
Corollary~\ref{cor:psbr_zero_shot_nash} are deferred to
Appendix~\ref{app:omitted_proofs}. For the sparse simulation menus
used in our applications, Appendix~\ref{app:retained_menu_identification} shows that
the needed concentration premise follows from deterministic hard refutation on the
finite menu. The richer engineering rollout menu may be broader; our
identification claim is only for the sparse appendix menu.

The main abstract theorem, Theorem \ref{thm:zero_shot_nash_final}, together with the PS-BR corollary \ref{cor:psbr_zero_shot_nash}, may seem counterintuitive: if each agent is learning, then what each agent is trying to
predict changes over time, so why should behavior ever stabilize?
This concern is valid for many myopic learning models, where the learner treats
the opponent as having a fixed action distribution even though the opponent is
also adapting. The promise of Bayesian learning \citep{kalai1993rational} is that, under a suitable grain-of-truth condition, agents'
\emph{posterior predictive} forecasts about future play can nonetheless become accurate (merge) along the realized path. In repeated games, the correct object of inference is
not a fixed action, but the opponent's \emph{repeated-game strategy}: a fixed contingent plan
(mapping histories to actions) that may be highly nonstationary. In particular, even if an opponent
updates beliefs and changes its period-by-period best response, once its prior, update rule, and decision
rule are fixed from time 0, its behavior defines a single mapping $f_{-i}:H\to\Delta(A_{-i})$ (hence a fixed repeated-game strategy in our sense).
Agents' beliefs change because they refine uncertainty about this fixed mapping (and its on-path implications), not because the mapping is being rewritten exogenously over time. 

Indeed, our main results do \emph{not} require that posteriors over opponent strategies literally
stop moving. Instead, they require on-path stabilization in two weaker senses:

\begin{enumerate}[leftmargin=*]
    \item \textit{Stability of forecasts (predictive merging).}
    Under the grain-of-truth condition (Assumption~\ref{ass:grain_of_truth}), Bayesian
    updating implies that, along $\mu^f$-almost every realized history $h^t(z)$, the
    agent's \emph{posterior predictive} distribution over future play becomes close to
    the true continuation distribution (formalized later by Definition~\ref{def:learn_predict_path}
    and Lemma~\ref{lem:absolute_cont_predict}). Importantly, this can happen even if the
    posterior over \emph{strategy labels} does not concentrate: distinct strategy hypotheses
    may be observationally equivalent \emph{on the realized path}, and any remaining
    disagreement can persist only on counterfactual histories that are not reached.
    \item \textit{Stability of (approximate) best responses.}
    Once an agent's predictive belief about continuation play is accurate on-path,
    playing an $\varepsilon$-best response to that belief is also nearly optimal against
    the true continuation play. Moreover, best-response \emph{sets} need not vary wildly:
    when the payoff gap between the best action and the runner-up is nontrivial, small
    changes in beliefs do not change which continuation strategies are $\varepsilon$-optimal.
    This is exactly why our RR definition imposes only \emph{asymptotic} on-path
    $\varepsilon$-consistency (Definition~\ref{def:rr}), rather than requiring perfect
    best-response optimality at every time and every counterfactual history.
\end{enumerate}

Even if beliefs keep updating forever, behavior can still stabilize because decisions
depend on the \emph{predictive} implications of beliefs on the realized continuation game.
If the posterior mass shifts among hypotheses that induce (nearly) the same continuation
distribution after $h^t(z)$, then the agent's best-response problem is (nearly) unchanged,
so play remains stable. For our sampled-label PS-BR implementation, we formalize exactly this
logic through the equivalence-class condition in
Assumption~\ref{ass:equiv_class_conc}: the posterior may keep moving, but only negligible mass
can remain on retained labels that would induce a different continuation optimization problem.
The proof of Proposition~\ref{prop:ps_implies_asymptotic_consistency} then shows that the resulting sampling randomness is
strategically innocuous, because the induced best-response gap is bounded by twice the posterior
mass outside the true equivalence class. This yields eventual on-path
$\varepsilon$-best-response continuation plans.

\subsection{Zero-shot stage-game Nash convergence for myopic rules}
\label{sec:myopic_stage_nash}

Theorem~\ref{thm:zero_shot_nash_final} and Corollary~\ref{cor:psbr_zero_shot_nash} establish eventual on-path convergence to a Nash equilibrium of the continuation game. That guarantee is deliberately strong: it concerns repeated-game optimality and therefore requires beliefs over opponents' full continuation strategies. Yet this level of reasoning may be unnecessary when the object of interest is only \emph{stage-wise} strategic optimality. If we ask instead whether the realized mixed action profile at each history is eventually an approximate Nash equilibrium of the one-shot stage game, then predicting the opponents' next joint action may suffice. This reduction captures the logic of SCoT \citep{akata2025playing}, which implements a ``predict the next move, then best respond'' procedure rather than full continuation planning. The purpose of this subsection is to justify this simplification formally. We analyze two one-step variants: myopic PS-BR, which best responds to a one-step predictive belief, and SCoT \citep{akata2025playing}, which best responds to a deterministic point prediction of the opponents' next action.

\subsubsection{Myopic PS-BR}

\emph{myopic PS-BR} retains the Bayesian-learning-plus-best-response structure of the previous subsection, but truncates both objects to one period: the agent forms a one-step predictive belief over the opponents' next joint action and then plays a myopic best response to that belief.


For notational convenience, as already used above, for any opponents' profile
$g_{-i}$ and history $h$, we write
\[
g_{-i}(h)\in \Delta(A_{-i})
\]
for the induced distribution over the opponents' joint next action at history $h$.
In particular, when $g_{-i}$ is an actual profile of opponents' mixed actions, this is the product distribution
\[
g_{-i}(h)=\bigotimes_{j\neq i} g_j(h).
\]

\begin{definition}[One-shot stage-game $\varepsilon$-best response and stage $\varepsilon$-Nash]
\label{def:stage_br_nash}
For $\alpha_i\in\Delta(A_i)$ and $q\in\Delta(A_{-i})$, define
\[
u_i(\alpha_i,q)
:=
\sum_{a_i\in A_i}\sum_{a_{-i}\in A_{-i}}
\alpha_i(a_i)\,q(a_{-i})\,u_i(a_i,a_{-i}).
\]
For $\varepsilon\ge 0$, define
\[
\mathrm{br}_i^\varepsilon(q)
:=
\left\{
\alpha_i\in\Delta(A_i):
u_i(\alpha_i,q)\ge \sup_{\alpha_i'\in\Delta(A_i)}u_i(\alpha_i',q)-\varepsilon
\right\}.
\]
We also write
\[
\mathrm{br}_i(q):=\mathrm{br}_i^0(q).
\]
At a history $h^t$, write
\[
f_{-i}(h^t):=\bigotimes_{j\neq i} f_j(h^t)\in \Delta(A_{-i})
\]
for the actual current joint mixed action of player $i$'s opponents.
The current mixed-action profile
\[
f(h^t):=(f_1(h^t),\ldots,f_N(h^t))\in \prod_{j\in I}\Delta(A_j)
\]
is a \emph{stage $\varepsilon$-Nash equilibrium} if
\[
f_i(h^t)\in \mathrm{br}_i^\varepsilon\!\bigl(f_{-i}(h^t)\bigr)
\qquad\text{for every }i\in I.
\]
\end{definition}

Fix player $i$ and let $f^i=(f_i,f_{-i}^i)$, where $f_{-i}^i$ is the prior-predictive reference model from Section~\ref{sec:beliefs}. Let $f_{-i}^{i,t}$ denote a continuation model reproducing player $i$'s predictive continuation law at history $h^t$. We write
\[
q_i^t(\cdot\mid h^t)\ :=\ f_{-i}^{i,t}(h^t)\in\Delta(A_{-i})
\]
for player $i$'s one-step posterior predictive belief about the opponents' next joint action. When the posterior $\mu_i^t(\cdot\mid h^t)$ is supported on a finite set
$\mathcal S_{-i}\subseteq \mathcal F_{-i}$, this is
\[
q_i^t(\cdot\mid h^t)
=
\sum_{g_{-i}\in\mathcal S_{-i}}
\mu_i^t(g_{-i}\mid h^t)\, g_{-i}(h^t)(\cdot).
\]

\begin{definition}[Myopic posterior-sampling best response (myopic PS-BR)]
\label{def:myopic_psbr}
Fix player $i$ and a history $h^t$.
Suppose $\mu_i^t(\cdot\mid h^t)$ is supported on a finite set $\mathcal S_{-i}$.
For each $g_{-i}\in\mathcal S_{-i}$, choose a mixed action
\[
\alpha_i^{g_{-i},h^t}\in \mathrm{br}_i\!\bigl(g_{-i}(h^t)\bigr).
\]
Myopic PS-BR:
\begin{enumerate}[leftmargin=*]
\item samples $\tilde f_{-i}\sim \mu_i^t(\cdot\mid h^t)$;
\item uses the mixed action $\alpha_i^{\tilde f_{-i},h^t}$.
\end{enumerate}
The induced ex ante mixed action is
\[
\alpha_{i,t}^{\mathrm{mPS}}(\cdot\mid h^t)
:=
\sum_{g_{-i}\in\mathcal S_{-i}}
\mu_i^t(g_{-i}\mid h^t)\,
\alpha_i^{g_{-i},h^t}(\cdot).
\]
Whenever player $i$ uses myopic PS-BR, we identify
\[
f_i(h^t)=\alpha_{i,t}^{\mathrm{mPS}}(\cdot\mid h^t).
\]
\end{definition}

\begin{lemma}[Stage best responses are stable under nearby beliefs]
\label{lem:stage_br_stability}
Fix player $i$ and define
\[
\|p-q\|_{\TV}:=\sup_{B\subseteq A_{-i}} |p(B)-q(B)|
\qquad\text{for }p,q\in\Delta(A_{-i}).
\]
If $\alpha_i\in \mathrm{br}_i^\xi(q)$, then
\[
\alpha_i\in \mathrm{br}_i^{\xi+2\|p-q\|_{\TV}}(p).
\]
\end{lemma}

\begin{lemma}[Myopic PS-BR is a $D_i^t$-stage best response]
\label{lem:myopic_psbr_gap}
Fix player $i$ and a history $h^t$. Suppose $\mu_i^t(\cdot\mid h^t)$ is supported on a
finite set $\mathcal S_{-i}$ and write
\[
p_t(g_{-i}) := \mu_i^t(g_{-i}\mid h^t),\qquad g_{-i}\in\mathcal S_{-i}.
\]
Define
\[
D_i^t(h^t)
:=
1-\sum_{g_{-i}\in\mathcal S_{-i}} p_t(g_{-i})^2.
\]
Let $\alpha_{i,t}^{\mathrm{mPS}}(\cdot\mid h^t)$ be myopic PS-BR and let
\[
q_i^t(\cdot\mid h^t)
=
\sum_{g_{-i}\in\mathcal S_{-i}}p_t(g_{-i})\,g_{-i}(h^t)(\cdot)
\]
be the one-step posterior predictive belief. Then
\[
u_i\!\bigl(\alpha_{i,t}^{\mathrm{mPS}},\, q_i^t(\cdot\mid h^t)\bigr)
\ge
\sup_{\alpha_i\in\Delta(A_i)}u_i\!\bigl(\alpha_i,\, q_i^t(\cdot\mid h^t)\bigr)
-
D_i^t(h^t).
\]
Equivalently,
\[
\alpha_{i,t}^{\mathrm{mPS}}(\cdot\mid h^t)
\in
\mathrm{br}_i^{D_i^t(h^t)}\!\bigl(q_i^t(\cdot\mid h^t)\bigr).
\]
\end{lemma}

\begin{lemma}[Strong path prediction implies one-step predictive accuracy]
\label{lem:one_step_predict_accuracy}
Fix player $i$. Suppose player $i$ learns to predict the path of play under $f$
(Definition~\ref{def:learn_predict_path}). Then
\[
\mu^f\!\left(
\left\{
z:\forall \eta>0,\ \exists T_i(z,\eta)<\infty\ \text{s.t.}\ \forall t\ge T_i(z,\eta),\
\big\|q_i^t(\cdot\mid h^t(z)) - f_{-i}(h^t(z))\big\|_{\TV}\le \eta
\right\}
\right)=1.
\]
\end{lemma}

For stage-game Nash convergence, exact identification of the opponents' full continuation
strategy is again stronger than needed. What matters is only whether the retained labels induce
the same \emph{one-step stage decision problem} at the reached history. The relevant comparison is
now myopic rather than intertemporal: two labels are equivalent if, at the current history, they
make every current mixed action deliver the same expected stage payoff to player $i$. In that
case, distinguishing between them has no effect on the current stage best-response problem.

\begin{definition}[On-path stage-payoff equivalence]
\label{def:stage_equiv}
Fix player $i$ and a history $h^t$. For opponents' continuation profiles
$g_{-i},g'_{-i}\in\mathcal F_{-i}$, write
\[
g_{-i}\approx_i^{h^t} g'_{-i}
\]
if
\[
u_i(\alpha_i,g_{-i}(h^t))
=
u_i(\alpha_i,g'_{-i}(h^t))
\qquad
\text{for every }\alpha_i\in\Delta(A_i).
\]
Equivalently, $g_{-i}$ and $g'_{-i}$ induce the same one-step stage optimization problem
for player $i$ at history $h^t$.
\end{definition}

\begin{assumption}[Strategy identifiability]
\label{ass:stage_equiv_class_conc}
Fix player $i$.
Assume the support of $\mu_i^0$ is finite; write
\[
\mathcal S_{-i}:=\mathrm{supp}(\mu_i^0)\subseteq \mathcal F_{-i}.
\]
This assumption again concerns only the finite retained strategy menu $\mathcal S_{-i}$ used
for posterior sampling, not the full strategy space $\mathcal F_{-i}$.
Assume:
\begin{enumerate}[leftmargin=*]
\item \textit{(Menu grain of truth)} $f_{-i}\in\mathcal S_{-i}$ and $\mu_i^0(f_{-i})>0$.
\item \textit{(On-path concentration only up to stage-payoff equivalence)} For each history $h^t$, define
\[
\mathcal E_i^{\mathrm{st}}(h^t)
:=
\{g_{-i}\in\mathcal S_{-i}: g_{-i}\approx_i^{h^t} f_{-i}\}.
\]
Then
\[
\mu^f\!\left(
\left\{
z:
1-\mu_i^t(\mathcal E_i^{\mathrm{st}}(h^t(z))\mid h^t(z))
\longrightarrow 0
\right\}
\right)=1.
\]
\end{enumerate}
\end{assumption}

As in the continuation-value case, Assumption~\ref{ass:stage_equiv_class_conc} is a
menu-level condition and does not require exact identification of the literal opponent label.
It requires only that, on the realized path, posterior mass within the retained menu eventually
falls on labels that are stage-payoff-equivalent to the truth at the reached history. Any
residual posterior uncertainty inside that one-step equivalence class is harmless for myopic
PS-BR, because it leaves the current stage optimization problem unchanged.

\begin{lemma}[Stage-equivalence concentration controls the myopic PS-BR gap]
\label{lem:myopic_equiv_gap}
Fix player $i$ and a history $h^t$. Let
\[
\beta_i^t(h^t)
:=
1-\mu_i^t(\mathcal E_i^{\mathrm{st}}(h^t)\mid h^t),
\]
where $\mathcal E_i^{\mathrm{st}}(h^t)$ is the true stage-payoff equivalence class from
Assumption~\ref{ass:stage_equiv_class_conc}. Let
$\alpha_{i,t}^{\mathrm{mPS}}(\cdot\mid h^t)$ be myopic PS-BR. Then
\[u_i\!\bigl(\alpha_{i,t}^{\mathrm{mPS}},\, f_{-i}(h^t)\bigr)
\ge
\sup_{\alpha_i\in\Delta(A_i)}u_i\!\bigl(\alpha_i,\, f_{-i}(h^t)\bigr)
-
\beta_i^t(h^t).
\]
Equivalently,
\[
\alpha_{i,t}^{\mathrm{mPS}}(\cdot\mid h^t)
\in
\mathrm{br}_i^{\beta_i^t(h^t)}\!\bigl(f_{-i}(h^t)\bigr).
\]
\end{lemma}

The proof logic is the exact myopic analogue of the continuation-level PS-BR argument above. Posterior mass
on the true stage-payoff equivalence class is strategically harmless, because every label in that
class induces the same current stage payoff function over player $i$'s mixed actions. Only the
remaining posterior mass on stage-distinct labels can create a myopic best-response gap.

\begin{theorem}[Stage-game Nash convergence under myopic PS-BR]
\label{thm:myopic_psbr_stage_nash}
Assume that for every player $i$, Assumption~\ref{ass:stage_equiv_class_conc} holds and player $i$
uses myopic PS-BR (Definition~\ref{def:myopic_psbr}) at every history.
Then for every $\varepsilon>0$,
\[
\mu^f\!\left(
\left\{
z:\exists T(z)<\infty\ \text{s.t.}\ \forall t\ge T(z),\
f(h^t(z))\ \text{is a stage $\varepsilon$-Nash equilibrium}
\right\}
\right)=1.
\]
\end{theorem}

\subsection{SCoT \citep{akata2025playing}}

The second reduction is SCoT \citep{akata2025playing}. Instead of best responding to the full one-step predictive distribution, the agent first forms a \emph{deterministic point prediction} of the opponents' next joint action and then best responds to that point prediction. In general, this is not equivalent to best responding to a mixed belief, so the argument is different from the classical Bayesian-learning-plus-best-response route. Nevertheless, when all players use deterministic point-prediction rules, the true next action along the realized path is pure at every history, and predictive accuracy is enough to make the point prediction eventually correct. This gives eventual stage-game Nash convergence under a different mechanism than myopic PS-BR.

\begin{definition}[Social Chain of Thought (SCoT) \citep{akata2025playing}]
\label{def:det_map_scot}
Fix player $i$.
At each history $h^t$, let
\[
q_i^t(\cdot\mid h^t):=f_{-i}^{i,t}(h^t)\in\Delta(A_{-i})
\]
denote player $i$'s one-step predictive distribution over opponents' next joint action.

A \emph{SCoT} rule for player $i$ consists of:
\begin{enumerate}[leftmargin=*]
\item a deterministic MAP (maximum a posteriori) selector
\[
\hat a_{-i}^t(h^t)\in \arg\max_{a_{-i}\in A_{-i}} q_i^t(a_{-i}\mid h^t);
\]
\item a deterministic pure best-response selector
\[
b_i:A_{-i}\to A_i
\qquad\text{such that}\qquad
b_i(a_{-i})\in \arg\max_{a_i\in A_i} u_i(a_i,a_{-i})
\ \ \text{for every }a_{-i}\in A_{-i}.
\]
\end{enumerate}
The induced strategy is
\[
f_i(h^t)\ :=\ \delta_{\,b_i(\hat a_{-i}^t(h^t))}\in\Delta(A_i).
\]
Thus a SCoT player uses a pure action at every history.
\end{definition}

\begin{lemma}[Deterministic truth implies asymptotic purity and eventual MAP correctness]
\label{lem:det_truth_implies_map_correct}
Fix player $i$ and suppose player $i$ learns to predict the path of play under $f$
in the sense of Definition~\ref{def:learn_predict_path}.
Assume that for every history $h\in H$ there exists an action
$a_{-i}^{\star}(h)\in A_{-i}$ such that
\[
f_{-i}(h)=\delta_{a_{-i}^{\star}(h)}.
\]
Then
\[
\mu^f\!\left(
\left\{
z:\exists T_i(z)<\infty\ \text{s.t.}\ \forall t\ge T_i(z),\
\hat a_{-i}^t(h^t(z))=a_{-i}^{\star}(h^t(z))
\right\}
\right)=1.
\]
In particular, along $\mu^f$-almost every realized path $z$,
\[
q_i^t\!\bigl(a_{-i}^{\star}(h^t(z))\mid h^t(z)\bigr)\longrightarrow 1
\qquad\text{and}\qquad
1-\max_{a_{-i}\in A_{-i}} q_i^t(a_{-i}\mid h^t(z))\longrightarrow 0.
\]
\end{lemma}

\begin{theorem}[One-shot stage-game Nash convergence for SCoT]
\label{thm:det_scot_stage_nash}
Suppose every player $i\in I$ uses SCoT in the sense of
Definition~\ref{def:det_map_scot}, and suppose every player learns to predict the
path of play under $f$ in the sense of Definition~\ref{def:learn_predict_path}.
Then
\[
\mu^f\!\left(
\left\{
z:\exists T(z)<\infty\ \text{s.t.}\ \forall t\ge T(z),\
f(h^t(z))\ \text{is a stage Nash equilibrium}
\right\}
\right)=1.
\]
Equivalently, along $\mu^f$-almost every realized path, the current mixed-action profile
eventually becomes a stage $0$-Nash equilibrium.
\end{theorem}

\begin{corollary}[Bayesian stage-game Nash convergence for SCoT]
\label{cor:det_scot_stage_nash_grain_truth}
Suppose every player uses deterministic MAP-SCoT and Assumption~\ref{ass:grain_of_truth}
holds for every player.
Then the conclusion of Theorem~\ref{thm:det_scot_stage_nash} holds:
\[
\mu^f\!\left(
\left\{
z:\exists T(z)<\infty\ \text{s.t.}\ \forall t\ge T(z),\
f(h^t(z))\ \text{is a stage Nash equilibrium}
\right\}
\right)=1.
\]
\end{corollary}

\begin{remark}
Theorem~\ref{thm:det_scot_stage_nash} relies on the fact that when \emph{all players} use SCoT with deterministic tie-breaking, the true current action profile
is pure at every history. This is why asymptotic purity need not be imposed separately: it is implied by Bayesian one-step predictive accuracy toward a pure truth.
If opponents are allowed to play genuinely mixed current actions, this argument breaks down,
and additional conditions such as asymptotic purity or BR-invariance are again needed.

The SCoT result is therefore naturally paired with the
grain-of-truth assumption (Assumption~\ref{ass:grain_of_truth}) and the corresponding
merging-of-opinions argument, rather than with Assumptions~\ref{ass:equiv_class_conc} and \ref{ass:stage_equiv_class_conc},
which are tailored to posterior-sampling rules such as PS-BR and myopic PS-BR.
\end{remark}

The proofs are deferred to
Appendix~\ref{app:omitted_proofs}. Taken together, Theorem~\ref{thm:myopic_psbr_stage_nash} and Theorem~\ref{thm:det_scot_stage_nash} show that, for the weaker objective of \emph{stage-game} Nash convergence, full continuation planning is not necessary. However, these one-step results are inherently limited to stage-game equilibrium. They do not by themselves recover more demanding \emph{continuation-game} or \emph{history-contingent repeated-game} equilibria, whose incentive structure is sustained by the value of future paths of play. Establishing convergence to those richer repeated-game equilibria requires a procedure, such as PS-BR, that reasons over full continuation strategies rather than only over the next-period action.

\section{Extension to unknown, stochastic, and private payoffs}
\label{sec:unknown_payoffs}

Sections~\ref{sec:setup}--\ref{sec:zero_shot} assumed that the stage payoff functions
$u_i:A\to[0,1]$ are common knowledge and deterministic. We now drop that assumption
and allow each agent to observe only its own privately realized stochastic payoffs.
The key simplifying choice in this extension is that the continuation game is still
evaluated on \emph{public histories}: private reward histories are used to update a
posterior over player $i$'s own mean payoff matrix, but the continuation strategies
being compared remain public-history strategies, exactly as in Section~\ref{sec:setup}.
This keeps the equilibrium benchmark on the realized public play path and removes the
need for an additional hidden-history-collapse assumption.

\subsection{Private-payoff repeated game and observable histories}

Fix the same action sets $(A_i)_{i\in I}$ and discount factors $(\lambda_i)_{i\in I}$ as
in Section~\ref{sec:setup}. For each player $i$, let $\mathcal R_i\subseteq \mathbb R$
denote the payoff space and let $\nu_i(\mathrm dr)$ be a dominating base measure
(counting measure in the discrete case, Lebesgue measure in the continuous case).

We assume that the payoff noise family is known. Concretely, for each player $i$ there
is a known family of densities
\[
\psi_i(r;\mu),\qquad r\in \mathcal R_i,\ \mu\in\mathbb R,
\]
where the parameter $\mu$ is the mean payoff. The true unknown object is player $i$'s
mean payoff matrix
\[
u_i:A\to[0,1].
\]
(As usual, any bounded payoff matrix can be affinely normalized into $[0,1]$ without
changing best responses or Nash inequalities.)

At round $t$, after the public joint action $a^t\in A$ is realized, player $i$ privately
observes
\begin{equation}
r_i^t \sim q_i^{u_i}(\cdot\mid a^t),
\qquad \text{where} \;\;
q_i^{u_i}(\mathrm dr\mid a)
:=
\psi_i(r;u_i(a))\,\nu_i(\mathrm dr).
\label{eq:private_payoff}
\end{equation}
Thus the true payoff kernel is determined by the true mean matrix $u_i$.

For payoff learning, player $i$'s observable history at time $t$ is
\[
x_i^t := (h^t, r_i^{1:t-1}) \in X_i^t := H^t \times \mathcal R_i^{t-1},
\qquad
X_i := \bigcup_{t\ge 1} X_i^t.
\]
The continuation plans remain the public history strategies from
Section~\ref{sec:setup}: after a realized public history $h^t$, candidate continuation
plans are elements of $\mathcal F_i(h^t)$.

The full sample space is
\[
\Omega
:=
\prod_{t\ge 1}
\Bigl(
A \times \prod_{i\in I}\mathcal R_i
\Bigr),
\]
whose typical element is
\[
\omega=(a^1,r^1,a^2,r^2,\ldots),
\qquad
r^t=(r_1^t,\ldots,r_N^t).
\]
We write $h^t(\omega)=(a^1,\ldots,a^{t-1})$ and
$x_i^t(\omega)=(h^t(\omega),r_i^{1:t-1}(\omega))$.

Let $\beta_i:X_i\to\Delta(A_i)$ denote player $i$'s actual one-step behavioral rule in the
private-payoff environment, and let $\beta=(\beta_i)_{i\in I}$. Together with the true mean
matrices $u=(u_i)_{i\in I}$, the behavioral profile $\beta$ induces a unique law $P^{\beta,u}$ on
$\Omega$ \footnote{This is by the Ionescu--Tulcea theorem \citep{ionescu1949mesures}; see \cite{pollard2002user} also.}.

The induced public-history strategy profile $f=(f_i)_{i\in I}\in\mathcal F$ is then defined by
\[
f_i(h^t)
:=
P^{\beta,u}(a_i^t\in\cdot\mid h^t),
\qquad i\in I.
\]
Its public-action law is exactly the public marginal of $P^{\beta,u}$, which we denote by
$\mu^f$.

Because continuation play is evaluated on public histories, the relevant continuation payoff
after $h^t$ is
\[
U_i^{u_i}(g\mid h^t)
:=
\mathbb E_{y\sim \mu_{h^t}^{g}}
\left[
(1-\lambda_i)\sum_{k=0}^{\infty}\lambda_i^k u_i(y^{k+1})
\right],
\]
for any continuation profile $g$ after $h^t$.
A continuation profile $g\big|_{h^t}$ is an $\varepsilon$-Nash equilibrium of the
private-payoff continuation game after $h^t$ if, for every $i\in I$,
\[
U_i^{u_i}(g\mid h^t)
\ge
\sup_{g_i'\in \mathcal F_i(h^t)}
U_i^{u_i}(g_i',g_{-i}\mid h^t)-\varepsilon.
\]

\subsection{Known-noise, unknown-mean parametrization}

We now impose the finite-menu structure used by PS-BR.
For player $i$, let $\mathcal M_i$ be a finite menu of candidate mean payoff matrices
\[
m_i:A\to[0,1].
\]
Each $m_i\in\mathcal M_i$ induces a payoff kernel
\[
q_i^{m_i}(\mathrm dr\mid a)
:=
\psi_i(r;m_i(a))\,\nu_i(\mathrm dr).
\]
Thus sampling a payoff-matrix label is exactly sampling a payoff kernel, expressed in
mean-matrix coordinates.

Given $x_i^t=(h^t,r_i^{1:t-1})$, player $i$'s posterior over candidate mean matrices is
\begin{equation}
\pi_i^t(m_i\mid x_i^t)
\propto
\pi_i^0(m_i)\prod_{s=1}^{t-1}\psi_i(r_i^s;m_i(a^s)),
\qquad
m_i\in\mathcal M_i.
\label{eq:payoff_posterior}
\end{equation}
As in Sections~\ref{sec:rr}--\ref{sec:zero_shot}, we model player $i$'s beliefs about
the opponents through a finite menu of public-action continuation models
\[
g_{-i}:H\to \Delta(A_{-i}).
\]
Let $\mathcal S_{-i}$ denote the finite retained menu and let
\[
\mu_i^t(\cdot\mid h^t)
\]
be player $i$'s posterior over $\mathcal S_{-i}$.

\subsection{Subjective continuation values and PS-BR}
\label{sec:psbr_private_payoffs}

Fix player $i$, an observable history $x_i^t=(h^t,r_i^{1:t-1})$, an
opponents' continuation model $g_{-i}\in\mathcal S_{-i}$, and a continuation strategy
$\tau_i\in\mathcal F_i(h^t)$.

Because continuation play is public-history based, a candidate mean matrix affects only
the continuation objective, not the induced public-action law. We therefore define the
$m_i$-subjective continuation value by
\begin{equation}
V_i^{m_i}(\tau_i\mid h^t; g_{-i})
:=
\mathbb E_{y\sim \mu_{h^t}^{(\tau_i,g_{-i})}}
\left[
(1-\lambda_i)\sum_{k=0}^{\infty}\lambda_i^k m_i(y^{k+1})
\right].
\label{eq:value_under_payoff_kernel}
\end{equation}
For $\varepsilon\ge 0$, define
\[
\mathrm{BR}_{i,m_i}^\varepsilon(g_{-i}\mid h^t)
:=
\left\{
\tau_i\in\mathcal F_i(h^t):
V_i^{m_i}(\tau_i\mid h^t; g_{-i})
\ge
\sup_{\tau_i'\in\mathcal F_i(h^t)}
V_i^{m_i}(\tau_i'\mid h^t; g_{-i})-\varepsilon
\right\},
\]
and write
\[
\mathrm{BR}_{i,m_i}(g_{-i}\mid h^t)
:=
\mathrm{BR}_{i,m_i}^0(g_{-i}\mid h^t).
\]

Player $i$'s mixed subjective continuation value is
\begin{equation}
V_i^{\mathrm{mix},t}(\tau_i\mid x_i^t)
:=
\mathbb E_{\substack{g_{-i}\sim \mu_i^t(\cdot\mid h^t)\\ m_i\sim \pi_i^t(\cdot\mid x_i^t)}}
\left[
V_i^{m_i}(\tau_i\mid h^t; g_{-i})
\right].
\label{eq:mix_value_def}
\end{equation}
For the true mean matrix $u_i$, define
\begin{equation}
V_i^{u_i,t}(\tau_i\mid x_i^t)
:=
\mathbb E_{g_{-i}\sim \mu_i^t(\cdot\mid h^t)}
\left[
V_i^{u_i}(\tau_i\mid h^t; g_{-i})
\right].
\label{eq:true_mean_value_private}
\end{equation}

Let $g_{-i}^{i,t}$ denote any representative continuation model for the
posterior predictive continuation law induced by $\mu_i^t(\cdot\mid h^t)$. Concretely,
$g_{-i}^{i,t}$ is chosen so that for every continuation strategy
$\tau_i\in\mathcal F_i(h^t)$,
\begin{equation}
V_i^{u_i,t}(\tau_i\mid x_i^t)
=
V_i^{u_i}(\tau_i\mid h^t; g_{-i}^{i,t}).
\label{eq:private_repr_value}
\end{equation}
When $\mathcal S_{-i}=\{g_{-i}^1,\dots,g_{-i}^K\}$ is finite, one convenient choice is
\[
g_{-i}^{i,t}(h)(a_{-i})
=
\sum_{k=1}^K \mu_i^{t,h}(g_{-i}^k)\, g_{-i}^k(h)(a_{-i}),
\qquad h\succeq h^t,
\]
where $\mu_i^{t,h}$ is the continuation posterior obtained by updating
$\mu_i^t(\cdot\mid h^t)$ along the continuation history $h$.

Because payoff uncertainty affects only the continuation objective, not the induced public
transition law of a fixed continuation strategy, player $i$'s posterior predictive law over
future public action paths under the realized continuation $f_i\big|_{h^t}$ is simply
\begin{equation}
\Pi_i^t(\cdot\mid x_i^t)
:=
\mu_{h^t}^{(f_i,g_{-i}^{i,t})}.
\label{eq:private_predictive_mixture}
\end{equation}

We can now state the private-payoff PS-BR rule.

\begin{definition}[Posterior-sampling best response (PS-BR) with private payoffs]
\label{def:psarbr}
Fix player $i$ and an observable history $x_i^t=(h^t,r_i^{1:t-1})$.
Given:
(i) the posterior $\mu_i^t(\cdot\mid h^t)$ over opponents' continuation
models, and
(ii) the posterior $\pi_i^t(\cdot\mid x_i^t)$ over player $i$'s own mean payoff
matrices,
PS-BR chooses a \emph{public-history continuation strategy} by:
\begin{enumerate}[leftmargin=*]
\item sample an opponents' continuation model
$\tilde g_{-i}\sim \mu_i^t(\cdot\mid h^t)$;
\item sample a mean payoff matrix
$\tilde m_i\sim \pi_i^t(\cdot\mid x_i^t)$;
\item play any continuation strategy
$\tau_i\in \mathrm{BR}_{i,\tilde m_i}(\tilde g_{-i}\mid h^t)$.
\end{enumerate}
Denote the resulting randomized continuation strategy by
$\sigma_{i,t}^{\mathrm{PS}}(\cdot\mid x_i^t)$. Its current-round prescription induces the
actual one-step behavioral rule via
\[
\beta_i(x_i^t)
:=
\sigma_{i,t}^{\mathrm{PS}}(h^t),
\]
and the induced public-history profile $f$ is defined from the resulting actual law
$P^{\beta,u}$ as above.
\end{definition}

\subsection{Learning requirements in the private-payoff game}

The opponent-side learning object is the same public continuation problem as in
Section~\ref{sec:ps_asymptotic_consistency}, now evaluated under player $i$'s true mean
matrix $u_i$.

Fix player $i$ and let
\[
\mathcal S_{-i}:=\mathrm{supp}(\mu_i^0)\subseteq\mathcal F_{-i}
\]
be the finite retained opponents' menu used by player $i$ for posterior sampling.
For each public history $h^t$, define
\[
\mathcal E_i^{\mathrm{priv}}(h^t)
:=
\left\{
g_{-i}\in\mathcal S_{-i}:
V_i^{u_i}(\tau_i\mid h^t;g_{-i})
=
V_i^{u_i}(\tau_i\mid h^t;f_{-i})
\ \text{for every }\tau_i\in\mathcal F_i(h^t)
\right\}.
\]
Thus $\mathcal E_i^{\mathrm{priv}}(h^t)$ is the true continuation-payoff equivalence class
in the private-payoff extension after public history $h^t$.

\begin{assumption}[Private-payoff analogue of Assumption~\ref{ass:equiv_class_conc}]
\label{ass:equiv_class_conc_private}
Using the retained opponents' menu $\mathcal S_{-i}$ above, assume:
\begin{enumerate}[leftmargin=*]
\item \textit{(Menu grain of truth)} $f_{-i}\in\mathcal S_{-i}$ and
$\mu_i^0(f_{-i})>0$.
\item \textit{(On-path learnability up to equivalence)} Along $P^{\beta,u}$-almost every
realized path,
\[
\mu_i^t(\mathcal E_i^{\mathrm{priv}}(h^t(\omega))\mid h^t(\omega))\longrightarrow 1.
\]
\end{enumerate}
\end{assumption}

As before, Appendix~\ref{app:retained_menu_identification}
records stronger sufficient routes via retained-menu identification and deterministic hard
refutation of wrong retained public-action labels.

On the payoff side, the relevant object is again a continuation-decision class rather than
literal identification of the retained mean matrix.

\begin{definition}[On-path payoff equivalence]
\label{def:payoff_equiv_private}
Fix player $i$ and a public history $h^t$. Let
\[
\mathcal S_{-i}:=\mathrm{supp}(\mu_i^0)\subseteq \mathcal F_{-i}
\]
be player $i$'s retained opponents' menu, and let $g_{-i}^{i,t}$ be the posterior-predictive
representative continuation model from Section~\ref{sec:psbr_private_payoffs}. For candidate
mean matrices $m_i,m_i'\in\mathcal M_i$, write
\[
m_i \approx_i^{h^t} m_i'
\]
if for every opponents' continuation model
\[
g_{-i}\in \mathcal S_{-i}\cup\{g_{-i}^{i,t}\}
\]
the equality
\[
V_i^{m_i}(\tau_i\mid h^t;g_{-i})
=
V_i^{m_i'}(\tau_i\mid h^t;g_{-i})
\qquad
\text{for every }\tau_i\in\mathcal F_i(h^t)
\]
holds.
Equivalently, $m_i$ and $m_i'$ induce the same continuation optimization problem for every
retained opponents' model that can matter for PS-BR after $h^t$.
\end{definition}

\begin{assumption}[Pointwise public-history payoff learnability]
\label{ass:payoff_menu_sep_private}
Fix player $i$ and let
\[
\mathcal M_i:=\mathrm{supp}(\pi_i^0)
\]
be the finite retained mean-payoff menu used by player $i$ for posterior sampling. This
assumption concerns only player $i$'s own retained payoff menu, not the opponents' payoff
matrices and not the full space of payoff kernels. For each public history $h^t$, define
\[
\mathcal U_i(h^t)
:=
\{m_i\in\mathcal M_i: m_i\approx_i^{h^t} u_i\}.
\]
For each observable history $x_i^t=(h^t,r_i^{1:t-1})$, define the pointwise payoff-learning error
\[
\delta_i^t(x_i^t)
:=
1-\pi_i^t(\mathcal U_i(h^t)\mid x_i^t).
\]
Assume:
\begin{enumerate}[leftmargin=*]
\item \textit{(Menu grain of truth)} The true mean matrix $u_i\in\mathcal M_i$ and
$\pi_i^0(u_i)>0$.
\item \textit{(On-path concentration up to payoff equivalence)} Along
$P^{\beta,u}$-almost every realized path,
\[
\delta_i^t(x_i^t(\omega))\longrightarrow 0.
\]
\end{enumerate}
\end{assumption}

For every public history $h^t$, let
\[
\alpha_i^t(h^t)
:=
1-\mu_i^t(\mathcal E_i^{\mathrm{priv}}(h^t)\mid h^t).
\]
Thus $\alpha_i^t(h^t)$ is the opponents-side learning error, while
$\delta_i^t(x_i^t)$ is the own-payoff continuation-decision learning error at the realized
observable history $x_i^t=(h^t,r_i^{1:t-1})$.
Because continuation strategies are public-history based, payoff uncertainty affects the
player's continuation objective but not the induced predictive public-action law: for every
observable history $x_i^t=(h^t,r_i^{1:t-1})$,
\[
\Pi_i^t(\cdot\mid x_i^t)=\mu_{h^t}^{(f_i,g_{-i}^{i,t})}.
\]

\subsection{PS-BR gap and asymptotic consistency}

The asymptotic-consistency argument now has the same shape as before, but with two distinct
error terms. The opponent-side error $\alpha_i^t(h^t)$ measures posterior mass on retained
public continuation models that remain strategically distinct under the true mean matrix.
The payoff-side error $\delta_i^t(x_i^t)$ measures posterior mass outside the true own-payoff
continuation-decision class.

Let
\[
p_t(g_{-i},m_i)
:=
\mu_i^t(g_{-i}\mid h^t)\,\pi_i^t(m_i\mid x_i^t),
\qquad
(g_{-i},m_i)\in \mathcal S_{-i}\times \mathcal M_i.
\]
Define the joint collision complement
\[
D_i^{t,\mathrm{joint}}(x_i^t)
:=
1-\sum_{(g_{-i},m_i)\in \mathcal S_{-i}\times\mathcal M_i}
p_t(g_{-i},m_i)^2.
\]

\begin{lemma}[PS-BR is a $D_i^{t,\mathrm{joint}}$-best response to the mixed subjective value]
\label{lem:psarbr_gap}
Fix player $i$ and an observable history $x_i^t=(h^t,r_i^{1:t-1})$.
Let $\sigma_{i,t}^{\mathrm{PS}}$ be PS-BR from Definition~\ref{def:psarbr}. Then
\[
V_i^{\mathrm{mix},t}(\sigma_{i,t}^{\mathrm{PS}}\mid x_i^t)
\ge
\sup_{\tau_i\in\mathcal F_i(h^t)}
V_i^{\mathrm{mix},t}(\tau_i\mid x_i^t)
-
D_i^{t,\mathrm{joint}}(x_i^t).
\]
Equivalently, $\sigma_{i,t}^{\mathrm{PS}}$ is a
$D_i^{t,\mathrm{joint}}(x_i^t)$-best response to the mixed subjective continuation
value \eqref{eq:mix_value_def}.
\end{lemma}

Because continuation values are normalized to lie in $[0,1]$, and hypotheses inside
$\mathcal U_i(h^t)$ leave the relevant subjective continuation-value functional unchanged, for
every $\tau_i\in \mathcal F_i(h^t)$,
\begin{equation}
\big|
V_i^{\mathrm{mix},t}(\tau_i\mid x_i^t)
-
V_i^{u_i,t}(\tau_i\mid x_i^t)
\big|
\le
\delta_i^t(x_i^t).
\label{eq:mix_to_true_value_bound}
\end{equation}

\begin{lemma}[PS-BR gap under opponents' equivalence and payoff concentration]
\label{lem:psarbr_equiv_gap}
Fix player $i$ and an observable history $x_i^t=(h^t,r_i^{1:t-1})$.
Let $\sigma_{i,t}^{\mathrm{PS}}$ be PS-BR from Definition~\ref{def:psarbr}. Then
\[
V_i^{u_i,t}(\sigma_{i,t}^{\mathrm{PS}}\mid x_i^t)
\ge
\sup_{\tau_i\in\mathcal F_i(h^t)}V_i^{u_i,t}(\tau_i\mid x_i^t)
-
2\alpha_i^t(h^t)-4\delta_i^t(x_i^t).
\]
Equivalently, $\sigma_{i,t}^{\mathrm{PS}}$ is a
$\bigl(2\alpha_i^t(h^t)+4\delta_i^t(x_i^t)\bigr)$-best response to the true-mean
subjective continuation value \eqref{eq:true_mean_value_private}.
\end{lemma}

\begin{proposition}[PS-BR-selected continuation plans are asymptotically $\varepsilon$-consistent in the private-payoff game]
\label{prop:psarbr_implies_asymptotic_consistency}
Fix player $i$.
Assume:
(i) Assumption~\ref{ass:equiv_class_conc_private} holds for player $i$'s finite retained
menu of opponents' continuation models,
(ii) Assumption~\ref{ass:payoff_menu_sep_private} holds for player $i$'s own finite retained
mean-matrix menu, and
(iii) player $i$ uses PS-BR at every observable history.
Then for every $\varepsilon>0$,
\[
P^{\beta,u}\!\left(
\left\{
\omega:\exists\,T_i(\omega,\varepsilon)<\infty\ \text{s.t.}\ \forall t\ge T_i(\omega,\varepsilon),\
\sigma_{i,t}^{\mathrm{PS}}(\cdot\mid x_i^t(\omega))
\in
\mathrm{BR}_{i,u_i}^{\varepsilon}\!\bigl(g_{-i}^{i,t}\mid h^t(\omega)\bigr)
\right\}
\right)=1.
\]
\end{proposition}

\subsection{Zero-shot Nash convergence with private payoffs}

The zero-shot argument now lifts directly to the public continuation game at the abstract
level. Because the continuation strategies being compared are public-history based, the
player's observable history $x_i^t=(h^t,r_i^{1:t-1})$ affects future public play only
through belief updating. No additional assumption is needed to collapse hidden private
histories into a single public continuation law.

The theorem below is the abstract private-payoff analogue of
Theorem~\ref{thm:zero_shot_nash_final}: it concerns the induced public-history strategy
profile $f$ itself. Proposition~\ref{prop:psarbr_implies_asymptotic_consistency} should
therefore be read as a planner-side result about the continuation plan selected by PS-BR at
a realized observable history. To convert that selected-plan statement into an actual-behavior
result for a particular implementation, one needs the implementation to execute the selected
continuation plan as the continuation after the current public history.

Because continuation play is public-history based, no separate observable-process grain-of-truth
assumption is needed in this section. The menu grain-of-truth part of
Assumption~\ref{ass:equiv_class_conc_private} already implies the ordinary public grain-of-truth
condition: since $f_{-i}\in\mathcal S_{-i}$ with $\mu_i^0(f_{-i})>0$, the prior predictive public
play-path law under $f_i$ assigns positive mass to the true continuation law.

\begin{lemma}[Menu grain of truth implies strong public-path prediction in the private-payoff game]
\label{lem:absolute_cont_predict_private}
Fix player $i$. Under the menu grain-of-truth part of
Assumption~\ref{ass:equiv_class_conc_private},
\[
d\!\left(
\Pi_i^t(\cdot\mid x_i^t(\omega)),
\mu_{h^t(\omega)}^f
\right)\longrightarrow 0
\qquad\text{for }P^{\beta,u}\text{-a.e. }\omega.
\]
\end{lemma}

The proof is deferred to Appendix~\ref{app:omitted_proofs}.

\begin{definition}[Weak subjective equilibrium in the private-payoff game]
\label{def:weak_subjective_eq_private}
Fix $\xi,\eta\ge 0$ and a public history $h^t$.
A continuation profile $g\big|_{h^t}$ is a weak $\xi$-subjective $\eta$-equilibrium after $h^t$
if, for every player $i$, there exists a supporting profile $g^i=(g_i,g_{-i}^i)$ such that
\begin{enumerate}[leftmargin=*]
\item \textit{(Subjective best response)}\;
$g_i\big|_{h^t}\in \mathrm{BR}_{i,u_i}^\xi(g_{-i}^i\mid h^t)$.
\item \textit{(Weak predictive accuracy)}\;
$d_{h^t}(\mu^g,\mu^{g^i})\le \eta$.
\end{enumerate}
\end{definition}

\begin{proposition}[Learning and asymptotic consistency imply weak subjective equilibrium in the private-payoff game]
\label{prop:weak_subjective_from_learning_private}
Suppose that, for every player $i$ and every $\xi>0$,
\[
P^{\beta,u}\!\left(
\left\{
\omega:\exists\,T_i(\omega,\xi)<\infty\ \text{s.t.}\ \forall t\ge T_i(\omega,\xi),\
f_i\big|_{h^t(\omega)}
\in
\mathrm{BR}_{i,u_i}^{\xi}\!\bigl(g_{-i}^{i,t}\mid h^t(\omega)\bigr)
\right\}
\right)=1,
\]
and suppose Lemma~\ref{lem:absolute_cont_predict_private} holds for every player.
Then for every $\xi>0$ and $\eta>0$,
\[
P^{\beta,u}\!\left(
\left\{
\omega:\exists\,T(\omega)<\infty\ \text{s.t.}\ \forall t\ge T(\omega),\
f\big|_{h^t(\omega)}
\text{ is a weak $\xi$-subjective $\eta$-equilibrium after }h^t(\omega)
\right\}
\right)=1.
\]
\end{proposition}

\begin{theorem}[Zero-shot Nash convergence with private payoffs]
\label{cor:psarbr_zero_shot_nash}
Assume that, for every player $i$ and every $\xi>0$,
\[
P^{\beta,u}\!\left(
\left\{
\omega:\exists\,T_i(\omega,\xi)<\infty\ \text{s.t.}\ \forall t\ge T_i(\omega,\xi),\
f_i\big|_{h^t(\omega)}
\in
\mathrm{BR}_{i,u_i}^{\xi}\!\bigl(g_{-i}^{i,t}\mid h^t(\omega)\bigr)
\right\}
\right)=1,
\]
and suppose Lemma~\ref{lem:absolute_cont_predict_private} holds for every player.
Then for every $\varepsilon>0$,
\begin{align*}
P^{\beta,u}\!\Big(
\big\{\omega:\exists\,T(\omega)<\infty\ \text{s.t.}\ \forall t\ge T(\omega),\
&\exists\ \hat f^{\varepsilon,t,\omega}\ \text{an $\varepsilon$-Nash equilibrium of the continuation game}
\\
&\text{after }  h^t(\omega)\text{ with}\ 
d_{h^t(\omega)}\!\left(
\mu^f,
\mu^{\hat f^{\varepsilon,t,\omega}}
\right)
\le \varepsilon
\big\}
\Big)=1.
\end{align*}
\end{theorem}

Theorem~\ref{cor:psarbr_zero_shot_nash} is the direct private-payoff analogue of
Theorem~\ref{thm:zero_shot_nash_final}. Because the continuation game remains public-history
based, private payoff observations affect the theorem only through the player's posterior over
its own retained mean-payoff menu. The payoff-learning problem remains entirely own-payoff:
each player learns its own retained mean-payoff menu only up to the on-path continuation-decision
class relevant for continuation optimality, not the opponents' payoff matrices and not
necessarily every off-path entry of its own matrix.

\section{Experiments}\label{sec:experiments}

In this section, we empirically evaluate whether off-the-shelf reasoning LLM agents exhibit the theoretical properties derived in previous sections, i.e., whether they converge toward Nash equilibrium behavior in repeated strategic interaction. After discussing the experiment setup common to all simulations in Section~\ref{ssec:expSetup}, we study three empirical implications of the theory:
\begin{enumerate}[leftmargin=*]
    \item For convergence to some equilibrium-consistent late-run behavior, simple predict--then--act reasoning may already be sufficient (Section~\ref{sec:myopic}).
    \item For convergence to a \emph{particular} nontrivial repeated-game cooperative equilibrium sustained by continuation incentives, myopic approaches should generally fail, whereas PS-BR should succeed once equilibrium-selection frictions are reduced (Section~\ref{ssec:deterministicExp}).
    \item PS-BR should remain effective even when the payoff matrix is not given ex ante and must be learned from noisy private payoff observations (Section~\ref{ssec:stochasticExp}).
\end{enumerate}

\subsection{Setup}\label{ssec:expSetup}

\paragraph{Baselines.}
We use Qwen 3.5-27B \citep{qwen3.5}, a small-scale open-source LLM. The empirical design is a controlled within-model intervention: we hold fixed the underlying model, game description, public history, role mapping, and strategy-label context, and vary only the reasoning procedure applied to that common information. Thus, the main comparisons are intended to isolate the effect of adding reasoning on top of the same model and task interface, rather than differences in model scale, training, or task-specific prompt optimization. We evaluate three decision procedures:
\begin{itemize}[leftmargin=*]
    \item \textit{Base}: direct action selection from the common rules, history, and strategy-label context (prompt templates in Appendix~\ref{app:prompts_base}). To isolate the behaviors of LLM without reasoning, we explicitly suppress the default reasoning capabilities of Qwen 3.5-27B. 
    
    \item \textit{SCoT}: chain-of-thought style ``predict--then--act'' prompting \citep{akata2025playing} added on top of Base. This is a natural empirical benchmark for the one-step reasoning result in Section~\ref{sec:myopic_stage_nash}; it uses the same rules, history, and strategy-label context, but separates opponent-action prediction from action choice (Appendix~\ref{app:prompts_scot}).
    \item \textit{PS-BR}: Reasonable reasoning (posterior-sampling best response) added on top of \textit{Base} (Definition~\ref{def:psbr}; implementation details in Appendix~\ref{app:impl_strategy_psbr}). It uses the same game, history, and strategy-label information to infer an opponent strategy label, then chooses actions by rollout-based strategy evaluation; the known-payoff prompt is in Appendix~\ref{app:prompts_psbr_known}, and the unknown-payoff prompt/configuration is in Appendix~\ref{app:prompts_psbr_unknown}.
\end{itemize}

\paragraph{Benchmarks.}

We consider five repeated-game environments in total: BoS, PD, Promo, Samaritan, and Lemons. 

\paragraph{(1) Battle of the Sexes (BoS; coordination with asymmetric equilibria).}
Actions each period: $J$ or $F$.
Per-period payoff matrix (Player 1, Player 2):
\[
\begin{array}{c|cc}
 & \text{P2: }J & \text{P2: }F\\
\hline
\text{P1: }J & (10,7) & (0,0)\\
\text{P1: }F & (0,0) & (7,10)
\end{array}
\]
The non-trivial cooperative Nash equilibrium (pure): $(J,J)$ and $(F,F)$.
One non-trivial cooperative Nash equilibrium is both of them sticking to one action:
\begin{itemize}[leftmargin=*]
\item Play $J$ after every history (outcome $(J,J)$ every period).
\item Play $F$ after every history (outcome $(F,F)$ every period).
\end{itemize}
Such a non-trivial cooperative Nash equilibrium is particularly plausible when a monetary transfer underlies the game. Another non-trivial cooperative Nash equilibrium is turn-taking:
\begin{itemize}[leftmargin=*]
\item Play $(J,J)$ in odd periods and $(F,F)$ in even periods.
\item After any history, continue the same odd/even phase convention.
\end{itemize}
\paragraph{(2) Prisoner's Dilemma (PD; social dilemma).}
Actions each period: $J$ or $F$.
Per-period payoff matrix (Player 1, Player 2):
\[
\begin{array}{c|cc}
 & \text{P2: }J & \text{P2: }F\\
\hline
\text{P1: }J & (3,3) & (-5,5)\\
\text{P1: }F & (5,-5) & (0,0)
\end{array}
\]
One-shot stage-game Nash equilibrium: $(F,F)$.
A baseline pure Nash equilibrium of the repeated game is stationary play of $(F,F)$ after every history.
A nontrivial cooperative benchmark strategy profile used in the empirical PS-BR implementation is a contrite finite-punishment rule:
\begin{itemize}[leftmargin=*]
\item Cooperative phase: play $(J,J)$ every period.
\item If a player chooses $J$ while the opponent chooses $F$, the exploited player switches to $F$ for three rounds and then returns to $J$.
\item A player who previously chose $F$ against $J$ does not counter-retaliate during the opponent's three-round retaliation window; after that window ends, both players return to $(J,J)$.
\end{itemize}

\paragraph{(3) Promo \citep[Appendix \ref{appssec:promo}]{lal1990price}}
Actions each period: $R$ (Regular), $P$ (Promotion), or $Z$ (price-war punishment).
Per-period payoff matrix (Player 1, Player 2):
\[
\begin{array}{c|ccc}
 & \text{P2: }R & \text{P2: }P & \text{P2: }Z \\
\hline
\text{P1: }R & (1,1) & (-1,4) & (-2,-2) \\
\text{P1: }P & (4,-1) & (0,0) & (-2,-2) \\
\text{P1: }Z & (-2,-2) & (-2,-2) & (-2,-2)
\end{array}
\]
One-shot stage-game Nash equilibrium (pure): $(P,P)$. A baseline pure Nash equilibrium of the repeated game is the stationary play of $(P,P)$ after every history. A nontrivial cooperative pure Nash equilibrium described in \cite{lal1990price} is:
\begin{itemize}[leftmargin=*]
    \item Cooperative phase: $(P, R)$ in the odd round, and $(R,P)$ in the even round. 
    \item If the opponent deviates from the cooperation, play $Z$ for two periods and revert to the cooperative phase. 
\end{itemize}

\paragraph{(4) Samaritan (altruism / one-sided moral hazard).}
Player 1 (Helper): Help ($H$) or No-help ($N$).
Player 2 (Recipient): Work ($W$) or Shirk ($S$).
Per-period payoff matrix (Helper, Recipient):
\[
\begin{array}{c|cc}
 & \text{Recipient: }W & \text{Recipient: }S\\
\hline
\text{Helper: }H & (2,-1) & (0,0)\\
\text{Helper: }N & (1,-2) & (-1,-3)
\end{array}
\]
One-shot stage-game Nash equilibrium (pure): $(H,S)$.
The helper has a dominant action (help), and the recipient best responds by shirking.
A nontrivial cooperative Nash equilibrium exists for sufficiently patient players:
\begin{itemize}[leftmargin=*]
\item Cooperative phase: play $(H,W)$ every period.
\item If the recipient ever shirks, switch forever to punishment $(N,W)$.
\item If, during punishment, the helper ever deviates by helping, the recipient switches forever to $(H,S)$ behavior.
\end{itemize}

\paragraph{(5) Lemons (adverse selection).}
Player 1 (Seller): High Quality ($HQ$) or Low Quality ($LQ$).
Player 2 (Buyer): Buy ($B$) or Don't buy ($D$).
Per-period payoff matrix (Seller, Buyer):
\[
\begin{array}{c|cc}
 & \text{Buyer: }B & \text{Buyer: }D\\
\hline
\text{Seller: }HQ & (3,3) & (-1,0)\\
\text{Seller: }LQ & (4,-1) & (0,0)
\end{array}
\]
One-shot stage-game Nash equilibrium (pure): $(LQ,D)$.
Seller has strict dominant action $LQ$; buyer best-responds to $LQ$ with $D$.
A baseline pure Nash equilibrium of the repeated game is the stationary play of $(LQ,D)$ after every history.
A nontrivial cooperative benchmark strategy profile used in the empirical PS-BR implementation is:
\begin{itemize}[leftmargin=*]
\item Cooperative phase: play $(HQ,B)$ every period.
\item If the seller chooses $HQ$ while the buyer chooses $D$, the seller switches to $LQ$ for three rounds and then returns to $HQ$.
\item If the buyer chooses $B$ while the seller chooses $LQ$, the buyer switches to $D$ for three rounds and then returns to $B$.
\item A player who caused the other's retaliation window accepts that three-round retaliation and does not counter-retaliate during it; after the window ends, play returns to $(HQ,B)$.
\end{itemize}

\paragraph{Interaction protocol.}
In experiments 1-3, each run uses two fresh copies of the same model (Qwen 3.5 27B, \citet{qwen3.5}) in symmetric self-play. 

Within a run, actions are chosen simultaneously, all actions are perfectly publicly observed, and no communication channel is available beyond the public action/payoff history. Across runs, model weights and prompts are reset. Hence, any learning is purely in-context within the realized supergame. For each game, model configuration, and horizon treatment, we run 20 independent supergames. This yields a distribution of outcomes for each specification rather than a single realized history.

\paragraph{Infinitely repeated game implementation.}
Our main implementation follows the standard laboratory approach to simulate infinitely repeated games: \emph{random termination} with a guaranteed 200-round prefix. In benchmark game $g$, the first 200 rounds are guaranteed; at the end of round 200 and after each subsequent round, the supergame continues with commonly known probability $\delta_g=0.75$ and terminates otherwise. We choose this continuation probability so that the analytical benchmark equilibria used in Section~\ref{sec:deltaCalculation} are sequentially rational in the indefinitely repeated game, while the empirical cooperative target profiles also face strong continuation incentives. Such random termination is the standard benchmark implementation of an infinitely repeated game, and matched finite controls are the cleanest way to isolate continuation incentives \citep{bo2005cooperation,frechette2017infinitely}.

\subsection{Experiment 1. Nash convergence}\label{sec:myopic}

Here, we test the first hypothesis: that simple predict--then--act reasoning may already be sufficient for convergence to some equilibrium-consistent late-run behavior.

\subsubsection{Experiment design}

In Section~\ref{sec:myopic_stage_nash}, we showed that if agents myopically learn to predict opponents' next actions and then best respond to those predictions, the realized path eventually converges to a stage-game $\varepsilon$-Nash equilibrium. SCoT \citep{akata2025playing} operationalizes precisely such a predict--then--act rule.

Operationally, the three methods receive the same task information: game rules, compact public action history, action set, allowed strategy labels and descriptions, role mapping, and the same strategy-prior line when that treatment is enabled. They differ only in how this common information is used. \textit{Base} asks for the next action directly from the enriched action prompt (Appendix~\ref{app:prompts_base}). \textit{SCoT} first asks for the opponent's next action using the same strategy-context block and then asks for an action conditional on that prediction (Appendix~\ref{app:prompts_scot}). \textit{PS-BR} uses the analogous strategy-label inference prompt to sample an opponent strategy label and then selects an action by rollout-based strategy evaluation (Appendix~\ref{app:prompts_psbr_known}).

For comparability across specifications, we retain a late-run equilibrium-follow statistic as one reported outcome. Because the random-termination treatment logs a guaranteed 200-round prefix, we evaluate late-run play over rounds $190$--$200$. In each round of that window, we record whether the realized joint action coincides with either a one-shot Nash action or an on-path action of the benchmark cooperative repeated-game equilibrium for that game. We then average these indicators within run and across runs.

The matched finite-horizon control plays a separate role. If a method's apparent success in the indefinite treatment merely reflects convergence in long but finitely repeated interaction, then similar success should appear in the matched finite game. If, instead, a gap emerges in favor of random termination, that gap is evidence that continuation incentives matter for the method's behavior.

Our primary comparison is therefore not only across \textit{Base}, \textit{SCoT}, and \textit{PS-BR}, but also across the indefinite and matched finite implementations for each game.

\subsubsection{Results}

\begin{table}[ht!]
\centering
\caption{Equilibrium-follow percentage in the terminal window (rounds 190--200) for any (one-shot Nash or cooperative on-path action) Nash equilibrium. Reported scores are averaged over 20 trials.}
\label{tab:myopic_results}
\begin{tabular}{lccc}
\toprule
Game & Base & SCoT & PS-BR \\
\midrule
BoS & \textbf{100.0}\% & \textbf{100.0}\% & \textbf{100.0}\% \\
PD & \textbf{100.0}\% & \textbf{100.0}\% & \textbf{100.0}\% \\
Promo & \textbf{100.0}\% & \textbf{100.0}\% & \textbf{100.0}\% \\
Samaritan & \textbf{100.0}\% & \textbf{100.0}\% & \textbf{100.0}\% \\
Lemons & 0.0\% & \textbf{100.0}\% & \textbf{100.0}\% \\
\bottomrule
\end{tabular}
\end{table}

Table~\ref{tab:myopic_results} shows that on the broad ``any Nash'' metric, terminal-window performance under random termination remains near ceiling in most environments. It deliberately credits two distinct kinds of equilibrium-consistent play: one-shot Nash actions and on-path actions of the benchmark cooperative repeated-game equilibrium. In PD, for example, both mutual defection and mutual cooperation count as equilibrium-consistent under this metric; in BoS, coordinating on either diagonal counts as success. In these runs, \textit{PS-BR} and \textit{SCoT} reach 100.0\% in all five games, while \textit{Base} reaches 100.0\% in BoS, PD, Promo, and Samaritan. This pattern supports the paper's first, weaker empirical prediction: myopic stage-game equilibrium consistency is relatively easy for the model to reach, especially once the task is only to settle on some equilibrium-consistent late-run action profile rather than to sustain a particular repeated-game path.

The \textit{SCoT} result is closely aligned with the theory in Section~\ref{sec:myopic_stage_nash}: explicit myopic predict--then--act reasoning is sufficient for convergence to an ``any Nash'' outcome, including a myopic stage-game Nash outcome. More surprisingly, the strong \textit{Base} performance on this permissive metric suggests that the underlying model may already inherently encode at least the myopic strategic-equilibrium concept, even without the explicit SCoT reasoning. This aligns with the fact that recent LLM models are often trained under reasoning tasks and can implicitly reason, i.e., without explicit test-time reasoning steps \citep{wang2024grokking}, despite the fact that implicit reasoning does not in general exhibit deep reasoning \citep{lin2025implicit}.

Note that Table~\ref{tab:myopic_results} by itself does not identify whether that behavior can support richer continuation-level reasoning beyond myopic strategic one-stage reasoning. Because the outcome metric is deliberately ``any Nash,'' high scores may reflect only convergence to locally stable stage-game play rather than understanding of the cooperative repeated-game equilibrium. That distinction is the object of Experiment~2.

\subsection{Experiment 2. Nontrivial Nash convergence}\label{ssec:deterministicExp}

We now move from asking whether play converges to \emph{some} equilibrium-consistent action profile to the harder question of whether agents can track a \emph{specific} nontrivial cooperative repeated-game equilibrium sustained by continuation incentives.

\subsubsection{Experiment design}

The repeated-games literature shows that cooperation depends critically on strategic uncertainty and on beliefs about the opponent's strategy, not merely on the existence of a cooperative equilibrium \citep{dal2018determinants,dal2025coordination}. Likewise, elicited beliefs over actions and over supergame strategies help rationalize repeated-game choices \citep{aoyagi2024beliefs}. We therefore propose the prompt-induced cooperative-target treatment as a \emph{belief-conditioned implementation test}: the prompt reduces selection frictions and asks whether the model can represent and execute a specific history-contingent cooperative target once favorable beliefs are supplied. 

This design follows two lessons from the repeated-games literature.
First, behavior in repeated games depends strongly on experience and on strategic uncertainty; whether cooperation is supportable in equilibrium is necessary but not sufficient to generate high cooperation \citep{dal2011evolution,dal2018determinants,dal2025coordination}. Second, beliefs over opponents' actions and strategies are central objects in repeated games, so experiments that manipulate or diagnose those beliefs can reveal whether observed play reflects genuine repeated-game reasoning rather than incidental path matching \citep{dal2019strategy, aoyagi2024beliefs}.

Concretely, for each game, we provide a prompt that specifies one benchmark cooperative repeated-game strategy profile and asks the agent to expect the opponent may follow that profile strongly; this information enters through the strategy-context block in the Base and SCoT templates (Appendices~\ref{app:prompts_base} and \ref{app:prompts_scot}) and through the PS-BR strategy-label inference prompt (Appendix~\ref{app:prompts_psbr_known}). In PD, for example, the prompt specifies sustained cooperation, followed by a contrite three-round retaliate-and-forgive response after exploitative defection. In Promo, it specifies the alternating $(P,R),(R,P),(P,R),\ldots$ cooperative phase with the corresponding punishment rule. Analogous game-specific cooperative strategy profiles are used in BoS, Samaritan, and Lemons, aligned with the benchmark constructions in Section~\ref{ssec:expSetup}; in Lemons, the benchmark is the analogous three-round retaliate-and-forgive reputation rule. Such a prompt reduces strategic uncertainty, thereby making a specific cooperative target salient \citep{dal2018determinants,dal2025coordination}. The score in this experiment checks whether the realized joint action matches the prescribed cooperative action profile for that round on the benchmark path, then averages those indicators within run and across runs. Under this interpretation, success in Experiment~2 should be read as evidence of \emph{target-path implementation conditional on favorable beliefs}.

\subsubsection{Results.}

\begin{table}[ht!]
\centering
\caption{Equilibrium-follow percentage in the terminal window (rounds 190--200) for the prompt-specified nontrivial cooperative target path. Reported scores are averaged over 20 trials.}
\label{tab:nontrivial_results}
\begin{tabular}{lccc}
\toprule
Game & Base & SCoT & PS-BR \\
\midrule
BoS & 0.0\% & 0.0\% & \textbf{100.0}\% \\
PD & 0.0\% & 0.0\% & 99.5\% \\
Promo & 0.0\% & 0.0\% & \textbf{100.0}\% \\
Samaritan & 0.0\% & 0.0\% & 99.5\% \\
Lemons & 0.0\% & 0.0\% & \textbf{100.0}\% \\
\bottomrule
\end{tabular}
\end{table}

Table~\ref{tab:nontrivial_results} shows that once the metric is restricted to the prompt-specified cooperative action path, PS-BR runs are near ceiling in all five benchmark games. \textit{PS-BR} reaches 100.0\% in BoS, Promo, and Lemons, and 99.5\% in PD and Samaritan. By contrast, \textit{Base} and \textit{SCoT} reach 0.0\% in every game.

This contrast with Table~\ref{tab:myopic_results} is the paper's main empirical test. Because the cooperative-target prompt reduces equilibrium-selection frictions by making one favorable strategy profile salient, the remaining performance gap isolates whether the agent can actually represent the opponent's strategy and act on continuation incentives rather than merely match a locally safe action. Once the evaluation requires tracking a specific history-contingent cooperative prescription rather than merely landing on some equilibrium-consistent action, the near-ceiling performance from Experiment~1 disappears for the simpler methods. Under the rounds 190--200 terminal-window criterion, only \textit{PS-BR} achieves near-perfect scores across all five environments.

This finding supports the theory's negative implication for myopic reasoning. Whether myopic reasoning is made explicit through \textit{SCoT} or appears implicitly in \textit{Base}, it can be enough to reach some stage-game Nash-consistent behavior, but it does not guarantee understanding of the long-term cooperation concept sustained by continuation values. The failure of both baselines on the cooperative-path metric therefore clarifies that the high ``any Nash'' scores in Experiment~1 should not be interpreted as evidence of cooperative repeated-game reasoning.

More broadly, Experiment~2 directly supports the paper's core argument that off-the-shelf reasoning agents can exhibit zero-shot repeated-game stability without post-training when they can infer opponent strategies and evaluate continuation plans. Table~\ref{tab:myopic_results} showed that most methods can reach some equilibrium-consistent late-run behavior, mostly the stage-Nash equilibrium. Table~\ref{tab:nontrivial_results} shows that maintaining a cooperative repeated-game prescription is much less robust, and that the ability to do so sharply favors \textit{PS-BR}. The fact that the same advantage appears not only in BoS and PD but also in Promo, Samaritan, and Lemons ties the result back to the paper's motivating market-style environments, suggesting that the theory matters beyond stylized matrix games and into promotion, moral-hazard, and adverse-selection settings.

\subsection{Experiment 3: Nontrivial Nash convergence under unknown payoffs}\label{ssec:stochasticExp}

\subsubsection{Setup}

Experiment~3 keeps the benchmark games, cooperative target profiles, and horizon treatments from Experiment~2, but removes common-knowledge stage payoffs. The question is now whether the same belief-conditioned cooperative-target benchmark can be implemented when agents must simultaneously infer payoff incentives from noisy private observations.

As in Experiment~2, we use the standard random-termination implementation as the main treatment and the matched finite-horizon game as the control. Agents still observe the full public action history, but they do \emph{not} receive the stage-game payoff matrix in the prompt. Instead, after public joint action $a^t$ is realized, player $i$ receives only a private payoff observation
\begin{equation}
r_i^t \;=\; u_i^g(a^t) + \epsilon_{i,t},
\qquad
\epsilon_{i,t}\stackrel{\text{i.i.d.}}{\sim}\mathcal N(0,\sigma_g^2),
\label{eq:exp2_gaussian_noise_rewrite}
\end{equation}
independent across players and rounds.

To mirror the finite payoff-menu setup used in Section~\ref{sec:unknown_payoffs}, we equip each agent with a finite hypothesis class over its own \emph{mean} payoff matrix. Fix a game $g$ and player $i$, and define the offset set
\[
K:=\{-2,-1.5,-1,-0.5,0,+0.5,+1,+1.5,+2\}.
\]
The finite menu of candidate mean matrices is
\[
\mathcal M_{i,g}
:=
\left\{
m:A\to\mathbb R:
m(a)=u_i^g(a)+k_a\sigma_g
\ \text{for each }a\in A,\ \text{with }k_a\in K
\right\}.
\]
In particular, the true mean matrix belongs to $\mathcal M_{i,g}$.

Operationally, player $i$ maintains a posterior over $\mathcal M_{i,g}$ using the Gaussian likelihood
\[
\pi_i^t(m\mid h^t,r_i^{1:t-1})
\propto
\pi_i^0(m)\prod_{s=1}^{t-1}\phi(r_i^s; m(a^s),\sigma_g^2),
\]
where $\phi(\cdot;\mu,\sigma_g^2)$ is the Gaussian density. PS-BR then samples one candidate mean matrix from this posterior and evaluates continuation strategies against the induced payoff kernel.

As discussed in Section~\ref{sec:unknown_payoffs}, this design is intended to justify only \emph{on-path, decision-relevant} payoff learning, not global recovery of the full payoff matrix. The prompt-specified cooperative target and its finite punishment branches concentrate play on a small recurrent set of joint actions. Hence, under the Gaussian noise family, any retained payoff hypothesis that remains strategically relevant but mis-specifies one of those recurrently reached profiles accumulates repeated likelihood loss under the realized private rewards and is eliminated, while hypotheses that differ only on unreached or decision-irrelevant profiles need not be separated. This is exactly the recurrent-separation verification route discussed in Section~\ref{sec:unknown_payoffs} and Appendix~\ref{app:retained_menu_identification}.

The prompt-induced cooperative-target benchmark is kept in action-strategy form but now omits the payoff matrix (see Appendix~\ref{app:prompts_psbr_unknown}; for the Base and SCoT baselines, we use the templates in Appendices~\ref{app:prompts_base} and \ref{app:prompts_scot} with the payoff matrix omitted from the rules text). Thus, the prompt still lowers strategic uncertainty about the opponent's intended continuation play, while leaving the agent to learn the payoff incentives that make those continuations attractive or unattractive. This design directly targets the practical case in which repeated-game agents must infer incentives from experience rather than from a fully specified stage game. 

As in Experiment~2, we evaluate both on-path follow-through and off-path equilibrium understanding, and we compare performance under random termination with the matched finite-horizon control.

\subsubsection{Results.}

We report two complementary terminal-window metrics under unknown stochastic payoffs: convergence to any Nash equilibrium action (Table~\ref{tab:unknown_any_nash_results}) and follow-through on the prompt-specified cooperative target path (Table~\ref{tab:unknown_coop_nash_results}).

\noindent
\begin{minipage}[t]{0.48\textwidth}
\centering
\captionof{table}{Unknown stochastic payoffs: equilibrium-follow percentage in the terminal window (rounds 190--200) for \textit{any} Nash equilibrium. Reported scores are averaged over 20 trials.}
\label{tab:unknown_any_nash_results}
\begin{tabular}{lccc}
\toprule
Game & Base & SCoT & PS-BR \\
\midrule
BoS & \textbf{100.0}\% & \textbf{100.0}\% & \textbf{100.0}\% \\
PD & \textbf{100.0}\% & \textbf{100.0}\% & \textbf{100.0}\% \\
Promo & \textbf{100.0}\% & 70.9\% & \textbf{100.0}\% \\
Samaritan & 69.1\% & 60.0\% & \textbf{90.5}\% \\
Lemons & 94.5\% & \textbf{100.0}\% & \textbf{100.0}\% \\
\bottomrule
\end{tabular}
\end{minipage}
\hfill
\begin{minipage}[t]{0.48\textwidth}
\centering
\captionof{table}{Unknown stochastic payoffs: equilibrium-follow percentage in the terminal window (rounds 190--200) for the prompt-specified cooperative target path. Reported scores are averaged over 20 trials.}
\label{tab:unknown_coop_nash_results}
\begin{tabular}{lccc}
\toprule
Game & Base & SCoT & PS-BR \\
\midrule
BoS & 0.0\% & 0.0\% & \textbf{78.6}\% \\
PD & 0.0\% & 15.0\% & \textbf{72.7}\% \\
Promo & 0.0\% & 3.6\% & \textbf{75.9}\% \\
Samaritan & 5.0\% & 4.5\% & \textbf{79.1}\% \\
Lemons & 10.0\% & 5.0\% & \textbf{75.0}\% \\
\bottomrule
\end{tabular}
\end{minipage}

On the broader ``any Nash'' metric (Table~\ref{tab:unknown_any_nash_results}), all three methods again remain strong in the terminal window, though performance is more uneven across games than in the known-payoff case. \textit{Base} reaches 100.0\% in BoS, PD, and Promo, and remains high in Lemons at 94.5\%, though it falls to 69.1\% in Samaritan. \textit{SCoT} attains 100.0\% in BoS, PD, and Lemons, but is weaker in Promo (70.9\%) and Samaritan (60.0\%). \textit{PS-BR} reaches 100.0\% in BoS, PD, Lemons, and Promo, and 90.5\% in Samaritan. Thus, as in the known-payoff case, payoff uncertainty does not prevent substantial equilibrium-consistent terminal behavior on this permissive metric, but it makes the non-planner baselines less uniformly reliable in Promo, Samaritan, and Lemons.

The stricter cooperative-path metric in Table~\ref{tab:unknown_coop_nash_results} again reveals a sharper distinction. Under unknown payoffs, the PS-BR runs perform best in all five games: BoS (78.6\%), PD (72.7\%), Promo (75.9\%), Samaritan (79.1\%), and Lemons (75.0\%). \textit{Base} is much weaker, with 0.0\% in BoS, PD, and Promo, 5.0\% in Samaritan, and 10.0\% in Lemons; \textit{SCoT} remains lower across all games, with 0.0\% in BoS, 15.0\% in PD, 3.6\% in Promo, 4.5\% in Samaritan, and 5.0\% in Lemons.

Experiment~3 also aligns with the theoretical results proposed in Section \ref{sec:unknown_payoffs}, suggesting that agents can extend the strategic-equilibrium reasoning observed in the idealized known-payoff setting to the more realistic unknown-payoff setting. The extension is neither as stable nor as fast as in Experiment~2, because agents must learn decision-relevant payoff incentives while also coordinating on the repeated-game path. This is consistent with the theory's payoff-learning extension: equilibrium reasoning can survive the removal of ex ante payoff knowledge, but speed of convergence may be much slower.

\section{Conclusion}

In this paper, we theoretically highlight the promising prospect that general-purpose AI agents can attain game-theoretic robustness through their inherent reasoning capabilities rather than through unrealistic, bespoke, unified training. By demonstrating that LLMs can evolve toward equilibrium behavior on the fly, we take a step toward safer and more autonomous multi-agent AI systems that remain effective across the myriad interactive scenarios they will encounter in the real world. The results bridge the gap between AI agents and classical game theory, indicating that the rich knowledge and inferential power of modern LLMs may be harnessed to meet longstanding challenges in multi-agent learning and interaction. Ultimately, enabling LLM-based agents to naturally exhibit equilibrium-like behavior during play not only advances our theoretical understanding of their behavior but also paves the way for their deployment in societally crucial domains that require reliable strategic decision-making.

\newpage

\printbibliography

\clearpage          
\onecolumn          
\appendix

\clearpage

\clearpage
\section*{Notation Summary}

\begingroup
\footnotesize
\setlength\LTleft{0pt}
\setlength\LTright{0pt}
\setlength{\tabcolsep}{4pt}
\renewcommand{\arraystretch}{1.08}

\begin{longtable}{@{}L{0.29\textwidth}L{0.67\textwidth}@{}}
\toprule
\textbf{Symbol} & \textbf{Meaning} \\
\midrule
\endfirsthead

\toprule
\textbf{Symbol} & \textbf{Meaning} \\
\midrule
\endhead

\bottomrule
\endfoot

\multicolumn{2}{@{}l}{\textbf{Repeated-game primitives}}\\
\cmidrule(l){1-2}
$I=\{1,\ldots,N\}$ & Finite set of players / AI agents. \\
$A_i,\quad A=\prod_{i\in I} A_i$ & Player-$i$ action set; $A$ is the joint action space. \\
$a^t=(a_1^t,\ldots,a_N^t)\in A$ & Joint action profile at round $t$. \\
$u_i:A\to[0,1]$ & Stage payoff in the common-knowledge benchmark; in the private-payoff extension, the true mean payoff matrix generating player $i$'s private stochastic rewards. \\
$\lambda_i\in(0,1)$ & Player-$i$ discount factor. \\
$\varphi_i,\ \Phi_i^\star$ & Stage-game minmax payoff and pure-action maxmin payoff used in the Non-MM$^\star$ condition. \\

\addlinespace[2pt]
\multicolumn{2}{@{}l}{\textbf{Histories, strategies, and induced laws}}\\
\cmidrule(l){1-2}
$h^t=(a^1,\ldots,a^{t-1})$ & Public history observed before round $t$. \\
$H^0,\ H^t,\ H,\ H^\infty$ & Empty history; histories of length $t-1$; all finite public histories; all infinite play paths. \\
$C(h)$ & Cylinder set of infinite paths having prefix $h$. \\
$f_i:H\to\Delta(A_i)$ & Public-history strategy of player $i$ in the benchmark repeated game. \\
$\mathcal F_i,\quad \mathcal F=\prod_{i\in I}\mathcal F_i$ & Individual and joint strategy spaces in the benchmark model. \\
$\mu^f$ & Play-path distribution induced by strategy profile $f$. \\
$\mu_{h^t}^{g}$ & Continuation distribution after history $h^t$ when continuation profile $g$ is played thereafter. \\
$U_i(f)$ & Objective discounted payoff of player $i$ under profile $f$. \\

\addlinespace[2pt]
\multicolumn{2}{@{}l}{\textbf{Beliefs and subjective continuation values}}\\
\cmidrule(l){1-2}
$\mu_i^0,\quad \mu_i^t(\cdot\mid h^t)$ & Player-$i$ prior and posterior over opponents' strategy profiles. \\
$P_i^{\mu_i,g_i}$ & Predictive play-path distribution induced by own strategy $g_i$ and belief $\mu_i$ over opponents' strategies. \\
$f_{-i}^{i},\quad f_{-i}^{i,t}$ & Representative continuation models for player $i$'s prior predictive and posterior predictive beliefs about opponents' play. \\
\makecell[l]{$V_i(\sigma_i\mid h^t;g_{-i})$\\ $V_i(\sigma_i\mid h^t)$} & Subjective continuation value of continuation strategy $\sigma_i$ against continuation model $g_{-i}$; shorthand $V_i(\sigma_i\mid h^t)=V_i(\sigma_i\mid h^t;f_{-i}^{i,t})$. \\
$\mathrm{BR}_i^\varepsilon(g_{-i}\mid h^t)$ & Set of $\varepsilon$-best-response continuation strategies at history $h^t$. \\
$g_{-i}\sim_i^{h^t}g'_{-i},\ \mathcal E_i(h^t)$ & On-path continuation-payoff equivalence; the true equivalence class inside the finite menu. \\
$\mu^f \ll P_i^{0,f_i}$ & Grain-of-truth condition: player $i$'s prior predictive does not rule out events that occur under the true play distribution. \\
$d(\mu,\nu),\quad d_{h^t}(\mu,\nu)$ & Weak distance between play-path distributions and its continuation version after history $h^t$. \\
$\xi,\ \eta,\ \varepsilon$ & Approximation tolerances for subjective best response, predictive accuracy, and Nash / best-response error. \\

\addlinespace[2pt]
\multicolumn{2}{@{}l}{\textbf{Posterior-sampling and one-step reasoning}}\\
\cmidrule(l){1-2}
$\mathcal S_{-i},\quad L_{i,T}(g_{-i};z)$ & Finite strategy menu of candidate opponent strategies in PS-BR; finite-history likelihood of hypothesis $g_{-i}$ along realized path $z$. \\
$p_t(g_{-i})=\mu_i^t(g_{-i}\mid h^t)$ & Posterior mass on opponent hypothesis $g_{-i}$. \\
$\sigma_{i,t}^{\mathrm{PS}}(\cdot\mid h^t)$ & Continuation strategy chosen by posterior-sampling best response (PS-BR) at history $h^t$. \\
$D_i^t(h^t)=1-\sum_{g_{-i}\in\mathcal S_{-i}} p_t(g_{-i})^2$ & Posterior collision complement; upper-bounds the generic PS-BR best-response gap under exact posterior concentration. \\
$D_{\mathrm{KL}}(p\|q)$ & Kullback--Leibler divergence between two distributions. \\
$q_i^t(\cdot\mid h^t),\quad \mathrm{br}_i^\varepsilon(q)$ & One-step posterior predictive belief over opponents' next joint action; stage-game $\varepsilon$-best-response set to next-action distribution $q$. \\
$g_{-i}\approx_i^{h^t}g'_{-i},\quad \mathcal E_i^{\mathrm{st}}(h^t),\quad \beta_i^t(h^t)$ & On-path stage-payoff equivalence; the true stage-equivalence class inside the finite menu; posterior mass outside that class. \\
$\alpha_{i,t}^{\mathrm{mPS}}(\cdot\mid h^t)$ & Mixed action induced by myopic PS-BR at history $h^t$. \\
$\hat a_{-i}^t(h^t),\quad b_i(\cdot)$ & SCoT MAP prediction of opponents' next action; deterministic pure best-response selector. \\

\addlinespace[2pt]

\multicolumn{2}{@{}l}{\textbf{Private-payoff extension}}\\
\cmidrule(l){1-2}
$r_i^t$ & Privately observed stochastic payoff of player $i$ at round $t$. \\
$\mathcal R_i,\quad \nu_i(\mathrm dr),\quad \psi_i(r;\mu)$ & Payoff space, dominating base measure, and known noise-family density indexed by mean $\mu$. \\
$x_i^t=(h^t,r_i^{1:t-1}),\quad X_i^t,\ X_i$ & Player-$i$ observable history used for payoff learning at time $t$, its time-$t$ state space, and the full observable-history space. \\
$\Omega,\quad P^{\beta,u}$ & Full sample space of public actions and private rewards; actual law induced by the observable-history behavioral rule $\beta$ and the true mean payoffs $u$. \\
$\beta_i:X_i\to\Delta(A_i)$ & Actual one-step behavioral rule in the private-payoff environment. \\
$U_i^{u_i}(g\mid h^t)$ & Objective continuation payoff after public history $h^t$ under the true mean matrix $u_i$. \\
$\mathcal M_i,\quad \pi_i^t(m_i\mid x_i^t)$ & Finite menu of candidate mean payoff matrices and player-$i$ posterior over that menu. \\
$q_i^{m_i}(\mathrm dr\mid a)$ & Payoff kernel induced by candidate mean matrix $m_i$ at joint action $a$. \\
\makecell[l]{$V_i^{m_i}(\tau_i\mid h^t;g_{-i})$\\ $\mathrm{BR}_{i,m_i}^\varepsilon(g_{-i}\mid h^t)$} & Subjective continuation value and corresponding $\varepsilon$-best-response set when player $i$ evaluates payoffs using candidate mean matrix $m_i$. \\
\makecell[l]{$V_i^{\mathrm{mix},t}(\tau_i\mid x_i^t)$\\ $V_i^{u_i,t}(\tau_i\mid x_i^t)$} & Mixture continuation value under the joint opponent/payoff posteriors; the same object evaluated under the true mean matrix $u_i$. \\
$g_{-i}^{i,t},\quad \Pi_i^t(\cdot\mid x_i^t)$ & Representative continuation model for player $i$'s posterior predictive public-action belief; posterior predictive law over future public-action paths. \\
$\mathcal E_i^{\mathrm{priv}}(h^t),\ \mathcal U_i(h^t),\ \alpha_i^t(h^t),\ \delta_i^t(x_i^t)$ & True opponents-side continuation-payoff equivalence class in the private-payoff game; true own-payoff continuation-decision class; opponents-side learning error; pointwise own-payoff continuation-decision learning error at observable history $x_i^t$. \\

\addlinespace[2pt]
\multicolumn{2}{@{}l}{\textbf{Bounded-memory appendix notation}}\\
\cmidrule(l){1-2}
$\mathrm{suffix}_\kappa(h),\ \mathcal F_i^\kappa,\ \mathsf S_\kappa,\ T_\kappa$ & Last $\kappa$ joint actions of history $h$; bounded-memory strategy space; finite suffix-state space; deterministic state-update map in the finite-state reduction. \\

\end{longtable}
\endgroup

\clearpage

\section{Predictive beliefs and representative continuation models}
\label{app:belief_repr}

The main text works directly with posterior predictive continuation laws. This appendix records the standard representative construction that justifies the notation $f_{-i}^{i}$ and $f_{-i}^{i,t}$ used there.

Fix player $i$ and a (possibly mixed) belief $\mu_i$ over opponents' strategy profiles $\mathcal F_{-i}$. For any own strategy $g_i\in\mathcal F_i$, define the predictive play-path distribution
\[
P_i^{\mu_i,g_i}(E)
:= \int_{\mathcal F_{-i}} \mu^{(g_i,f_{-i})}(E)\, d\mu_i(f_{-i})
\qquad\text{for measurable }E\subseteq H^\infty.
\]

\begin{lemma}[Existence of predictive representatives]
\label{lem:belief_repr_exists}
Fix player $i$ and a belief $\mu_i$ over opponents' strategy profiles $\mathcal F_{-i}$. There exists a behavior-strategy profile $\bar f_{-i}\in\mathcal F_{-i}$ such that, for every own strategy $g_i\in\mathcal F_i$,
\[
\mu^{(g_i,\bar f_{-i})} \,=\, P_i^{\mu_i,g_i}.
\]
When $\mu_i$ has finite support $\{g_{-i}^1,\dots,g_{-i}^K\}$, one convenient choice is
\[
\bar f_{-i}(h)(a_{-i})
=\sum_{k=1}^K \mu_i(g_{-i}^k\mid h)\, g_{-i}^k(h)(a_{-i}),
\]
for histories $h$ where Bayes' rule is defined.
\end{lemma}

Lemma~\ref{lem:belief_repr_exists} is the standard mixed-to-behavior reduction for extensive-form games with perfect recall; see \citet{kuhn1953extensive,aumann1961mixed,kalai1993rational}. We use it only as notation-saving infrastructure.

Applying Lemma~\ref{lem:belief_repr_exists} to the prior $\mu_i^0$ yields a representative continuation model $f_{-i}^i\in\mathcal F_{-i}$ such that, for every $g_i$,
\[
\mu^{(g_i,f_{-i}^i)}=P_i^{0,g_i}.
\]
At any history $h^t$ where Bayes' rule is defined, applying the same lemma to the posterior $\mu_i^t(\cdot\mid h^t)$ yields a representative continuation model $f_{-i}^{i,t}$. For notational convenience, we may choose these representatives continuation-consistently:
\begin{equation}
\label{eq:app_rep_choice}
f_{-i}^{i,t}\big|_{h^t} \;:=\; f_{-i}^{i}\big|_{h^t}.
\end{equation}
This is the selection convention used in the appendix proofs.

\section{Continuity and Finite-Horizon Robustness}
\label{app:patching}

\begin{lemma}[Continuity of discounted payoff]
\label{lem:continuity}
For each agent $i$ and every $\delta>0$, there exists $\rho_i(\delta)>0$ such that
for any strategy profiles $f,g \in \mathcal{F}$,
\[
d(\mu^f,\mu^g) \le \rho_i(\delta)
\quad\Rightarrow\quad
\bigl| U_i(f) - U_i(g) \bigr| \le \delta.
\]
In particular, if $\rho(\delta) = \min_{i\in I} \rho_i(\delta)$ and
$d(\mu^f,\mu^g)\le\rho(\delta)$, then
$\bigl| U_i(f) - U_i(g) \bigr| \le \delta$ for all $i \in I$.
\end{lemma}

\subsection{Finite-horizon variants and robustness}

For a finite horizon $T \in \mathbb{N}$, we denote by $\mathcal{F}^T$ the set of
behaviour strategies specified on histories of length at most $T$; two full
strategies that coincide on these histories induce the same distribution over
histories up to time $T$ and the same truncated payoff. For $f \in \mathcal{F}^T$,
define the $T$-period discounted payoff
\[
U_i^T(f)
=
\mathbb{E}_{z \sim \mu^f}
\Big[
(1-\lambda_i) \sum_{t=1}^T \lambda_i^{t-1} u_i(z^t)
\Big].
\]

\begin{definition}[Finite-horizon weak $\xi$-subjective $\eta$-equilibrium]
\label{def:finite_weak_subjective_eq}
Let $\xi,\eta \ge 0$ and a fixed horizon $T$.
A truncated strategy profile $f \in \mathcal{F}^T$ is a \emph{finite-horizon weak
$\xi$-subjective $\eta$-equilibrium} if for each agent $i \in I$ there exists a
supporting truncated profile $f^i \in \mathcal{F}^T$ such that:
\begin{itemize}[leftmargin=*]
    \item $f_i^i = f_i$;
    \item $U_i^T(f_i, f_{-i}^i) \ge \sup_{g_i \in \mathcal{F}_i^T} U_i^T(g_i, f_{-i}^i) - \xi$;
    \item $d(\mu^{f^i},\mu^f) \le \eta$ when $d$ is computed using only cylinder
    events in $\mathcal{B}^t$ with $t \le T$.
\end{itemize}
\end{definition}

We now show that finite-horizon weak subjective equilibria can be ``patched''
into approximate finite-horizon Nash equilibria without changing the induced
distribution of play up to time $T$.

\begin{lemma}[Finite-horizon purification for $\eta=0$ \cite{norman2022possibility}]
\label{lem:finite_eta0}
Fix a finite horizon $T$ and a profile $f \in \mathcal{F}^T$.
Suppose $f$ is a finite-horizon weak $\psi$-subjective $0$-equilibrium for some $\psi\ge 0$.
Then there exists a truncated strategy profile $\hat f \in \mathcal{F}^T$ such that:
\begin{itemize}[leftmargin=*]
    \item $\hat f$ is a $\psi$-Nash equilibrium of the $T$-period game, i.e.,
    for all $i\in I$ and all $g_i \in \mathcal{F}_i^T$,
    \[
    U_i^T(\hat f_i,\hat f_{-i})
    \;\ge\;
    U_i^T(g_i,\hat f_{-i}) - \psi;
    \]
    \item the induced distributions of histories of length at most $T$ coincide:
    for every $E \in \mathcal{B}^T$,
    $
    \mu^{\hat f}(E) = \mu^f(E)
    $.
\end{itemize}
\end{lemma}

We next extend this to the case where $\eta>0$ but small, using a compactness and limit argument.

\begin{lemma}[Finite-horizon robustness]
\label{lem:finite_eta_small}
Fix a finite horizon $T$ and $\psi>0$.
For every $\theta>0$ there exists $\bar\eta_T(\psi,\theta)>0$ such that:
if $f \in \mathcal{F}^T$ is a finite-horizon weak $\psi$-subjective $\eta$-equilibrium with
$\eta \le \bar\eta_T(\psi,\theta)$, then there exists a $\psi$-Nash equilibrium
$\hat f \in \mathcal{F}^T$ satisfying
\[
d(\mu^{\hat f},\mu^f) \le \theta
\]
(again with $d$ computed on cylinder events of length at most $T$).
\end{lemma}

We now patch finite-horizon robustness to the infinite-horizon game by truncating
the payoff at a sufficiently large horizon and using Lemma~\ref{lem:continuity};
the resulting infinite-horizon patching lemma is recorded below.

\begin{lemma}[Infinite-horizon patching]
\label{lem:infinite_patching}
Fix $\xi>0$ and $\varepsilon>0$.
There exists $\hat\eta(\xi,\varepsilon)>0$ such that if
$f \in \mathcal{F}$ is a weak $\xi$-subjective $\eta$-equilibrium in the sense of
Definition~\ref{def:weak_subjective_eq} with
$\eta \le \hat\eta(\xi,\varepsilon)$, then there exists a strategy profile
$\hat f \in \mathcal{F}$ satisfying:
\begin{itemize}[leftmargin=*]
    \item $\hat f$ is a $(\xi+\varepsilon)$-Nash equilibrium of the infinite-horizon game;
    \item $d(\mu^{\hat f},\mu^f) \le \varepsilon$.
\end{itemize}
\end{lemma}

\begin{remark}[Continuation-game analogue for the private-payoff extension]
\label{rem:private_patching}
Because the private-payoff section now works directly with public-history continuation games
after $h^t$, Lemmas~\ref{lem:finite_eta0}--\ref{lem:infinite_patching} already apply
directly. No separate hidden-history continuation-game variant is needed.
\end{remark}

\clearpage

\section{Proofs}
\label{app:omitted_proofs}

\begin{proof}[Proof of Lemma \ref{lem:continuity}]
Fix $i$ and $\delta>0$.
Choose a finite horizon $T\in\mathbb{N}$ large enough that
\begin{equation}
(1-\lambda_i) \sum_{t=T+1}^\infty \lambda_i^{t-1} 
\;\le\; \frac{\delta}{4}.
\label{eq:tail_bound}
\end{equation}
For any profile $g \in \mathcal{F}$, define the truncated payoff
\[
U_i^T(g)
=
\mathbb{E}_{z \sim \mu^g}
\left[
(1-\lambda_i)
\sum_{t=1}^T \lambda_i^{t-1} u_i(z^t)
\right].
\]
Then for any $g$ we have
\[
\bigl|U_i(g) - U_i^T(g)\bigr|
\le
(1-\lambda_i) \sum_{t=T+1}^\infty \lambda_i^{t-1}
\le \frac{\delta}{4}
\]
by \eqref{eq:tail_bound}, using that $u_i(\cdot)\in[0,1]$.

Now fix $f,g \in \mathcal{F}$. We can decompose
\[
\bigl|U_i(f)-U_i(g)\bigr|
\le
\bigl|U_i(f)-U_i^T(f)\bigr|
+
\bigl|U_i^T(f)-U_i^T(g)\bigr|
+
\bigl|U_i^T(g)-U_i(g)\bigr|.
\]
By the bound above, the first and third terms are each at most $\delta/4$.
It remains to control $|U_i^T(f)-U_i^T(g)|$.

For each $t\in\{1,\dots,T\}$ and each joint action profile $a\in A$,
let
\[
\alpha_t^f(a)
=
\mu^f\bigl(\{z\in H^\infty : z^t = a\}\bigr),
\quad
\alpha_t^g(a)
=
\mu^g\bigl(\{z\in H^\infty : z^t = a\}\bigr).
\]
Since $u_i(a)\in[0,1]$ for all $a$, we have
\[
\left|
\sum_{a\in A} u_i(a)\bigl(\alpha_t^f(a)-\alpha_t^g(a)\bigr)
\right|
\le
\sup_{E \in \mathcal{B}^t} \bigl|\mu^f(E) - \mu^g(E)\bigr|.
\]
Hence
\begin{align*}
\bigl|U_i^T(f)-U_i^T(g)\bigr|
&=
\left|
\sum_{t=1}^T (1-\lambda_i)\lambda_i^{t-1}
\sum_{a\in A} u_i(a)\bigl(\alpha_t^f(a)-\alpha_t^g(a)\bigr)
\right|\\
&\le
\sum_{t=1}^T (1-\lambda_i)\lambda_i^{t-1}
\sup_{E \in \mathcal{B}^t} \bigl|\mu^f(E) - \mu^g(E)\bigr|.
\end{align*}

By the definition~\eqref{def:weak_distance} of $d(\mu^f,\mu^g)$,
for each $t$ we have
\[
2^{-t}\sup_{E \in \mathcal{B}^t} \bigl|\mu^f(E) - \mu^g(E)\bigr|
\le d(\mu^f,\mu^g),
\]
hence
\[
\sup_{E \in \mathcal{B}^t} \bigl|\mu^f(E) - \mu^g(E)\bigr|
\le 2^t d(\mu^f,\mu^g).
\]
Thus
\[
\bigl|U_i^T(f)-U_i^T(g)\bigr|
\le
d(\mu^f,\mu^g)
\sum_{t=1}^T (1-\lambda_i)\lambda_i^{t-1} 2^t.
\]
The finite sum on the right depends only on $T$ and $\lambda_i$; call it $C_i(T)$.
Define
\[
\rho_i(\delta)
=
\min\left\{
\frac{\delta}{4 C_i(T)},\,
1
\right\}.
\]
If $d(\mu^f,\mu^g) \le \rho_i(\delta)$, then
\[
\bigl|U_i^T(f)-U_i^T(g)\bigr|
\le C_i(T)\rho_i(\delta)
\le \frac{\delta}{4}.
\]
Combining the three bounds gives
\[
\bigl|U_i(f)-U_i(g)\bigr|
\le
\frac{\delta}{4} + \frac{\delta}{4} + \frac{\delta}{4}
\;<\; \delta.
\]
Setting $\rho(\delta)=\min_{i\in I}\rho_i(\delta)$ yields the final claim.
\end{proof}

\begin{proof}[Proof of Lemma \ref{lem:finite_eta0}]
For each player $i$, let $f^i=(f_i,f_{-i}^i)$ be a supporting truncated profile
from Definition~\ref{def:finite_weak_subjective_eq}. Because $\eta=0$,
\[
\mu^{f^i}(E)=\mu^f(E)
\qquad\text{for every }E\in\mathcal B^T.
\]
Thus $f^i$ and $f$ induce the same distribution over histories of length at most
$T$.

For each player $i$, let $\mathcal D_i$ denote the set of histories $h^t$,
$t\le T$, such that along $h^t$ the earliest departure from the profile $f$ is a
unilateral deviation by player $i$. Define a truncated profile $\hat f$ as
follows:
\begin{itemize}[leftmargin=*]
    \item on histories that are still consistent with $f$, set $\hat f=f$;
    \item on every descendant of a history in $\mathcal D_i$, set $\hat f=f^i$;
    \item on any remaining off-path histories, choose arbitrary actions.
\end{itemize}
Because every history in $\bigcup_i \mathcal D_i$ is off-path under $f$, these
replacements do not change the play distribution when all players follow
$\hat f$. Hence $\mu^{\hat f}(E)=\mu^f(E)$ for all $E\in\mathcal B^T$, proving
item 2.

Now fix player $i$ and an arbitrary deviation $g_i\in\mathcal F_i^T$. Under the
profile $(g_i,\hat f_{-i})$, play coincides with $f$ until the first deviation by
$i$, and from that point onward the continuation of the other players is exactly
$f_{-i}^i$. Therefore
\[
U_i^T(g_i,\hat f_{-i})=U_i^T(g_i,f_{-i}^i).
\]
Similarly, since $\mu^{\hat f}=\mu^f=\mu^{f^i}$ on $\mathcal B^T$ and
$f_i^i=f_i=\hat f_i$,
\[
U_i^T(\hat f_i,\hat f_{-i})=U_i^T(f_i,f_{-i}^i).
\]
Using the defining inequality for the supporting profile $f^i$,
\[
U_i^T(\hat f_i,\hat f_{-i})
=
U_i^T(f_i,f_{-i}^i)
\ge
\sup_{g_i'\in\mathcal F_i^T} U_i^T(g_i',f_{-i}^i)-\psi
=
\sup_{g_i'\in\mathcal F_i^T} U_i^T(g_i',\hat f_{-i})-\psi.
\]
Since this holds for every player $i$, $\hat f$ is a $\psi$-Nash equilibrium of
the $T$-period game, proving item 1.
\end{proof}

\begin{proof}[Proof of Lemma \ref{lem:finite_eta_small}]
Suppose, towards a contradiction, that there exist $T,\psi>0$ and $\theta>0$ such that
for every $m\in\mathbb{N}$ there is a finite-horizon weak $\psi$-subjective $\eta_m$-equilibrium
$f^{(m)}\in\mathcal{F}^T$ with $\eta_m \le 1/m$ and such that
no $\psi$-Nash equilibrium lies within weak distance $\theta$ of $\mu^{f^{(m)}}$ (measured on
$\mathcal{B}^T$).

For each $m$ and each $i\in I$, let $f^{i,(m)}$ be a supporting truncated profile witnessing that
$f^{(m)}$ is a finite-horizon weak $\psi$-subjective $\eta_m$-equilibrium, i.e.,
$f_i^{i,(m)} = f_i^{(m)}$,
\[
U_i^T(f_i^{(m)},f_{-i}^{i,(m)})
\ge
\sup_{g_i \in \mathcal{F}_i^T} U_i^T(g_i,f_{-i}^{i,(m)}) - \psi,
\quad
d(\mu^{f^{i,(m)}},\mu^{f^{(m)}})\le \eta_m.
\]

Because the horizon $T$ and action sets are finite, the space of behaviour strategies
$\mathcal{F}^T$ is a finite-dimensional product of simplices and hence compact
in the product topology.
Thus, by sequential compactness, there exists a subsequence (which we relabel for
notational convenience) such that
\[
f^{(m)} \to f^\star
\quad\text{and}\quad
f^{i,(m)} \to f^{i,\star}
\quad\text{for all } i\in I,
\]
as $m\to\infty$, in the product topology on $\mathcal{F}^T$.

The map $f \mapsto \mu^f$ on finite histories (up to time $T$) is continuous with respect to this topology and the
weak topology induced by $d$ (restricted to $\mathcal{B}^T$), so
\[
\mu^{f^{(m)}} \to \mu^{f^\star},
\quad
\mu^{f^{i,(m)}} \to \mu^{f^{i,\star}}.
\]
Since $d(\mu^{f^{i,(m)}},\mu^{f^{(m)}})\le \eta_m \to 0$, we must have
$d(\mu^{f^{i,\star}},\mu^{f^\star})=0$, so $\mu^{f^{i,\star}}=\mu^{f^\star}$ on $\mathcal{B}^T$.

Moreover, the best-response inequality passes to the limit.
Fix $i$ and any $g_i \in \mathcal{F}_i^T$.
For all $m$,
\[
U_i^T(f_i^{(m)},f_{-i}^{i,(m)})
\ge
\sup_{g'_i \in \mathcal{F}_i^T} U_i^T(g'_i,f_{-i}^{i,(m)}) - \psi
\ge
U_i^T(g_i,f_{-i}^{i,(m)}) - \psi.
\]
By continuity of $U_i^T$ in the product topology (an immediate consequence of
Lemma~\ref{lem:continuity} restricted to horizon $T$), taking $m\to\infty$ yields
\[
U_i^T(f_i^\star,f_{-i}^{i,\star})
\ge
U_i^T(g_i,f_{-i}^{i,\star}) - \psi.
\]
Since $g_i$ was arbitrary and $f_i^{i,\star}=f_i^\star$ (by pointwise convergence
of $f_i^{i,(m)}$ to $f_i^{i,\star}$ and of $f_i^{(m)}$ to $f_i^\star$), we conclude that
\[
U_i^T(f_i^\star,f_{-i}^{i,\star})
\ge
\sup_{g_i \in \mathcal{F}_i^T} U_i^T(g_i,f_{-i}^{i,\star}) - \psi.
\]
Together with $d(\mu^{f^{i,\star}},\mu^{f^\star})=0$, this shows that $f^\star$
is a finite-horizon weak $\psi$-subjective $0$-equilibrium of the $T$-period game.

By Lemma~\ref{lem:finite_eta0}, there exists a profile $\hat f^\star\in\mathcal{F}^T$
such that $\hat f^\star$ is a $\psi$-Nash equilibrium of the $T$-period game and
$\mu^{\hat f^\star}$ coincides with $\mu^{f^\star}$ on histories of length at most $T$.
In particular, $d(\mu^{\hat f^\star},\mu^{f^\star})=0$.

Since $\mu^{f^{(m)}}\to\mu^{f^\star}$ in the weak metric $d$ (restricted to $\mathcal{B}^T$), we have
$d(\mu^{f^{(m)}},\mu^{\hat f^\star}) \to 0$ as $m\to\infty$.
Thus for all sufficiently large $m$, $d(\mu^{f^{(m)}},\mu^{\hat f^\star})\le\theta$.
But $\hat f^\star$ is a $\psi$-Nash equilibrium, contradicting the assumption that
no $\psi$-Nash equilibrium lies within weak distance $\theta$ of $\mu^{f^{(m)}}$.
This contradiction shows that such a sequence $(f^{(m)})$ cannot exist, and hence
there must exist $\bar\eta_T(\psi,\theta)>0$ with the stated property.
\end{proof}

\begin{proof}[Proof of Lemma \ref{lem:infinite_patching}]
Fix $\xi>0$ and $\varepsilon>0$.
Choose a finite horizon $T$ large enough that, for all $i\in I$ and all profiles $h\in\mathcal{F}$,
\begin{equation}
\bigl|
U_i(h) - U_i^T(h)
\bigr|
\;\le\; \frac{\varepsilon}{8},
\label{eq:trunc_choice2}
\end{equation}
and also
\begin{equation}
\sum_{t>T} 2^{-t} \;\le\; \frac{\varepsilon}{4}.
\label{eq:metric_tail2}
\end{equation}
Such a $T$ exists because the tails of both geometric series are uniformly small.

Let $f$ be a weak $\xi$-subjective $\eta$-equilibrium with supporting profiles
$\{f^i\}_{i\in I}$ as in Definition~\ref{def:weak_subjective_eq}, i.e.,
for each $i$,
\[
f_i^i = f_i,\quad
U_i(f_i,f_{-i}^i) \ge \sup_{g_i \in \mathcal{F}_i} U_i(g_i,f_{-i}^i) - \xi,
\quad
d(\mu^{f^i},\mu^f) \le \eta.
\]

Consider the truncated profiles $f^{(T)}$ and $(f^i)^{(T)}$ obtained by restricting
the prescriptions of $f$ and $f^i$ to histories of length at most $T$. For each $i$ we have
$(f_i^i)^{(T)} = f_i^{(T)}$ and, since the weak distance on histories up to $T$ is bounded by
the full weak distance,
\[
d(\mu^{(f^i)^{(T)}},\mu^{f^{(T)}}) \le d(\mu^{f^i},\mu^f) \le \eta.
\]

We now show that $f^{(T)}$ is a finite-horizon weak $\psi_T$-subjective $\eta$-equilibrium
for a slightly relaxed parameter $\psi_T$.
Fix $i$ and note that for any profile $h$,
\[
|U_i(h) - U_i^T(h)| \le \frac{\varepsilon}{8}
\]
by \eqref{eq:trunc_choice2}.
Using the weak subjective inequality for $f$ and $f^i$, we obtain
\begin{align*}
U_i^T(f_i^{(T)},(f_{-i}^i)^{(T)})
&= U_i^T(f_i,f_{-i}^i) \\
&\ge U_i(f_i,f_{-i}^i) - \frac{\varepsilon}{8} \\
&\ge \sup_{g_i \in \mathcal{F}_i} U_i(g_i,f_{-i}^i) - \xi - \frac{\varepsilon}{8}.
\end{align*}
For any truncated deviation $g_i^{(T)} \in \mathcal{F}_i^T$ we can extend it arbitrarily
to a full strategy $g_i \in \mathcal{F}_i$, and then
\[
U_i(g_i,f_{-i}^i) \ge U_i^T(g_i^{(T)},(f_{-i}^i)^{(T)}) - \frac{\varepsilon}{8},
\]
again by \eqref{eq:trunc_choice2}.
Taking the supremum over $g_i^{(T)}$ yields
\begin{align*}
U_i^T(f_i^{(T)},(f_{-i}^i)^{(T)})
&\ge \sup_{g_i^{(T)} \in \mathcal{F}_i^T} U_i^T(g_i^{(T)},(f_{-i}^i)^{(T)}) - \xi - \frac{\varepsilon}{4}.
\end{align*}
Thus, if we define
\[
\psi_T := \xi + \frac{\varepsilon}{4},
\]
then for each $i$ the truncated profiles $f^{(T)}$ and $(f^i)^{(T)}$ satisfy
\[
U_i^T(f_i^{(T)},(f_{-i}^i)^{(T)})
\ge
\sup_{g_i^{(T)}\in\mathcal{F}_i^T} U_i^T(g_i^{(T)},(f_{-i}^i)^{(T)}) - \psi_T,
\]
and $d(\mu^{(f^i)^{(T)}},\mu^{f^{(T)}})\le\eta$, so $f^{(T)}$ is a finite-horizon weak
$\psi_T$-subjective $\eta$-equilibrium in the sense of Definition~\ref{def:finite_weak_subjective_eq}.

Applying Lemma~\ref{lem:finite_eta_small} with this $T$, $\psi=\psi_T$ and
$\theta = \varepsilon/2$, there exists $\bar\eta_T(\psi_T,\varepsilon/2)>0$ such that
if $\eta \le \bar\eta_T(\psi_T,\varepsilon/2)$ then there is a $\psi_T$-Nash equilibrium
$\tilde f^{(T)}\in\mathcal{F}^T$ for the $T$-period game with
\[
d(\mu^{\tilde f^{(T)}},\mu^{f^{(T)}}) \le \frac{\varepsilon}{2}.
\]
Define
\[
\hat\eta(\xi,\varepsilon) := \bar\eta_T\bigl(\xi+\tfrac{\varepsilon}{4},\tfrac{\varepsilon}{2}\bigr).
\]
Assume henceforth that $\eta \le \hat\eta(\xi,\varepsilon)$ so that this conclusion holds.

Extend $\tilde f^{(T)}$ arbitrarily to a full strategy profile $\hat f\in\mathcal{F}$ by specifying
its behaviour after period $T$ in any way. Then $\hat f$ and $\tilde f^{(T)}$ coincide on periods
$t\le T$, and similarly $f$ and $f^{(T)}$ coincide on $t\le T$. The weak distance between $\hat f$ and $f$
can be bounded as
\[
d(\mu^{\hat f},\mu^f)
\le
d(\mu^{\hat f},\mu^{\tilde f^{(T)}})
+
d(\mu^{\tilde f^{(T)}},\mu^{f^{(T)}})
+
d(\mu^{f^{(T)}},\mu^f).
\]
The second term is at most $\varepsilon/2$ by construction. For the first and third terms, any discrepancy
between $\hat f$ and $\tilde f^{(T)}$ (respectively, $f$ and $f^{(T)}$) occurs only at times $t>T$, so
each of these weak distances is bounded by the tail $\sum_{t>T}2^{-t} \le \varepsilon/4$ by
\eqref{eq:metric_tail2}. Hence
\[
d(\mu^{\hat f},\mu^f) \le \frac{\varepsilon}{4} + \frac{\varepsilon}{2} + \frac{\varepsilon}{4} = \varepsilon.
\]

It remains to show that $\hat f$ is a $(\xi+\varepsilon)$-Nash equilibrium of the
infinite-horizon game. Fix $i\in I$ and any deviation $g_i\in\mathcal{F}_i$.
Let $g_i^{(T)}$ denote the truncation of $g_i$ to a $T$-period strategy, i.e., its
prescriptions on histories of length at most $T$; clearly $U_i^T(g_i,\hat f_{-i})
= U_i^T(g_i^{(T)},\tilde f_{-i}^{(T)})$ since $\hat f$ and $\tilde f^{(T)}$ coincide
on the first $T$ periods.

Because $\tilde f^{(T)}$ is a $\psi_T$-Nash equilibrium of the $T$-period game,
\[
U_i^T(\tilde f_i^{(T)},\tilde f_{-i}^{(T)})
\;\ge\;
U_i^T(g_i^{(T)},\tilde f_{-i}^{(T)}) - \psi_T.
\]
Using the truncation bound \eqref{eq:trunc_choice2}, we obtain
\[
U_i(\hat f_i,\hat f_{-i})
\;\ge\;
U_i^T(\hat f_i,\hat f_{-i}) - \frac{\varepsilon}{8}
=
U_i^T(\tilde f_i^{(T)},\tilde f_{-i}^{(T)}) - \frac{\varepsilon}{8}
\]
and
\[
U_i(g_i,\hat f_{-i})
\;\le\;
U_i^T(g_i,\hat f_{-i}) + \frac{\varepsilon}{8}
=
U_i^T(g_i^{(T)},\tilde f_{-i}^{(T)}) + \frac{\varepsilon}{8}.
\]
Combining these inequalities yields
\begin{align*}
U_i(\hat f_i,\hat f_{-i})
&\ge
U_i^T(\tilde f_i^{(T)},\tilde f_{-i}^{(T)}) - \frac{\varepsilon}{8}\\
&\ge
U_i^T(g_i^{(T)},\tilde f_{-i}^{(T)}) - \psi_T - \frac{\varepsilon}{8}\\
&\ge
U_i(g_i,\hat f_{-i}) - \psi_T - \frac{\varepsilon}{4}.
\end{align*}
Recalling that $\psi_T = \xi + \varepsilon/4$, we have
\[
\psi_T + \frac{\varepsilon}{4} = \xi + \frac{\varepsilon}{2} \le \xi + \varepsilon,
\]
so for every deviation $g_i$,
\[
U_i(\hat f_i,\hat f_{-i}) \ge U_i(g_i,\hat f_{-i}) - (\xi+\varepsilon).
\]
Thus $\hat f$ is a $(\xi+\varepsilon)$-Nash equilibrium.
\end{proof}

\subsection{Auxiliary retained-menu identification assumption}
\label{app:retained_menu_identification}

The main text works with the concentration-up-to-equivalence condition in
Assumption~\ref{ass:equiv_class_conc}. For concrete verification, it is convenient to record
stronger sufficient routes. The logic is one-way: exact retained-menu identification is stronger
than what PS-BR needs, but it is easy to check in sparse deterministic menus.
We first state a deterministic hard-refutation condition, and then the more general
likelihood-ratio formulation that only requires on-path elimination of wrong retained hypotheses.

\begin{assumption}[Deterministic hard refutation on a strategy menu]
\label{ass:det_menu_sep_app}
Fix player $i$.
Assume the support of $\mu_i^0$ is finite; write
$\mathcal S_{-i}:=\mathrm{supp}(\mu_i^0)\subseteq \mathcal F_{-i}$.
Assume every element of $\mathcal S_{-i}$ is a deterministic public automaton. Moreover,
\begin{enumerate}[leftmargin=*]
\item \textit{(Menu grain of truth)} $f_{-i}\in\mathcal S_{-i}$ and $\mu_i^0(f_{-i})>0$.
\item \textit{(On-path finite separation)} For every
$g_{-i}\in\mathcal S_{-i}\setminus\{f_{-i}\}$,
\[
\mu^f\!\left(
\left\{
z:\exists\,t\ge 1\ \text{s.t.}\ 
g_{-i}(h^t(z))\neq f_{-i}(h^t(z))
\right\}
\right)=1.
\]
\end{enumerate}
\end{assumption}

\begin{assumption}[Retained-menu identification]\label{ass:finite_menu_kl}
Fix player $i$.
Assume the support of $\mu_i^0$ is finite; write $\mathcal S_{-i}:=\mathrm{supp}(\mu_i^0)\subseteq \mathcal F_{-i}$.
For any $g_{-i}\in\mathcal S_{-i}$ and realized path $z\in H^\infty$, define the finite-history likelihood
\[
L_{i,T}(g_{-i};z)
:=
\prod_{t=1}^{T-1} g_{-i}(h^t(z))(a_{-i}^t(z)),
\qquad T\ge 1,
\]
with the convention that once a factor is zero, all later likelihoods remain zero.
Then the following conditions hold:
\begin{enumerate}[leftmargin=*]
\item \textit{(Menu grain of truth)} $f_{-i}\in\mathcal S_{-i}$ and $\mu_i^0(f_{-i})>0$.
\item \textit{(On-path elimination of wrong retained hypotheses)} For every
$g_{-i}\in\mathcal S_{-i}\setminus\{f_{-i}\}$,
\[
\frac{L_{i,T}(g_{-i};z)}{L_{i,T}(f_{-i};z)} \longrightarrow 0
\qquad\text{$\mu^f$-a.s. in } z.
\]
\end{enumerate}
\end{assumption}

\begin{lemma}[Deterministic on-path separation implies retained-menu identification]
\label{lem:det_menu_sep_implies_finite_menu_kl}
Suppose Assumption~\ref{ass:det_menu_sep_app} holds. Then Assumption~\ref{ass:finite_menu_kl} holds.
\end{lemma}

\begin{proof}[Proof of Lemma \ref{lem:det_menu_sep_implies_finite_menu_kl}]
Item 1 of Assumption~\ref{ass:finite_menu_kl} is exactly
Assumption~\ref{ass:det_menu_sep_app}(1). It remains to verify Item 2.

Fix $g_{-i}\in\mathcal S_{-i}\setminus\{f_{-i}\}$ and define
\[
\tau_g(z):=
\inf\{t\ge 1: g_{-i}(h^t(z))\neq f_{-i}(h^t(z))\}.
\]
By Assumption~\ref{ass:det_menu_sep_app}(2), $\tau_g(z)<\infty$ for $\mu^f$-almost every $z$.
Fix such a path $z$ and let $t=\tau_g(z)$. Since $f_{-i}$ is deterministic and the
true opponents' strategy under $\mu^f$ is $f_{-i}$, the realized opponent action
$a_{-i}^t(z)$ is exactly the unique action assigned probability one by
$f_{-i}(h^t(z))$. Because $g_{-i}(h^t(z))\neq f_{-i}(h^t(z))$ and $g_{-i}$ is also
deterministic, it assigns probability zero to that realized action:
\[
g_{-i}(h^t(z))(a_{-i}^t(z))=0.
\]
Hence $g_{-i}$ is hard-refuted at time $t$, and by definition of the likelihood,
\[
L_{i,T}(g_{-i};z)=0
\qquad\text{for all }T>t.
\]
On the other hand, along the true path generated by the deterministic strategy
$f_{-i}$,
\[
f_{-i}(h^s(z))(a_{-i}^s(z))=1
\qquad\text{for every }s\ge 1,
\]
so
\[
L_{i,T}(f_{-i};z)=1
\qquad\text{for every }T\ge 1.
\]
Therefore
\[
\frac{L_{i,T}(g_{-i};z)}{L_{i,T}(f_{-i};z)}=0
\qquad\text{for all sufficiently large }T.
\]
This proves Assumption~\ref{ass:finite_menu_kl}(2).
\end{proof}

\begin{lemma}[Posterior concentration under retained-menu identification]
\label{lem:posterior_concentration}
Fix player $i$ and suppose Assumption~\ref{ass:finite_menu_kl} holds.
Then $\mu^f$-a.s.\ in $z$,
\[
\mu_i^t(f_{-i}\mid h^t(z))\ \longrightarrow\ 1,
\qquad\text{and hence}\qquad
\max_{g_{-i}\in\mathcal S_{-i}\setminus\{f_{-i}\}} \mu_i^t(g_{-i}\mid h^t(z))\ \longrightarrow\ 0.
\]
\end{lemma}

\begin{corollary}[Retained-menu identification implies the main-text concentration condition]
\label{cor:finite_menu_kl_implies_equiv_class_conc}
Suppose Assumption~\ref{ass:finite_menu_kl} holds. Then Assumption~\ref{ass:equiv_class_conc}
holds.
\end{corollary}

\begin{proof}
For every history $h^t$, the true label $f_{-i}$ belongs to its own continuation-payoff
equivalence class $\mathcal E_i(h^t)$. Therefore,
\[
\mu_i^t(\mathcal E_i(h^t(z))\mid h^t(z))
\ge
\mu_i^t(f_{-i}\mid h^t(z)).
\]
By Lemma~\ref{lem:posterior_concentration}, the right-hand side converges to $1$ on a
$\mu^f$-full-measure set of realized paths. Hence
\[
\mu_i^t(\mathcal E_i(h^t(z))\mid h^t(z))\longrightarrow 1
\qquad\text{$\mu^f$-a.s. in }z,
\]
which is exactly Assumption~\ref{ass:equiv_class_conc}(2).
\end{proof}

\begin{corollary}[Deterministic hard refutation implies the main-text concentration condition]
\label{cor:det_menu_sep_implies_equiv_class_conc}
Suppose Assumption~\ref{ass:det_menu_sep_app} holds. Then Assumption~\ref{ass:equiv_class_conc}
holds.
\end{corollary}

\begin{proof}
By Lemma~\ref{lem:det_menu_sep_implies_finite_menu_kl},
Assumption~\ref{ass:det_menu_sep_app} implies Assumption~\ref{ass:finite_menu_kl}. The claim
therefore follows immediately from
Corollary~\ref{cor:finite_menu_kl_implies_equiv_class_conc}.
\end{proof}

\begin{corollary}[Retained-menu identification implies the private-payoff opponents-side concentration condition]
\label{cor:finite_menu_kl_implies_equiv_class_conc_private}
Suppose player $i$'s finite retained opponents' menu and posterior satisfy
Assumption~\ref{ass:finite_menu_kl}. Then Assumption~\ref{ass:equiv_class_conc_private}
holds.
\end{corollary}

\begin{proof}
The exact same Bayes-ratio argument as in Lemma~\ref{lem:posterior_concentration} yields
\[
\mu_i^t(f_{-i}\mid h^t(\omega))\longrightarrow 1
\qquad\text{$P^{\beta,u}$-a.s. in }\omega.
\]
For every public history $h^t(\omega)$, the true opponents' model $f_{-i}$ belongs to its own
continuation-payoff equivalence class $\mathcal E_i^{\mathrm{priv}}(h^t(\omega))$. Therefore,
\[
\mu_i^t(\mathcal E_i^{\mathrm{priv}}(h^t(\omega))\mid h^t(\omega))
\ge
\mu_i^t(f_{-i}\mid h^t(\omega))
\longrightarrow 1,
\]
which is exactly Assumption~\ref{ass:equiv_class_conc_private}(2).
\end{proof}

\begin{corollary}[Deterministic hard refutation implies the private-payoff opponents-side concentration condition]
\label{cor:det_menu_sep_implies_equiv_class_conc_private}
Suppose player $i$'s finite retained opponents' menu satisfies
Assumption~\ref{ass:det_menu_sep_app}. Then
Assumption~\ref{ass:equiv_class_conc_private} holds.
\end{corollary}

\begin{proof}
By Lemma~\ref{lem:det_menu_sep_implies_finite_menu_kl},
Assumption~\ref{ass:det_menu_sep_app} implies Assumption~\ref{ass:finite_menu_kl}. The claim
therefore follows immediately from
Corollary~\ref{cor:finite_menu_kl_implies_equiv_class_conc_private}.
\end{proof}

The more general retained-menu identification condition (Assumption~\ref{ass:finite_menu_kl})
and Lemma~\ref{lem:posterior_concentration} remain useful as a stronger appendix-level route:
they show that exact posterior concentration can also be obtained when wrong retained labels are
eliminated asymptotically by likelihood-ratio decay rather than by finite-time hard refutation.
But for the sparse simulation menus used in our applications,
Corollary~\ref{cor:det_menu_sep_implies_equiv_class_conc} is the practically relevant route.
Those menus are built to exclude duplicate labels that would remain observationally or
strategically equivalent to the benchmark equilibrium automaton on the realized path. As a
result, a wrong retained deterministic label is eventually contradicted by some observed public
action, so deterministic hard refutation directly verifies the main-text concentration
condition.

The next appendix condition is the private-payoff analogue of the public-action verification
route above. In the main text, Assumption~\ref{ass:payoff_menu_sep_private} is stated as
menu-level concentration on player $i$'s true own-payoff continuation-decision class rather than
as literal global identification of the full retained mean matrix. No claim is made
about learning the opponents' payoff matrices. The intended economic logic is again on-path:
player $i$ need not learn its own payoffs everywhere, only at the reached action profiles that
matter for the continuation decision problem. For concrete verification, it is convenient to
record a stronger route based on likelihood-ratio elimination of wrong retained
mean matrices along the realized private payoff history.

\begin{assumption}[Identification of player $i$'s own payoff menu]
\label{ass:payoff_menu_sep_private_app}
Fix player $i$ and let $\mathcal M_i=\mathrm{supp}(\pi_i^0)$ be finite.
Assume:
\begin{enumerate}[leftmargin=*]
\item \textit{(Menu grain of truth)} The true mean matrix $u_i\in\mathcal M_i$ and
$\pi_i^0(u_i)>0$.
\item \textit{(Known common Gaussian noise family)} There exists $\sigma_i^2>0$ such that
\[
r_i^t\mid a^t \sim \mathcal N(u_i(a^t),\sigma_i^2),
\]
and each menu element $m_i\in\mathcal M_i$ induces the Gaussian payoff kernel
\[
q_i^{m_i}(\mathrm dr\mid a)=\phi(r;m_i(a),\sigma_i^2)\,\mathrm dr.
\]
Equivalently, for every $m_i\in\mathcal M_i$, the finite-history payoff likelihood is
\[
K_{i,T}(m_i;\omega)
:=
\prod_{t=1}^{T-1}
\phi\!\bigl(r_i^t(\omega);m_i(a^t(\omega)),\sigma_i^2\bigr),
\qquad T\ge 1,
\]
where $\phi(\cdot;\mu,\sigma_i^2)$ denotes the Gaussian density with mean $\mu$
and variance $\sigma_i^2$.
\item \textit{(On-path elimination of wrong retained payoff hypotheses)} For every
$m_i\in\mathcal M_i\setminus\{u_i\}$,
\[
\frac{K_{i,T}(m_i;\omega)}{K_{i,T}(u_i;\omega)}\longrightarrow 0
\qquad\text{$P^{\beta,u}$-a.s. in }\omega.
\]
\end{enumerate}
\end{assumption}

\begin{lemma}[Own-payoff-menu identification implies pointwise payoff posterior concentration]
\label{lem:payoff_posterior_concentration}
Fix player $i$ and suppose Assumption~\ref{ass:payoff_menu_sep_private_app} holds.
Then
\[
\pi_i^t(\mathcal U_i(h^t(\omega))\mid x_i^t(\omega))\longrightarrow 1
\qquad\text{for }P^{\beta,u}\text{-a.e. }\omega.
\]
\end{lemma}

The next assumption is a concrete sufficient route for the payoff-menu
identification condition under the same known common Gaussian setting used in
Experiment~3. Its role is to formalize the idea that wrong retained own-payoff hypotheses do not
have to be contradicted everywhere: it is enough that each such hypothesis disagree with the truth
at some joint action that the realized dynamics revisit often enough for private reward evidence to
accumulate.

\begin{assumption}[Gaussian recurrent mean-gap separation]
\label{ass:gaussian_payoff_menu_sep_private}
Fix player $i$ and let $\mathcal M_i=\mathrm{supp}(\pi_i^0)$ be finite. Assume:
\begin{enumerate}[leftmargin=*]
\item \textit{(Menu grain of truth)} The true mean matrix $u_i\in\mathcal M_i$ and
$\pi_i^0(u_i)>0$.
\item \textit{(Known common Gaussian noise family)} There exists $\sigma_i^2>0$ such that,
for every $m_i\in\mathcal M_i$ and every $a\in A$,
\[
q_i^{m_i}(\mathrm dr\mid a)
=
\frac{1}{\sqrt{2\pi}\sigma_i}
\exp\!\left(
-\frac{(r-m_i(a))^2}{2\sigma_i^2}
\right)\mathrm dr.
\]
\item \textit{(On-path recurrent mean separation)} For every
$m_i\in\mathcal M_i\setminus\{u_i\}$, there exist a joint action
$a^\star(m_i)\in A$ and a constant $\rho_i(m_i)>0$ such that
\[
m_i(a^\star(m_i))\neq u_i(a^\star(m_i))
\]
and, under the true interaction law $P^{\beta,u}$,
\[
\liminf_{T\to\infty}
\frac{1}{T}\sum_{t=1}^T \mathbf 1\{a^t=a^\star(m_i)\}
\ge
\rho_i(m_i)
\qquad\text{a.s.}
\]
\end{enumerate}
\end{assumption}

\begin{lemma}[Gaussian recurrent mean-gap separation implies payoff-menu identification]
\label{lem:gaussian_payoff_menu_sep_private_implies_payoff_menu_sep_private}
Suppose Assumption~\ref{ass:gaussian_payoff_menu_sep_private} holds. Then
Assumption~\ref{ass:payoff_menu_sep_private_app} holds.
\end{lemma}

\begin{proof}[Proof of Lemma \ref{lem:gaussian_payoff_menu_sep_private_implies_payoff_menu_sep_private}]
Item 1 of Assumption~\ref{ass:payoff_menu_sep_private_app} is exactly Item 1 of
Assumption~\ref{ass:gaussian_payoff_menu_sep_private}. Item 2 of
Assumption~\ref{ass:gaussian_payoff_menu_sep_private} supplies the known common
Gaussian noise family appearing in Item 2 of
Assumption~\ref{ass:payoff_menu_sep_private_app}. It remains to verify Item 3.

Fix $m_i\in\mathcal M_i\setminus\{u_i\}$ and define
\[
\Delta_t:=u_i(a^t)-m_i(a^t),
\qquad
\varepsilon_t:=r_i^t-u_i(a^t).
\]
Under the true interaction law $P^{\beta,u}$, $\varepsilon_t$ is conditionally mean-zero
Gaussian with variance $\sigma_i^2$ given $a^t$. For Gaussian likelihoods,
\[
\log\frac{\phi(r_i^t;u_i(a^t),\sigma_i^2)}{\phi(r_i^t;m_i(a^t),\sigma_i^2)}
=
\frac{\Delta_t^2}{2\sigma_i^2}
+
\frac{\Delta_t\varepsilon_t}{\sigma_i^2}.
\]
Therefore, pathwise,
\[
\log\frac{K_{i,T}(u_i;\omega)}{K_{i,T}(m_i;\omega)}
=
\frac{1}{2\sigma_i^2}\sum_{t=1}^{T-1}\Delta_t^2
+
\frac{1}{\sigma_i^2}\sum_{t=1}^{T-1}\Delta_t\varepsilon_t.
\]

Because $A$ and $\mathcal M_i$ are finite, there exists $C_i(m_i)<\infty$ such that
$|\Delta_t|\le C_i(m_i)$ for all $t$. Let
\[
\mathcal H_t:=\sigma(h^{t+1},r_i^{1:t-1}),
\qquad
Y_t:=\Delta_t\varepsilon_t.
\]
Then $a^t$, hence $\Delta_t$, is $\mathcal H_t$-measurable, and
\[
\mathbb E[Y_t\mid \mathcal H_t]
=
\Delta_t\,\mathbb E[\varepsilon_t\mid \mathcal H_t]
=
0.
\]
Also,
\[
\sup_t \mathbb E[Y_t^2]
\le
C_i(m_i)^2\sigma_i^2
<\infty.
\]
Hence
\[
\sum_{t=1}^{\infty}\frac{\mathbb E[Y_t^2]}{t^2}<\infty,
\]
so the martingale strong law implies
\[
\frac{1}{T}\sum_{t=1}^T Y_t\longrightarrow 0
\qquad\text{a.s.}
\]

By Assumption~\ref{ass:gaussian_payoff_menu_sep_private}(3), there exist
$a^\star(m_i)\in A$ and $\rho_i(m_i)>0$ such that
$m_i(a^\star(m_i))\neq u_i(a^\star(m_i))$ and
\[
\liminf_{T\to\infty}
\frac{1}{T}\sum_{t=1}^T \mathbf 1\{a^t=a^\star(m_i)\}
\ge
\rho_i(m_i)
\qquad\text{a.s.}
\]
Therefore,
\[
\liminf_{T\to\infty}
\frac{1}{T}\sum_{t=1}^T \Delta_t^2
\ge
\rho_i(m_i)\bigl(u_i(a^\star(m_i))-m_i(a^\star(m_i))\bigr)^2
>0
\qquad\text{a.s.}
\]
Combining the last three displays yields
\[
\liminf_{T\to\infty}
\frac{1}{T}\log\frac{K_{i,T}(u_i;\omega)}{K_{i,T}(m_i;\omega)}
>0
\qquad\text{$P^{\beta,u}$-a.s.}
\]
Hence
\[
\frac{K_{i,T}(m_i;\omega)}{K_{i,T}(u_i;\omega)}\longrightarrow 0
\qquad\text{$P^{\beta,u}$-a.s. in }\omega,
\]
which is exactly Item 3 of Assumption~\ref{ass:payoff_menu_sep_private_app}.
\end{proof}

\begin{proof}[Proof of Lemma \ref{lem:psbr_gap}]
For each $g_{-i}\in\mathcal S_{-i}$ define the continuation value envelope
\[
M(g_{-i})\ :=\ \sup_{\sigma_i} V_i(\sigma_i\mid h^t; g_{-i})\ \in\ [0,1].
\]
For each $g_{-i}$ pick a (measurable) best response $\sigma_i^{g_{-i}}\in\mathrm{BR}_i(g_{-i}\mid h^t)$, so that
$V_i(\sigma_i^{g_{-i}}\mid h^t; g_{-i})=M(g_{-i})$.

By definition, PS-BR first samples $\tilde g_{-i}\sim p_t(\cdot)$ and then plays $\sigma_i^{\tilde g_{-i}}$.
Evaluating against the posterior predictive belief and using linearity in the mixing over opponent hypotheses,
\begin{align*}
V_i(\sigma^{\mathrm{PS}}_{i,t}\mid h^t)
&= \sum_{\tilde g_{-i}\in\mathcal S_{-i}} p_t(\tilde g_{-i})\ 
\sum_{g_{-i}\in\mathcal S_{-i}} p_t(g_{-i})\ V_i(\sigma_i^{\tilde g_{-i}}\mid h^t; g_{-i}) \\
&\ge \sum_{g_{-i}\in\mathcal S_{-i}} p_t(g_{-i})^2\ V_i(\sigma_i^{g_{-i}}\mid h^t; g_{-i}) \\
&= \sum_{g_{-i}\in\mathcal S_{-i}} p_t(g_{-i})^2\, M(g_{-i}).
\end{align*}
On the other hand,
\[
\sup_{\sigma_i} V_i(\sigma_i\mid h^t)
= \sup_{\sigma_i}\sum_{g_{-i}\in\mathcal S_{-i}} p_t(g_{-i})\, V_i(\sigma_i\mid h^t; g_{-i})
\le \sum_{g_{-i}\in\mathcal S_{-i}} p_t(g_{-i})\, M(g_{-i}).
\]
Subtracting and using $M(g_{-i})\le 1$,
\begin{align*}
\sup_{\sigma_i} V_i(\sigma_i\mid h^t) - V_i(\sigma^{\mathrm{PS}}_{i,t}\mid h^t)
&\le \sum_{g_{-i}\in\mathcal S_{-i}} \Big(p_t(g_{-i})-p_t(g_{-i})^2\Big) M(g_{-i}) \\
&\le \sum_{g_{-i}\in\mathcal S_{-i}} \Big(p_t(g_{-i})-p_t(g_{-i})^2\Big) \\
&= 1-\sum_{g_{-i}\in\mathcal S_{-i}} p_t(g_{-i})^2
= D_i^t(h^t).
\end{align*}
This proves the claim.
\end{proof}

\begin{proof}[Proof of Lemma \ref{lem:posterior_concentration}]
Fix any $g_{-i}\in\mathcal S_{-i}\setminus\{f_{-i}\}$ and a realized path $z\in H^\infty$
in the full-measure event from Assumption~\ref{ass:finite_menu_kl}(2). By Bayes' rule, for every $T\ge 1$,
\[
\frac{\mu_i^T(g_{-i}\mid h^T(z))}{\mu_i^T(f_{-i}\mid h^T(z))}
=
\frac{\mu_i^0(g_{-i})}{\mu_i^0(f_{-i})}\,
\frac{L_{i,T}(g_{-i};z)}{L_{i,T}(f_{-i};z)}.
\]
Assumption~\ref{ass:finite_menu_kl}(2) implies that the likelihood ratio on the right converges to $0$, hence
\[
\frac{\mu_i^T(g_{-i}\mid h^T(z))}{\mu_i^T(f_{-i}\mid h^T(z))}\longrightarrow 0.
\]
Since this holds for every $g_{-i}\neq f_{-i}$ and the menu is finite,
\[
R_T(z):=\sum_{g_{-i}\in\mathcal S_{-i}\setminus\{f_{-i}\}}
\frac{\mu_i^T(g_{-i}\mid h^T(z))}{\mu_i^T(f_{-i}\mid h^T(z))}
\longrightarrow 0
\qquad\text{$\mu^f$-a.s.}
\]
But
\[
1
=
\mu_i^T(f_{-i}\mid h^T(z))
+
\sum_{g_{-i}\neq f_{-i}}\mu_i^T(g_{-i}\mid h^T(z))
=
\mu_i^T(f_{-i}\mid h^T(z))\bigl(1+R_T(z)\bigr),
\]
so
\[
\mu_i^T(f_{-i}\mid h^T(z))
=
\frac{1}{1+R_T(z)}
\longrightarrow 1.
\]
Consequently,
\[
\max_{g_{-i}\in\mathcal S_{-i}\setminus\{f_{-i}\}} \mu_i^T(g_{-i}\mid h^T(z))\longrightarrow 0.
\]
\end{proof}

\begin{proof}[Proof of Proposition \ref{prop:ps_implies_asymptotic_consistency}]
Fix $\varepsilon>0$ and work on a realized path $z$ in the full-measure event from
Assumption~\ref{ass:equiv_class_conc}(2). For each $t$, write
\[
\mathcal E:=\mathcal E_i(h^t(z))
\qquad\text{and}\qquad
\beta_i^t(h^t(z)):=1-\mu_i^t(\mathcal E\mid h^t(z)).
\]
By Assumption~\ref{ass:equiv_class_conc}(2),
\[
\beta_i^t(h^t(z))\longrightarrow 0
\qquad\text{$\mu^f$-a.s. in }z.
\]

Fix such a history $h^t(z)$. For each sampled label $\tilde g_{-i}\in\mathcal S_{-i}$, let
$\sigma_i^{\tilde g_{-i}}$ denote the continuation strategy chosen by PS-BR when
$\tilde g_{-i}$ is sampled. If $\tilde g_{-i}\in\mathcal E$, then by
Definition~\ref{def:cont_equiv}, every label in $\mathcal E$ induces the same continuation-value
functional after $h^t(z)$. Denote that common functional by
\[
W_i(\sigma_i\mid h^t(z)).
\]
Let
\[
M_i(h^t(z)):=\sup_{\sigma_i} W_i(\sigma_i\mid h^t(z)).
\]
Because PS-BR best responds to the sampled label, for every $\tilde g_{-i}\in\mathcal E$ we have
\[
W_i(\sigma_i^{\tilde g_{-i}}\mid h^t(z))=M_i(h^t(z)).
\]

Now evaluate PS-BR under the full posterior predictive continuation value. Writing
$p_t(g_{-i})=\mu_i^t(g_{-i}\mid h^t(z))$, we have
\begin{align*}
V_i(\sigma_{i,t}^{\mathrm{PS}}\mid h^t(z))
&=
\sum_{\tilde g_{-i}\in\mathcal S_{-i}} p_t(\tilde g_{-i})
\sum_{g_{-i}\in\mathcal S_{-i}} p_t(g_{-i})
V_i(\sigma_i^{\tilde g_{-i}}\mid h^t(z);g_{-i})\\
&\ge
\sum_{\tilde g_{-i}\in\mathcal E} p_t(\tilde g_{-i})
\sum_{g_{-i}\in\mathcal E} p_t(g_{-i})
W_i(\sigma_i^{\tilde g_{-i}}\mid h^t(z))\\
&=
\bigl(1-\beta_i^t(h^t(z))\bigr)^2 M_i(h^t(z)).
\end{align*}
On the other hand, for any continuation strategy $\sigma_i$,
\begin{align*}
V_i(\sigma_i\mid h^t(z))
&=
\sum_{g_{-i}\in\mathcal E} p_t(g_{-i}) W_i(\sigma_i\mid h^t(z))
+
\sum_{g_{-i}\notin\mathcal E} p_t(g_{-i})
V_i(\sigma_i\mid h^t(z);g_{-i})\\
&\le
\bigl(1-\beta_i^t(h^t(z))\bigr) M_i(h^t(z)) + \beta_i^t(h^t(z)),
\end{align*}
where the last inequality uses that continuation values lie in $[0,1]$. Therefore
\begin{align*}
\sup_{\sigma_i}V_i(\sigma_i\mid h^t(z))
-
V_i(\sigma_{i,t}^{\mathrm{PS}}\mid h^t(z))
&\le
\bigl(1-\beta_i^t(h^t(z))\bigr) M_i(h^t(z))
+\beta_i^t(h^t(z))
-\bigl(1-\beta_i^t(h^t(z))\bigr)^2 M_i(h^t(z))\\
&=
\beta_i^t(h^t(z))
+\beta_i^t(h^t(z))\bigl(1-\beta_i^t(h^t(z))\bigr)M_i(h^t(z))\\
&\le 2\beta_i^t(h^t(z)).
\end{align*}
Hence
\[
\sigma_{i,t}^{\mathrm{PS}}(\cdot\mid h^t(z))
\in
\mathrm{BR}_i^{2\beta_i^t(h^t(z))}\!\bigl(f_{-i}^{i,t}\big|_{h^t(z)}\mid h^t(z)\bigr).
\]
Because $\beta_i^t(h^t(z))\to 0$, there exists a finite time $T_i(z,\varepsilon)$ such that
$2\beta_i^t(h^t(z))\le \varepsilon$ for all $t\ge T_i(z,\varepsilon)$. Therefore, for all
$t\ge T_i(z,\varepsilon)$,
\[
\sigma_{i,t}^{\mathrm{PS}}(\cdot\mid h^t(z))
\in
\mathrm{BR}_i^{\varepsilon}\!\bigl(f_{-i}^{i,t}\big|_{h^t(z)}\mid h^t(z)\bigr),
\]
which is exactly the claimed planner-side asymptotic $\varepsilon$-consistency statement.
\end{proof}

\begin{proof}[Proof of Lemma \ref{lem:absolute_cont_predict}]
Let $\mu^{f^i} \equiv P_i^{0, f_i}$ be the distribution induced by the predictive reference profile $(f_i, f_{-i}^i)$ representing the prior predictive. By Assumption~\ref{ass:grain_of_truth}, $\mu^f \ll \mu^{f^i}$. 

By the merging of opinions theorem \citep{kalai1993rational, blackwell1962merging}, absolute continuity guarantees that the conditional predictive distributions over future play paths merge almost surely in total variation. Specifically, for $\mu^f$-almost every path $z \in H^\infty$:
\[
\lim_{t \to \infty} \sup_{E \in \mathcal{B}} \big| \mu^f(E \mid C(h^t(z))) - \mu^{f^i}(E \mid C(h^t(z))) \big| = 0,
\]
where $\mathcal{B}$ is the product $\sigma$-algebra on $H^\infty$.

Recall from Definition~\ref{def:weak_distance} that the continuation weak distance is bounded by the total variation distance. For any finite length $k$, the $\sigma$-algebra $\mathcal{B}^k$ generated by cylinder events of length $k$ is a sub-$\sigma$-algebra of $\mathcal{B}$. Therefore:
\begin{align*}
\sup_{E \in \mathcal{B}^k} \big| \mu^f(E \mid C(h^t(z))) - \mu^{f^i}(E \mid C(h^t(z))) \big| \le \sup_{E \in \mathcal{B}} \big| \mu^f(E \mid C(h^t(z))) - \mu^{f^i}(E \mid C(h^t(z))) \big|.
\end{align*}
Using this bound, the continuation weak distance $d_{h^t(z)}(\mu^f, \mu^{f^i})$ satisfies:
\begin{align*}
d_{h^t(z)}(\mu^f, \mu^{f^i}) &= \sum_{k=1}^\infty 2^{-k} \sup_{E \in \mathcal{B}^k} \big| \mu^f(E \mid C(h^t(z))) - \mu^{f^i}(E \mid C(h^t(z))) \big| \\
&\le \sum_{k=1}^\infty 2^{-k} \sup_{E \in \mathcal{B}} \big| \mu^f(E \mid C(h^t(z))) - \mu^{f^i}(E \mid C(h^t(z))) \big| \\
&= \sup_{E \in \mathcal{B}} \big| \mu^f(E \mid C(h^t(z))) - \mu^{f^i}(E \mid C(h^t(z))) \big|.
\end{align*}
Since the total variation distance on the right-hand side converges to zero as $t \to \infty$ for $\mu^f$-almost every $z$, we have:
\[
\lim_{t \to \infty} d_{h^t(z)}(\mu^f, \mu^{f^i}) = 0 \quad \text{$\mu^f$-a.s.}
\]
By the definition of the limit, for any $\eta > 0$, there $\mu^f$-a.s.\ exists a finite time $T_i(z, \eta)$ such that for all $t \ge T_i(z, \eta)$, $d_{h^t(z)}(\mu^f, \mu^{f^i}) \le \eta$. This precisely satisfies the strong path prediction requirement in Definition~\ref{def:learn_predict_path}.
\end{proof}

\begin{proof}[Proof of Proposition \ref{prop:weak_subjective_from_learning}]
Fix $\xi,\eta>0$.
For each player $i$, the asymptotic $\varepsilon$-consistency on-path condition implies that $\mu^f$-a.s.\ in $z$ there exists $T_i^\mathrm{br}(z)$ such that for all $t\ge T_i^\mathrm{br}(z)$,
\[
f_i\big|_{h^t(z)}\in \mathrm{BR}_i^\xi\!\big(f_{-i}^{i,t}\big|_{h^t(z)}\mid h^t(z)\big).
\]
By the representative choice \eqref{eq:app_rep_choice}, we may equivalently write $f_{-i}^{i,t}\big|_{h^t(z)}\equiv f_{-i}^i\big|_{h^t(z)}$, so for all $t\ge T_i^\mathrm{br}(z)$,
\[
f_i\big|_{h^t(z)}\in \mathrm{BR}_i^\xi\!\big(f_{-i}^i\big|_{h^t(z)}\mid h^t(z)\big),
\]
which is exactly the subjective best-response condition in Definition~\ref{def:weak_subjective_eq}.

Similarly, strong prediction implies that $\mu^f$-a.s.\ in $z$ there exists $T_i^\mathrm{pred}(z)$ such that for all $t\ge T_i^\mathrm{pred}(z)$,
\[
d_{h^t(z)}(\mu^f,\mu^{f^i})\le \eta,
\]
which is the weak predictive accuracy condition in Definition~\ref{def:weak_subjective_eq}.

Let $T(z):=\max_i\{T_i^\mathrm{br}(z),\,T_i^\mathrm{pred}(z)\}$, which is finite $\mu^f$-a.s.\ since $I$ is finite.
Then for all $t\ge T(z)$ and every player $i$, both conditions in Definition~\ref{def:weak_subjective_eq} hold with supporting profile $f^i$, so $f\big|_{h^t(z)}$ is a weak $\xi$-subjective $\eta$-equilibrium after $h^t(z)$.
\end{proof}

\begin{proof}[Proof of Theorem \ref{thm:zero_shot_nash_final}]
Fix $\varepsilon>0$ and set $\xi:=\varepsilon/2$.
Let $\hat\eta(\cdot,\cdot)$ be the function from the infinite patching lemma (Lemma~\ref{lem:infinite_patching} in Appendix~\ref{app:patching}), and set $\eta:=\hat\eta(\xi,\varepsilon/2)$.

By Proposition~\ref{prop:weak_subjective_from_learning}, $\mu^f$-a.s.\ in $z$ there exists $T(z)$ such that for all $t\ge T(z)$, the continuation profile $f\big|_{h^t(z)}$ is a weak $\xi$-subjective $\eta$-equilibrium after $h^t(z)$.
Applying Lemma~\ref{lem:infinite_patching} at each such $t$ yields an $\varepsilon$-Nash equilibrium $\hat f^{\varepsilon,t,z}$ of the continuation game after $h^t(z)$ satisfying
$d_{h^t(z)}(\mu^f,\mu^{\hat f^{\varepsilon,t,z}})\le \varepsilon$.
\end{proof}

\begin{proof}[Proof of Corollary \ref{cor:psbr_zero_shot_nash}]
By Proposition~\ref{prop:ps_implies_asymptotic_consistency}, under
Assumption~\ref{ass:equiv_class_conc}, each player's selected PS-BR continuation plan
eventually satisfies the asymptotic best-response condition on the realized path.
Under the continuation-plan implementation condition in the corollary, this is exactly the
asymptotic $\varepsilon$-consistency premise required in Theorem~\ref{thm:zero_shot_nash_final}.
Because Assumption~\ref{ass:equiv_class_conc}(1) places positive prior mass on the true opponent
strategy, Assumption~\ref{ass:grain_of_truth} follows. Lemma~\ref{lem:absolute_cont_predict}
therefore guarantees each player learns to predict the path of play under $f$.
Theorem~\ref{thm:zero_shot_nash_final} then applies.
\end{proof}

\begin{proof}[Proof of Lemma~\ref{lem:payoff_posterior_concentration}]
Fix any $m_i\in\mathcal M_i\setminus\{u_i\}$ and a realized path $\omega\in\Omega$.
By Bayes' rule \eqref{eq:payoff_posterior} and the definition of $K_{i,T}$ in
Assumption~\ref{ass:payoff_menu_sep_private_app},
\[
\frac{\pi_i^T(m_i\mid x_i^T(\omega))}{\pi_i^T(u_i\mid x_i^T(\omega))}
=
\frac{\pi_i^0(m_i)}{\pi_i^0(u_i)}
\frac{K_{i,T}(m_i;\omega)}{K_{i,T}(u_i;\omega)}.
\]
By Assumption~\ref{ass:payoff_menu_sep_private_app}(3), the likelihood ratio on the
right converges to $0$ for every $m_i\neq u_i$, hence
\[
\frac{\pi_i^T(m_i\mid x_i^T(\omega))}{\pi_i^T(u_i\mid x_i^T(\omega))}
\longrightarrow 0
\qquad\text{$P^{\beta,u}$-a.s. in }\omega.
\]
Since $\mathcal M_i$ is finite,
\[
\sum_{m_i\in\mathcal M_i\setminus\{u_i\}}
\frac{\pi_i^T(m_i\mid x_i^T(\omega))}{\pi_i^T(u_i\mid x_i^T(\omega))}
\longrightarrow 0
\qquad\text{$P^{\beta,u}$-a.s.}
\]
Using
\[
1
=
\pi_i^T(u_i\mid x_i^T(\omega))
\left(
1+
\sum_{m_i\in\mathcal M_i\setminus\{u_i\}}
\frac{\pi_i^T(m_i\mid x_i^T(\omega))}{\pi_i^T(u_i\mid x_i^T(\omega))}
\right),
\]
we obtain
\[
\pi_i^T(u_i\mid x_i^T(\omega))\longrightarrow 1
\qquad\text{for }P^{\beta,u}\text{-almost every }\omega.
\]
By Definition~\ref{def:payoff_equiv_private}, the true mean matrix $u_i$ belongs to its own
continuation-decision class $\mathcal U_i(h^T(\omega))$ at every public history. Hence
\[
\pi_i^T(\mathcal U_i(h^T(\omega))\mid x_i^T(\omega))
\ge
\pi_i^T(u_i\mid x_i^T(\omega))
\longrightarrow 1
\qquad\text{$P^{\beta,u}$-a.s.}
\]
which is exactly the stated pointwise payoff posterior concentration.
\end{proof}

\begin{proof}[Proof of Lemma~\ref{lem:psarbr_gap}]
Fix player $i$ and an observable history $x_i^t=(h^t,r_i^{1:t-1})$.
Let $\mathcal M:=\mathcal S_{-i}\times\mathcal M_i$, and for each
$m=(g_{-i},m_i)\in\mathcal M$ define the continuation value functional
\[
V_i^{m}(\tau_i\mid x_i^t)
\;:=\;
V_i^{m_i}(\tau_i\mid h^t; g_{-i})
\;\in\;[0,1],
\]
and the value envelope
\[
M(m)\;:=\;\sup_{\tau_i\in\mathcal F_i(h^t)} V_i^{m}(\tau_i\mid x_i^t)\;\in\;[0,1].
\]
For each $m\in\mathcal M$ fix a (measurable) best response $\tau_i^{m}$ attaining $M(m)$, i.e.,
$V_i^{m}(\tau_i^{m}\mid x_i^t)=M(m)$.

By Definition~\ref{def:psarbr}, PS-BR samples $(\tilde g_{-i},\tilde m_i)\sim p_t(\cdot)$ and then plays
$\tau_i^{(\tilde g_{-i},\tilde m_i)}$. Let $\sigma^{\mathrm{PS}}_{i,t}$ denote this randomized continuation
strategy at $x_i^t$.

Because $V_i^{\mathrm{mix},t}$ is linear in both the opponents-mixture and the payoff-matrix mixture, we can write
\[
V_i^{\mathrm{mix},t}(\tau_i\mid x_i^t)
=
\sum_{(g_{-i},m_i)\in\mathcal M} p_t(g_{-i},m_i)\,V_i^{(g_{-i},m_i)}(\tau_i\mid x_i^t)
=
\sum_{m\in\mathcal M} p_t(m)\,V_i^{m}(\tau_i\mid x_i^t).
\]
Therefore, evaluating PS-BR under the mixed subjective objective gives
\begin{align*}
V_i^{\mathrm{mix},t}(\sigma^{\mathrm{PS}}_{i,t}\mid x_i^t)
&=
\sum_{\tilde m\in\mathcal M} p_t(\tilde m)\;
V_i^{\mathrm{mix},t}(\tau_i^{\tilde m}\mid x_i^t)\\
&=
\sum_{\tilde m\in\mathcal M} p_t(\tilde m)\;
\sum_{m\in\mathcal M} p_t(m)\,V_i^{m}(\tau_i^{\tilde m}\mid x_i^t)\\
&\ge
\sum_{m\in\mathcal M} p_t(m)^2\,V_i^{m}(\tau_i^{m}\mid x_i^t)
=
\sum_{m\in\mathcal M} p_t(m)^2\,M(m).
\end{align*}
On the other hand,
\[
\sup_{\tau_i\in\mathcal F_i(h^t)} V_i^{\mathrm{mix},t}(\tau_i\mid x_i^t)
=
\sup_{\tau_i\in\mathcal F_i(h^t)}\sum_{m\in\mathcal M} p_t(m)\,V_i^{m}(\tau_i\mid x_i^t)
\le
\sum_{m\in\mathcal M} p_t(m)\,\sup_{\tau_i\in\mathcal F_i(h^t)}V_i^{m}(\tau_i\mid x_i^t)
=
\sum_{m\in\mathcal M} p_t(m)\,M(m).
\]
Subtracting and using $M(m)\le 1$ for all $m$,
\begin{align*}
\sup_{\tau_i\in\mathcal F_i(h^t)}V_i^{\mathrm{mix},t}(\tau_i\mid x_i^t)-V_i^{\mathrm{mix},t}(\sigma^{\mathrm{PS}}_{i,t}\mid x_i^t)
&\le
\sum_{m\in\mathcal M}\big(p_t(m)-p_t(m)^2\big)M(m)\\
&\le
\sum_{m\in\mathcal M}\big(p_t(m)-p_t(m)^2\big)
=
1-\sum_{m\in\mathcal M}p_t(m)^2
=
D_i^{t,\mathrm{joint}}(x_i^t).
\end{align*}
This proves the claim.
\end{proof}

\begin{proof}[Proof of Lemma \ref{lem:psarbr_equiv_gap}]
Fix an observable history $x_i^t=(h^t,r_i^{1:t-1})$. Write
\[
p_t(g_{-i},m_i):=\mu_i^t(g_{-i}\mid h^t)\,\pi_i^t(m_i\mid x_i^t),
\]
let
\[
\mathcal F_t:=\mathcal E_i^{\mathrm{priv}}(h^t)\times\mathcal U_i(h^t),
\qquad
r_t:=\sum_{(g_{-i},m_i)\in\mathcal F_t} p_t(g_{-i},m_i)
=(1-\alpha_i^t(h^t))(1-\delta_i^t(x_i^t)).
\]
For each sampled pair $(\tilde g_{-i},\tilde m_i)$, let
$\tau_i^{\tilde g_{-i},\tilde m_i}$ denote the continuation strategy selected by PS-BR.
If $(\tilde g_{-i},\tilde m_i)\in\mathcal F_t$, then
$\tilde g_{-i}\in\mathcal E_i^{\mathrm{priv}}(h^t)$ and
$\tilde m_i\in\mathcal U_i(h^t)$. By
Definition~\ref{def:payoff_equiv_private}, $\tilde m_i$ induces the same continuation-value
functional as the true mean matrix $u_i$ against the sampled opponents' model
$\tilde g_{-i}$. By the definition of $\mathcal E_i^{\mathrm{priv}}(h^t)$,
every $\tilde g_{-i}\in\mathcal E_i^{\mathrm{priv}}(h^t)$ is continuation-payoff equivalent to
$f_{-i}$ under $u_i$. Therefore, for every $(\tilde g_{-i},\tilde m_i)\in\mathcal F_t$ and every $\tau_i\in\mathcal F_i(h^t)$,
\[
V_i^{\tilde m_i}(\tau_i\mid h^t;\tilde g_{-i})
=
V_i^{u_i}(\tau_i\mid h^t;f_{-i}).
\]
Let
\[
M_i(h^t):=\sup_{\tau_i\in\mathcal F_i(h^t)} V_i^{u_i}(\tau_i\mid h^t;f_{-i}).
\]
For every $(\tilde g_{-i},\tilde m_i)\in\mathcal F_t$, PS-BR chooses a best response to that
pair, hence
\[
V_i^{u_i}(\tau_i^{\tilde g_{-i},\tilde m_i}\mid h^t;f_{-i})=M_i(h^t).
\]
Evaluating against the mixed subjective objective gives
\begin{align*}
V_i^{\mathrm{mix},t}(\sigma_{i,t}^{\mathrm{PS}}\mid x_i^t)
&=
\sum_{(\tilde g_{-i},\tilde m_i)} p_t(\tilde g_{-i},\tilde m_i)
\sum_{(g_{-i},m_i)} p_t(g_{-i},m_i)
V_i^{m_i}(\tau_i^{\tilde g_{-i},\tilde m_i}\mid h^t;g_{-i}) \\
&\ge
\sum_{(\tilde g_{-i},\tilde m_i)\in\mathcal F_t} p_t(\tilde g_{-i},\tilde m_i)
\sum_{(g_{-i},m_i)\in\mathcal F_t} p_t(g_{-i},m_i)
V_i^{u_i}(\tau_i^{\tilde g_{-i},\tilde m_i}\mid h^t;f_{-i}) \\
&=
r_t^2 M_i(h^t).
\end{align*}
For any continuation strategy $\tau_i\in\mathcal F_i(h^t)$,
\[
V_i^{\mathrm{mix},t}(\tau_i\mid x_i^t)
\le
r_t M_i(h^t) + (1-r_t),
\]
since all continuation values lie in $[0,1]$. Therefore
\begin{align*}
\sup_{\tau_i\in\mathcal F_i(h^t)}V_i^{\mathrm{mix},t}(\tau_i\mid x_i^t)
-
V_i^{\mathrm{mix},t}(\sigma_{i,t}^{\mathrm{PS}}\mid x_i^t)
&\le
(1-r_t) + r_t M_i(h^t) - r_t^2 M_i(h^t) \\
&\le 2(1-r_t) \\
&=2\bigl(\alpha_i^t(h^t)+\delta_i^t(x_i^t)-\alpha_i^t(h^t)\delta_i^t(x_i^t)\bigr) \\
&\le 2\alpha_i^t(h^t)+2\delta_i^t(x_i^t).
\end{align*}
Using \eqref{eq:mix_to_true_value_bound},
\begin{align*}
V_i^{u_i,t}(\sigma_{i,t}^{\mathrm{PS}}\mid x_i^t)
&\ge
V_i^{\mathrm{mix},t}(\sigma_{i,t}^{\mathrm{PS}}\mid x_i^t)-\delta_i^t(x_i^t) \\
&\ge
\sup_{\tau_i\in\mathcal F_i(h^t)}V_i^{\mathrm{mix},t}(\tau_i\mid x_i^t)-2\alpha_i^t(h^t)-3\delta_i^t(x_i^t) \\
&\ge
\sup_{\tau_i\in\mathcal F_i(h^t)}V_i^{u_i,t}(\tau_i\mid x_i^t)-2\alpha_i^t(h^t)-4\delta_i^t(x_i^t).
\end{align*}
This proves the claim.
\end{proof}

\begin{proof}[Proof of Proposition~\ref{prop:psarbr_implies_asymptotic_consistency}]
Work on the full-measure event on which both convergence statements hold:
\[
\alpha_i^t(h^t(\omega))\longrightarrow 0
\qquad\text{and}\qquad
\delta_i^t(x_i^t(\omega))\longrightarrow 0.
\]
The first follows from Assumption~\ref{ass:equiv_class_conc_private}; the second follows
from Assumption~\ref{ass:payoff_menu_sep_private}. By Lemma~\ref{lem:psarbr_equiv_gap},
for every observable history $x_i^t=(h^t,r_i^{1:t-1})$,
\[
\sup_{\tau_i\in\mathcal F_i(h^t)}V_i^{u_i,t}(\tau_i\mid x_i^t)
-
V_i^{u_i,t}(\sigma_{i,t}^{\mathrm{PS}}\mid x_i^t)
\le
2\alpha_i^t(h^t)+4\delta_i^t(x_i^t).
\]
Using the representative identity \eqref{eq:private_repr_value}, this is exactly
\[
\sup_{\tau_i\in\mathcal F_i(h^t)}V_i^{u_i}(\tau_i\mid h^t;g_{-i}^{i,t})
-
V_i^{u_i}(\sigma_{i,t}^{\mathrm{PS}}\mid h^t;g_{-i}^{i,t})
\le
2\alpha_i^t(h^t)+4\delta_i^t(x_i^t).
\]
Fix $\varepsilon>0$. On the full-measure event above, choose $T_i(\omega,\varepsilon)$ such that
for all $t\ge T_i(\omega,\varepsilon)$,
\[
2\alpha_i^t(h^t(\omega))+4\delta_i^t(x_i^t(\omega))\le \varepsilon.
\]
Then for all such $t$,
\[
\sigma_{i,t}^{\mathrm{PS}}(\cdot\mid x_i^t(\omega))
\in
\mathrm{BR}_{i,u_i}^{\varepsilon}\!\bigl(g_{-i}^{i,t}\mid h^t(\omega)\bigr),
\]
which is exactly the claimed planner-side asymptotic $\varepsilon$-consistency statement.
\end{proof}

\begin{proof}[Proof of Lemma \ref{lem:absolute_cont_predict_private}]
Fix player $i$. By the menu grain-of-truth part of
Assumption~\ref{ass:equiv_class_conc_private}, the true opponents' public-history strategy
$f_{-i}$ belongs to player $i$'s retained menu with positive prior mass. Therefore, for every
measurable event $E\subseteq H^\infty$,
\[
P_i^{0,f_i}(E)
=
\int_{\mathcal F_{-i}} \mu^{(f_i,g_{-i})}(E)\,d\mu_i^0(g_{-i})
\ge
\mu_i^0(f_{-i})\,\mu^{(f_i,f_{-i})}(E)
=
\mu_i^0(f_{-i})\,\mu^f(E).
\]
Hence $\mu^f\ll P_i^{0,f_i}$, i.e., the ordinary public grain-of-truth condition from
Assumption~\ref{ass:grain_of_truth} holds automatically.
Applying Lemma~\ref{lem:absolute_cont_predict} therefore yields
\[
d_{h^t(\omega)}\!\left(
\mu^f,
\mu^{f^i}
\right)\longrightarrow 0
\qquad\text{for }P^{\beta,u}\text{-a.e. }\omega,
\]
where $f^i=(f_i,f_{-i}^i)$ is the representative predictive profile for player $i$.
By the continuation-consistent representative choice from
Appendix~\ref{app:belief_repr}, the conditional law $\mu_{h^t(\omega)}^{f^i}$ is exactly the
posterior predictive public-action law $\Pi_i^t(\cdot\mid x_i^t(\omega))$. Hence
\[
d\!\left(
\Pi_i^t(\cdot\mid x_i^t(\omega)),
\mu_{h^t(\omega)}^f
\right)\longrightarrow 0
\qquad\text{for }P^{\beta,u}\text{-a.e. }\omega.
\]
\end{proof}

\begin{proof}[Proof of Proposition \ref{prop:weak_subjective_from_learning_private}]
Fix $\xi,\eta>0$.
For each player $i$, the assumed asymptotic $\xi$-consistency condition implies that
$P^{\beta,u}$-a.s.\ there exists $T_i^{\mathrm{br}}(\omega)$ such that for all
$t\ge T_i^{\mathrm{br}}(\omega)$,
\[
f_i\big|_{h^t(\omega)}
\in
\mathrm{BR}_{i,u_i}^{\xi}\!\bigl(g_{-i}^{i,t}\mid h^t(\omega)\bigr).
\]
Also, Lemma~\ref{lem:absolute_cont_predict_private} implies that $P^{\beta,u}$-a.s.\ there
exists $T_i^{\mathrm{pred}}(\omega)$ such that for all
$t\ge T_i^{\mathrm{pred}}(\omega)$,
\[
d_{h^t(\omega)}\!\left(
\mu^f,
\mu^{(f_i,g_{-i}^{i,t})}
\right)\le \eta.
\]
Let
\[
T(\omega):=\max_{i\in I}\{T_i^{\mathrm{br}}(\omega),T_i^{\mathrm{pred}}(\omega)\}.
\]
Then for all $t\ge T(\omega)$ and every player $i$, both conditions in
Definition~\ref{def:weak_subjective_eq_private} hold with supporting profile
$(f_i,g_{-i}^{i,t})$.
\end{proof}

\begin{proof}[Proof of Theorem \ref{cor:psarbr_zero_shot_nash}]
Fix $\varepsilon>0$ and set $\xi:=\varepsilon/2$.
Let $\hat\eta(\cdot,\cdot)$ be the function from Lemma~\ref{lem:infinite_patching}, and set
\[
\eta:=\hat\eta(\xi,\varepsilon/2).
\]
By the assumed asymptotic on-path $\xi$-consistency condition together with
Lemma~\ref{lem:absolute_cont_predict_private}, the hypotheses of
Proposition~\ref{prop:weak_subjective_from_learning_private} hold. Therefore
$P^{\beta,u}$-a.s.\ there exists $T(\omega)$ such that for all $t\ge T(\omega)$, the
continuation profile $f\big|_{h^t(\omega)}$ is a weak $\xi$-subjective $\eta$-equilibrium
after $h^t(\omega)$.
Applying Lemma~\ref{lem:infinite_patching} at each such $t$ yields a continuation profile
$\hat f^{\varepsilon,t,\omega}$ that is an $\varepsilon$-Nash equilibrium after
$h^t(\omega)$ and satisfies
\[
d_{h^t(\omega)}\!\left(
\mu^f,
\mu^{\hat f^{\varepsilon,t,\omega}}
\right)\le \varepsilon/2\le \varepsilon.
\]
This is exactly the stated conclusion.
\end{proof}

\begin{proof}[Proof of Lemma \ref{lem:stage_br_stability}]
Fix player $i$, let $p,q\in\Delta(A_{-i})$, and suppose
$\alpha_i\in \mathrm{br}_i^\xi(q)$.

For any $\alpha_i\in\Delta(A_i)$ define
\[
\phi_{\alpha_i}(a_{-i})
:=
\sum_{a_i\in A_i}\alpha_i(a_i)\,u_i(a_i,a_{-i}),
\qquad a_{-i}\in A_{-i}.
\]
Since $u_i(a_i,a_{-i})\in[0,1]$, we have $\phi_{\alpha_i}(a_{-i})\in[0,1]$ for all
$a_{-i}\in A_{-i}$. Also,
\[
u_i(\alpha_i,p)-u_i(\alpha_i,q)
=
\sum_{a_{-i}\in A_{-i}} \phi_{\alpha_i}(a_{-i})\bigl(p(a_{-i})-q(a_{-i})\bigr).
\]

Set
\[
S^+ := \{a_{-i}\in A_{-i}: p(a_{-i})\ge q(a_{-i})\}.
\]
Because $0\le \phi_{\alpha_i}\le 1$, we have
\begin{align*}
u_i(\alpha_i,p)-u_i(\alpha_i,q)
&=
\sum_{a_{-i}\in S^+}\phi_{\alpha_i}(a_{-i})\bigl(p(a_{-i})-q(a_{-i})\bigr)
+
\sum_{a_{-i}\notin S^+}\phi_{\alpha_i}(a_{-i})\bigl(p(a_{-i})-q(a_{-i})\bigr)\\
&\le
\sum_{a_{-i}\in S^+}\bigl(p(a_{-i})-q(a_{-i})\bigr)\\
&=
p(S^+)-q(S^+)\\
&\le \|p-q\|_{\TV}.
\end{align*}
Applying the same argument with $p$ and $q$ interchanged yields
\[
u_i(\alpha_i,q)-u_i(\alpha_i,p)\le \|p-q\|_{\TV}.
\]
Therefore
\begin{equation}
\label{eq:stage_tv_bound}
|u_i(\alpha_i,p)-u_i(\alpha_i,q)|\le \|p-q\|_{\TV}
\qquad\text{for every }\alpha_i\in\Delta(A_i).
\end{equation}

Now suppose $\alpha_i\in \mathrm{br}_i^\xi(q)$. Then
\[
u_i(\alpha_i,q)\ge \sup_{\alpha_i'\in\Delta(A_i)} u_i(\alpha_i',q)-\xi.
\]
Using \eqref{eq:stage_tv_bound},
\begin{align*}
u_i(\alpha_i,p)
&\ge u_i(\alpha_i,q)-\|p-q\|_{\TV}\\
&\ge \sup_{\alpha_i'\in\Delta(A_i)} u_i(\alpha_i',q)-\xi-\|p-q\|_{\TV}\\
&\ge \sup_{\alpha_i'\in\Delta(A_i)}
\bigl(u_i(\alpha_i',p)-\|p-q\|_{\TV}\bigr)-\xi-\|p-q\|_{\TV}\\
&=
\sup_{\alpha_i'\in\Delta(A_i)} u_i(\alpha_i',p)-\xi-2\|p-q\|_{\TV}.
\end{align*}
Hence
\[
\alpha_i\in \mathrm{br}_i^{\xi+2\|p-q\|_{\TV}}(p).
\]
\end{proof}

\begin{proof}[Proof of Lemma \ref{lem:myopic_psbr_gap}]
Fix player $i$ and history $h^t$. For each $g_{-i}\in\mathcal S_{-i}$ define
\[
M(g_{-i})
:=
\sup_{\alpha_i\in\Delta(A_i)}u_i\!\bigl(\alpha_i,g_{-i}(h^t)\bigr)\in[0,1].
\]
By Definition~\ref{def:myopic_psbr}, for each $g_{-i}\in\mathcal S_{-i}$ we have chosen
\[
\alpha_i^{g_{-i},h^t}\in \mathrm{br}_i\!\bigl(g_{-i}(h^t)\bigr),
\]
so
\[
u_i\!\bigl(\alpha_i^{g_{-i},h^t},g_{-i}(h^t)\bigr)=M(g_{-i}).
\]

Write $p_t(g_{-i})=\mu_i^t(g_{-i}\mid h^t)$.
The ex ante mixed action induced by myopic PS-BR is
\[
\alpha_{i,t}^{\mathrm{mPS}}(\cdot\mid h^t)
=
\sum_{\tilde g_{-i}\in\mathcal S_{-i}}
p_t(\tilde g_{-i})\,\alpha_i^{\tilde g_{-i},h^t}(\cdot),
\]
and the one-step posterior predictive belief is
\[
q_i^t(\cdot\mid h^t)
=
\sum_{g_{-i}\in\mathcal S_{-i}}
p_t(g_{-i})\,g_{-i}(h^t)(\cdot).
\]
By bilinearity of $u_i(\cdot,\cdot)$,
\begin{align*}
u_i\!\bigl(\alpha_{i,t}^{\mathrm{mPS}},q_i^t\bigr)
&=
\sum_{\tilde g_{-i}\in\mathcal S_{-i}} p_t(\tilde g_{-i})
\sum_{g_{-i}\in\mathcal S_{-i}} p_t(g_{-i})\,
u_i\!\bigl(\alpha_i^{\tilde g_{-i},h^t},g_{-i}(h^t)\bigr)\\
&\ge
\sum_{g_{-i}\in\mathcal S_{-i}} p_t(g_{-i})^2\,
u_i\!\bigl(\alpha_i^{g_{-i},h^t},g_{-i}(h^t)\bigr)\\
&=
\sum_{g_{-i}\in\mathcal S_{-i}} p_t(g_{-i})^2\,M(g_{-i}).
\end{align*}
On the other hand, again by bilinearity,
\begin{align*}
\sup_{\alpha_i\in\Delta(A_i)}u_i(\alpha_i,q_i^t)
&=
\sup_{\alpha_i\in\Delta(A_i)}
\sum_{g_{-i}\in\mathcal S_{-i}} p_t(g_{-i})\,u_i\!\bigl(\alpha_i,g_{-i}(h^t)\bigr)\\
&\le
\sum_{g_{-i}\in\mathcal S_{-i}} p_t(g_{-i})\,
\sup_{\alpha_i\in\Delta(A_i)}u_i\!\bigl(\alpha_i,g_{-i}(h^t)\bigr)\\
&=
\sum_{g_{-i}\in\mathcal S_{-i}} p_t(g_{-i})\,M(g_{-i}).
\end{align*}
Subtracting,
\begin{align*}
\sup_{\alpha_i}u_i(\alpha_i,q_i^t)-u_i(\alpha_{i,t}^{\mathrm{mPS}},q_i^t)
&\le
\sum_{g_{-i}\in\mathcal S_{-i}}
\bigl(p_t(g_{-i})-p_t(g_{-i})^2\bigr)\,M(g_{-i})\\
&\le
\sum_{g_{-i}\in\mathcal S_{-i}}
\bigl(p_t(g_{-i})-p_t(g_{-i})^2\bigr)\\
&=
1-\sum_{g_{-i}\in\mathcal S_{-i}}p_t(g_{-i})^2\\
&=
D_i^t(h^t).
\end{align*}
This proves the claim.
\end{proof}

\begin{proof}[Proof of Lemma \ref{lem:myopic_equiv_gap}]
Fix a history $h^t$ and write
\[
p_t(g_{-i}) := \mu_i^t(g_{-i}\mid h^t),
\qquad
\mathcal E^{\mathrm{st}} := \mathcal E_i^{\mathrm{st}}(h^t),
\qquad
p_{\mathcal E}^{\mathrm{st}}:=\mu_i^t(\mathcal E_i^{\mathrm{st}}(h^t)\mid h^t)=1-\beta_i^t(h^t).
\]
For each $g_{-i}\in\mathcal S_{-i}$, let
$\alpha_i^{g_{-i},h^t}\in \mathrm{br}_i(g_{-i}(h^t))$ be the mixed action selected in
Definition~\ref{def:myopic_psbr}. If $g_{-i}\in\mathcal E^{\mathrm{st}}$, then by
Definition~\ref{def:stage_equiv}, $g_{-i}(h^t)$ induces exactly the same one-step stage
optimization problem as the true current opponents' mixed action $f_{-i}(h^t)$. Hence
$\alpha_i^{g_{-i},h^t}$ is also a best response to $f_{-i}(h^t)$.
Let
\[
M_i(h^t):=\sup_{\alpha_i\in\Delta(A_i)} u_i(\alpha_i,f_{-i}(h^t)).
\]
Using linearity of stage payoffs in the player's own mixed action,
\begin{align*}
u_i\!\bigl(\alpha_{i,t}^{\mathrm{mPS}},f_{-i}(h^t)\bigr)
&=
\sum_{g_{-i}\in\mathcal S_{-i}} p_t(g_{-i})
u_i\!\bigl(\alpha_i^{g_{-i},h^t},f_{-i}(h^t)\bigr) \\
&\ge
\sum_{g_{-i}\in\mathcal E^{\mathrm{st}}} p_t(g_{-i}) M_i(h^t) \\
&=
(1-\beta_i^t(h^t)) M_i(h^t).
\end{align*}
Since $M_i(h^t)\le 1$, we obtain
\[
M_i(h^t) - u_i\!\bigl(\alpha_{i,t}^{\mathrm{mPS}},f_{-i}(h^t)\bigr)
\le \beta_i^t(h^t),
\]
which is exactly the stated bound.
\end{proof}

\begin{proof}[Proof of Lemma \ref{lem:one_step_predict_accuracy}]
Fix player $i$ and let $f^i=(f_i,f_{-i}^i)$ be the supporting profile from
Definition~\ref{def:learn_predict_path}.
Fix a realized path $z\in H^\infty$ in the full-measure event from
Definition~\ref{def:learn_predict_path}. By definition of $q_i^t$ and the
representative choice \eqref{eq:app_rep_choice},
\[
q_i^t(\cdot\mid h^t(z))=f_{-i}^{i,t}(h^t(z))=f_{-i}^i(h^t(z)).
\]
Let $\eta>0$.
By Definition~\ref{def:learn_predict_path}, there exists
$T_i(z,\eta/2)<\infty$ such that for all $t\ge T_i(z,\eta/2)$,
\[
d_{h^t(z)}(\mu^f,\mu^{f^i})\le \eta/2.
\]

Fix such a $t$. For any subset $B\subseteq A_{-i}$, define the one-step cylinder event
\[
E_B:=\{y\in H^\infty:\ y_{-i}^1\in B\}\in \mathcal B^1.
\]
By the definition of continuation measures,
\[
\mu^f_{h^t(z)}(E_B)=f_{-i}(h^t(z))(B),
\qquad
\mu^{f^i}_{h^t(z)}(E_B)=f_{-i}^i(h^t(z))(B)=q_i^t(B\mid h^t(z)).
\]
Therefore,
\begin{align*}
\big\|q_i^t(\cdot\mid h^t(z))-f_{-i}(h^t(z))\big\|_{\TV}
&=
\sup_{B\subseteq A_{-i}}
\left|
q_i^t(B\mid h^t(z))-f_{-i}(h^t(z))(B)
\right|\\
&=
\sup_{B\subseteq A_{-i}}
\left|
\mu^{f^i}_{h^t(z)}(E_B)-\mu^f_{h^t(z)}(E_B)
\right|\\
&\le
\sup_{E\in\mathcal B^1}
\left|
\mu^{f^i}_{h^t(z)}(E)-\mu^f_{h^t(z)}(E)
\right|.
\end{align*}
By Definition~\ref{def:weak_distance},
\[
d_{h^t(z)}(\mu^f,\mu^{f^i})
=
\sum_{k=1}^\infty 2^{-k}
\sup_{E\in\mathcal B^k}
\left|
\mu^f_{h^t(z)}(E)-\mu^{f^i}_{h^t(z)}(E)
\right|.
\]
In particular,
\[
\frac12
\sup_{E\in\mathcal B^1}
\left|
\mu^f_{h^t(z)}(E)-\mu^{f^i}_{h^t(z)}(E)
\right|
\le
d_{h^t(z)}(\mu^f,\mu^{f^i}),
\]
so
\[
\sup_{E\in\mathcal B^1}
\left|
\mu^f_{h^t(z)}(E)-\mu^{f^i}_{h^t(z)}(E)
\right|
\le
2\,d_{h^t(z)}(\mu^f,\mu^{f^i})
\le \eta.
\]
Hence
\[
\big\|q_i^t(\cdot\mid h^t(z))-f_{-i}(h^t(z))\big\|_{\TV}\le \eta
\]
for all $t\ge T_i(z,\eta/2)$.
Since $\eta>0$ was arbitrary, this proves the claim.
\end{proof}

\begin{proof}[Proof of Theorem \ref{thm:myopic_psbr_stage_nash}]
Fix $\varepsilon>0$.
For each player $i$, work on the full-measure event from
Assumption~\ref{ass:stage_equiv_class_conc}(2), so that along the realized path $z$,
\[
\beta_i^t(h^t(z))
:=
1-\mu_i^t(\mathcal E_i^{\mathrm{st}}(h^t(z))\mid h^t(z))
\longrightarrow 0.
\]
Because player $i$ uses myopic PS-BR, we have
\[
f_i(h^t(z))=\alpha_{i,t}^{\mathrm{mPS}}(\cdot\mid h^t(z)).
\]
Choose $T_i(z,\varepsilon)$ such that
\[
\beta_i^t(h^t(z))\le \varepsilon
\qquad\text{for all } t\ge T_i(z,\varepsilon).
\]
Then for every $t\ge T_i(z,\varepsilon)$, Lemma~\ref{lem:myopic_equiv_gap} implies
\[
f_i(h^t(z))
\in
\mathrm{br}_i^{\varepsilon}\!\bigl(f_{-i}(h^t(z))\bigr).
\]
Intersecting these full-measure events over the finite player set and taking
\[
T(z):=\max_{i\in I} T_i(z,\varepsilon)<\infty,
\]
we obtain that for all $t\ge T(z)$ and all players $i$,
\[
f_i(h^t(z))
\in
\mathrm{br}_i^{\varepsilon}\!\bigl(f_{-i}(h^t(z))\bigr).
\]
By Definition~\ref{def:stage_br_nash}, this means that $f(h^t(z))$ is a stage
$\varepsilon$-Nash equilibrium for all $t\ge T(z)$.
\end{proof}

\begin{proof}[Proof of Lemma \ref{lem:det_truth_implies_map_correct}]
Fix player $i$ and let $f^i=(f_i,f_{-i}^i)$ be the supporting profile from
Definition~\ref{def:learn_predict_path}.
Fix a realized path $z$ in the full-measure event from
Definition~\ref{def:learn_predict_path}. By definition of $q_i^t$ and the
representative choice \eqref{eq:app_rep_choice},
\[
q_i^t(\cdot\mid h^t(z))=f_{-i}^{i,t}(h^t(z))=f_{-i}^i(h^t(z)).
\]
For each $t$, define the one-step cylinder event
\[
E_t(z)
:=
\{y\in H^\infty:\ y_{-i}^1 = a_{-i}^{\star}(h^t(z))\}\in \mathcal B^1.
\]
Because the true opponents' next action at history $h^t(z)$ is pure,
\[
f_{-i}(h^t(z))=\delta_{a_{-i}^{\star}(h^t(z))},
\]
so
\[
\mu^f_{h^t(z)}(E_t(z))=1.
\]
Also, by the on-path identification above,
\[
\mu^{f^i}_{h^t(z)}(E_t(z))
=
f_{-i}^i(h^t(z))\bigl(a_{-i}^{\star}(h^t(z))\bigr)
=
q_i^t\!\bigl(a_{-i}^{\star}(h^t(z))\mid h^t(z)\bigr).
\]
Hence
\begin{align*}
1-q_i^t\!\bigl(a_{-i}^{\star}(h^t(z))\mid h^t(z)\bigr)
&=
\left|
\mu^f_{h^t(z)}(E_t(z))-\mu^{f^i}_{h^t(z)}(E_t(z))
\right|\\
&\le
\sup_{E\in\mathcal B^1}
\left|
\mu^f_{h^t(z)}(E)-\mu^{f^i}_{h^t(z)}(E)
\right|.
\end{align*}
As in the proof of Lemma~\ref{lem:one_step_predict_accuracy},
\[
\sup_{E\in\mathcal B^1}
\left|
\mu^f_{h^t(z)}(E)-\mu^{f^i}_{h^t(z)}(E)
\right|
\le
2\,d_{h^t(z)}(\mu^f,\mu^{f^i}).
\]
Because player $i$ learns to predict the path of play,
\[
d_{h^t(z)}(\mu^f,\mu^{f^i})\longrightarrow 0.
\]
Therefore
\[
q_i^t\!\bigl(a_{-i}^{\star}(h^t(z))\mid h^t(z)\bigr)\longrightarrow 1.
\]

It follows immediately that
\[
1-\max_{a_{-i}\in A_{-i}} q_i^t(a_{-i}\mid h^t(z))
\le
1-q_i^t\!\bigl(a_{-i}^{\star}(h^t(z))\mid h^t(z)\bigr)
\longrightarrow 0,
\]
which proves asymptotic purity.

Finally, because
\[
q_i^t\!\bigl(a_{-i}^{\star}(h^t(z))\mid h^t(z)\bigr)\longrightarrow 1,
\]
there exists $T_i(z)<\infty$ such that for all $t\ge T_i(z)$,
\[
q_i^t\!\bigl(a_{-i}^{\star}(h^t(z))\mid h^t(z)\bigr)>\frac12.
\]
For such $t$, the action $a_{-i}^{\star}(h^t(z))$ is the unique maximizer of
$q_i^t(\cdot\mid h^t(z))$, because all other probabilities sum to
\[
1-q_i^t\!\bigl(a_{-i}^{\star}(h^t(z))\mid h^t(z)\bigr)<\frac12.
\]
Hence the deterministic MAP selector must satisfy
\[
\hat a_{-i}^t(h^t(z))=a_{-i}^{\star}(h^t(z))
\qquad\text{for all }t\ge T_i(z).
\]
This proves the claim.
\end{proof}

\begin{proof}[Proof of Theorem \ref{thm:det_scot_stage_nash}]
Because every player $j\in I$ uses deterministic MAP-SCoT, for every history $h\in H$ we have
\[
f_j(h)=\delta_{a_j^\star(h)}
\qquad\text{for some }a_j^\star(h)\in A_j.
\]
Hence for every player $i$ and every history $h$,
\[
f_{-i}(h)=\delta_{a_{-i}^\star(h)}
\qquad\text{for some }a_{-i}^\star(h)\in A_{-i}.
\]

For each player $i$, apply Lemma~\ref{lem:det_truth_implies_map_correct}.
There is a full-measure event on which there exists $T_i(z)<\infty$ such that
for all $t\ge T_i(z)$,
\[
\hat a_{-i}^t(h^t(z))=a_{-i}^\star(h^t(z)).
\]
Because the player set $I$ is finite, the intersection of these full-measure events
over all players still has measure one.

Fix a realized path $z$ in that intersection.
For any player $i$ and any $t\ge T_i(z)$, Definition~\ref{def:det_map_scot} gives
\[
f_i(h^t(z))
=
\delta_{\,b_i(\hat a_{-i}^t(h^t(z)))}
=
\delta_{\,b_i(a_{-i}^\star(h^t(z)))}.
\]
By definition of the pure best-response selector $b_i$,
\[
b_i(a_{-i})\in \arg\max_{a_i\in A_i} u_i(a_i,a_{-i})
\qquad\text{for every }a_{-i}\in A_{-i}.
\]
Therefore
\[
\delta_{\,b_i(a_{-i}^\star(h^t(z)))}
\in
\mathrm{br}_i\!\bigl(\delta_{a_{-i}^\star(h^t(z))}\bigr)
=
\mathrm{br}_i\!\bigl(f_{-i}(h^t(z))\bigr).
\]
So for every player $i$ and all $t\ge T_i(z)$,
\[
f_i(h^t(z))\in \mathrm{br}_i\!\bigl(f_{-i}(h^t(z))\bigr).
\]

Define
\[
T(z):=\max_{i\in I} T_i(z)<\infty.
\]
Then for all $t\ge T(z)$ and every player $i$,
\[
f_i(h^t(z))\in \mathrm{br}_i\!\bigl(f_{-i}(h^t(z))\bigr).
\]
By Definition~\ref{def:stage_br_nash}, this means that
$f(h^t(z))$ is a stage Nash equilibrium for all $t\ge T(z)$.
\end{proof}

\begin{proof}[Proof of Corollary \ref{cor:det_scot_stage_nash_grain_truth}]
By Lemma~\ref{lem:absolute_cont_predict}, Assumption~\ref{ass:grain_of_truth}
implies that every player learns to predict the path of play under $f$ in the sense of
Definition~\ref{def:learn_predict_path}.
Theorem~\ref{thm:det_scot_stage_nash} therefore applies directly.
\end{proof}

\clearpage


\section{Bounded-memory strategies and finite-state reduction}
\label{sec:bounded_memory}

Many practical agent policies (including menu-based planners) depend only on a bounded
window of recent interaction. Following the bounded-recall restriction in
\citet{norman2022possibility}, we formalize this as a \emph{bounded-memory} condition.

For a history $h=(a^1,\dots,a^{t-1})\in H$ let $|h|:=t-1$ denote its length.  For
$\kappa\in\mathbb N$, define
\[
\mathrm{suffix}_\kappa(h)\ :=\ (a^{t-\min\{\kappa,t-1\}},\dots,a^{t-1}) \in \bigcup_{m=0}^{\kappa} A^m,
\]
i.e., the last $\min\{\kappa,|h|\}$ joint actions of $h$ (with $\mathrm{suffix}_\kappa(\emptyset)=\emptyset$).

\begin{definition}[$\kappa$-memory (bounded-recall) strategy]
\label{def:kappa_memory_strategy}
A strategy $f_i:H\to\Delta(A_i)$ has \emph{memory at most $\kappa$} if for all histories
$h,h'\in H$,
\[
\mathrm{suffix}_\kappa(h)=\mathrm{suffix}_\kappa(h') \quad\Longrightarrow\quad f_i(h)=f_i(h').
\]
Let $\mathcal F_i^\kappa\subseteq\mathcal F_i$ denote the set of $\kappa$-memory strategies
for player $i$, and write $\mathcal F^\kappa:=\prod_{i\in I}\mathcal F_i^\kappa$.
\end{definition}

Let
\[
\mathsf S_\kappa\ :=\ \bigcup_{m=0}^{\kappa} A^m
\]
be the finite set of action-suffixes of length at most $\kappa$.  Define the deterministic
state update map $T_\kappa:\mathsf S_\kappa\times A\to\mathsf S_\kappa$ by
\[
T_\kappa(s,a)\ :=\ \mathrm{suffix}_\kappa((s,a)),
\]
i.e., append the new joint action $a$ to the suffix $s$ and keep the last $\kappa$ entries.
For any play path $z=(a^1,a^2,\dots)\in H^\infty$, define the induced memory state at time $t$:
\[
s^t(z)\ :=\ \mathrm{suffix}_\kappa(h^t(z))\ \in\ \mathsf S_\kappa.
\]

\begin{lemma}[Finite-state Markov property under bounded memory]
\label{lem:markov_property_bounded_memory}
If $f\in\mathcal F^\kappa$, then for every $t\ge 1$ and every history $h^t$ with
$s=\mathrm{suffix}_\kappa(h^t)$, the next-period action distribution depends on $h^t$ only
through $s$:
\[
\mu^f(a^t=a\mid h^t)\ =\ \prod_{i\in I} f_i(s)(a_i).
\]
Moreover, the induced state process satisfies $s^{t+1}=T_\kappa(s^t,a^t)$ almost surely,
so $(s^t)_{t\ge 1}$ is a time-homogeneous Markov chain on $\mathsf S_\kappa$.
\end{lemma}

\begin{proof}
Fix $t$ and history $h^t$. By Definition~\ref{def:path_distribution},
\[
\mu^f(a^t=a\mid h^t)=\prod_{i\in I} f_i(h^t)(a_i).
\]
If $f\in\mathcal F^\kappa$, then $f_i(h^t)=f_i(\mathrm{suffix}_\kappa(h^t))=f_i(s)$ for each
$i$, giving the displayed equality. The state update is deterministic by construction of
$T_\kappa$: $s^{t+1}=\mathrm{suffix}_\kappa(h^{t+1})=\mathrm{suffix}_\kappa((h^t,a^t))
=T_\kappa(\mathrm{suffix}_\kappa(h^t),a^t)=T_\kappa(s^t,a^t)$.
Thus $(s^t)$ is Markov with kernel induced by the conditional law of $a^t$ given $s^t$.
\end{proof}

\begin{lemma}[Continuation distributions depend only on the memory state]
\label{lem:continuation_depends_on_state}
Let $g\in\mathcal F^\kappa$ and let $h,h'\in H$ satisfy
$\mathrm{suffix}_\kappa(h)=\mathrm{suffix}_\kappa(h')$. Then the continuation play-path
distributions coincide:
\[
\mu^{g}_{h}\ =\ \mu^{g}_{h'}.
\]
\end{lemma}

\begin{proof}
By Lemma~\ref{lem:markov_property_bounded_memory}, the conditional distribution of the next
action profile and all future evolution under $g$ depends on the past only through the
current memory state $s=\mathrm{suffix}_\kappa(\cdot)$. Since $h$ and $h'$ induce the same
state, the induced kernels for $(a^t,a^{t+1},\dots)$ are identical from either starting
history. Therefore the induced continuation measures coincide.
\end{proof}

\subsection{Best responses to bounded-memory opponents are bounded-memory}

A key benefit of bounded-memory opponents is that each player faces a finite-state
discounted MDP in the continuation game. In particular, the best-response search in
$\mathrm{BR}_i^\varepsilon(g_{-i}\mid h^t)$ can be restricted without loss to bounded-memory
policies.

\begin{lemma}[Markovian best responses to $\kappa$-memory opponents]
\label{lem:markov_best_response}
Fix player $i$, a history $h^t$, and an opponents' continuation profile
$g_{-i}\in\mathcal F_{-i}^\kappa$. Then there exists a best response
$\sigma_i^\star\in\mathrm{BR}_i(g_{-i}\mid h^t)$ that is \emph{stationary Markov} with respect
to the memory state. That is, there exists a map $\pi_i:\mathsf S_\kappa\to\Delta(A_i)$ such
that for every continuation history $\bar h\succeq h^t$,
\[
\sigma_i^\star(\bar h)\ =\ \pi_i(\mathrm{suffix}_\kappa(\bar h)).
\]
Consequently, for every $\varepsilon\ge 0$,
\[
\sup_{\sigma_i\in\mathcal F_i(h^t)} V_i(\sigma_i\mid h^t;g_{-i})
\ =\
\sup_{\sigma_i\in\mathcal F_i^\kappa(h^t)} V_i(\sigma_i\mid h^t;g_{-i}),
\]
and $\mathrm{BR}_i(g_{-i}\mid h^t)\cap\mathcal F_i^\kappa(h^t)\neq\emptyset$.
\end{lemma}

\begin{proof}
Let $s_0:=\mathrm{suffix}_\kappa(h^t)\in\mathsf S_\kappa$.  Fix $g_{-i}\in\mathcal F_{-i}^\kappa$.
Define a controlled Markov process on $\mathsf S_\kappa$ as follows. In state $s$, the player
chooses $a_i\in A_i$, the opponents' joint action is drawn as
$a_{-i}\sim g_{-i}(s)\in\Delta(A_{-i})$, the stage payoff is $u_i(a_i,a_{-i})$, and the next
state is $s' = T_\kappa(s,(a_i,a_{-i}))$.

For any bounded function $v:\mathsf S_\kappa\to\mathbb R$, define the Bellman operator
$\mathcal T$ by
\[
(\mathcal T v)(s)
:=
\max_{\alpha\in\Delta(A_i)}
\mathbb E_{\substack{a_i\sim \alpha\\ a_{-i}\sim g_{-i}(s)}}
\Big[
(1-\lambda_i)\,u_i(a_i,a_{-i})
+
\lambda_i\, v\!\big(T_\kappa(s,(a_i,a_{-i}))\big)
\Big].
\]
Because $\lambda_i\in(0,1)$, $\mathcal T$ is a contraction in $\|\cdot\|_\infty$:
for any $v,w$ and any $s$,
\[
|(\mathcal T v)(s)-(\mathcal T w)(s)|
\le
\max_{\alpha}\mathbb E\big[\lambda_i\,|v(s')-w(s')|\big]
\le
\lambda_i\|v-w\|_\infty.
\]
Hence $\mathcal T$ has a unique fixed point $V^\star:\mathsf S_\kappa\to\mathbb R$.

For each $s$, the maximization over $\alpha\in\Delta(A_i)$ attains its maximum because
$\Delta(A_i)$ is compact and the objective is continuous and linear in $\alpha$.
Fix a maximizer $\pi_i(s)\in\Delta(A_i)$ for each $s$ and define the associated
\emph{policy evaluation} operator
\[
(\mathcal T_{\pi_i} v)(s)
:=
\mathbb E_{\substack{a_i\sim \pi_i(s)\\ a_{-i}\sim g_{-i}(s)}}
\Big[
(1-\lambda_i)\,u_i(a_i,a_{-i})
+
\lambda_i\, v\!\big(T_\kappa(s,(a_i,a_{-i}))\big)
\Big].
\]
Then $(\mathcal T_{\pi_i} V^\star)(s)=(\mathcal T V^\star)(s)=V^\star(s)$ for all $s$, so
$V^\star$ is a fixed point of $\mathcal T_{\pi_i}$. Since $\mathcal T_{\pi_i}$ is also a
$\lambda_i$-contraction, its fixed point is unique; denote it by $V^{\pi_i}$. We conclude
$V^{\pi_i}=V^\star$.

Now define $\sigma_i^\star$ to be the stationary Markov continuation strategy induced by
$\pi_i$, i.e.\ $\sigma_i^\star(\bar h)=\pi_i(\mathrm{suffix}_\kappa(\bar h))$ for all
$\bar h\succeq h^t$. By construction, the induced continuation value from $h^t$ is
$V_i(\sigma_i^\star\mid h^t;g_{-i})=V^\star(s_0)$.

It remains to show optimality against \emph{all} continuation strategies, including those
with unbounded memory. Let $\sigma_i$ be any continuation strategy and define its
\emph{statewise value envelope}
\[
W_{\sigma_i}(s)\ :=\ \sup\Big\{V_i(\sigma_i\mid \bar h;g_{-i}): \bar h\succeq h^t,\ 
\mathrm{suffix}_\kappa(\bar h)=s\Big\}.
\]
Fix any $s$ and $\epsilon>0$, and choose $\bar h$ with $\mathrm{suffix}_\kappa(\bar h)=s$ and
$V_i(\sigma_i\mid \bar h;g_{-i})\ge W_{\sigma_i}(s)-\epsilon$.
Let $\alpha:=\sigma_i(\bar h)\in\Delta(A_i)$ be the first-step mixed action.
Conditioning on the first joint action $(a_i,a_{-i})$ and using that the next state is
$s'=T_\kappa(s,(a_i,a_{-i}))$, we have
\begin{align*}
V_i(\sigma_i\mid \bar h;g_{-i})
&=
\mathbb E\Big[(1-\lambda_i)u_i(a_i,a_{-i})
+\lambda_i\,V_i(\sigma_i\mid (\bar h,(a_i,a_{-i}));g_{-i})\Big]\\
&\le
\mathbb E\Big[(1-\lambda_i)u_i(a_i,a_{-i})
+\lambda_i\,W_{\sigma_i}(s')\Big].
\end{align*}
Therefore,
\[
W_{\sigma_i}(s)-\epsilon
\le
\mathbb E_{\substack{a_i\sim \alpha\\ a_{-i}\sim g_{-i}(s)}}
\Big[(1-\lambda_i)u_i(a_i,a_{-i})+\lambda_i W_{\sigma_i}(T_\kappa(s,(a_i,a_{-i})))\Big]
\le
(\mathcal T W_{\sigma_i})(s).
\]
Letting $\epsilon\downarrow 0$ gives $W_{\sigma_i}\le \mathcal T W_{\sigma_i}$ pointwise.
By monotonicity of $\mathcal T$ and contraction, iterating yields
$W_{\sigma_i}\le \mathcal T^n W_{\sigma_i}$ for all $n$, and $\mathcal T^n W_{\sigma_i}\to V^\star$
uniformly as $n\to\infty$. Hence $W_{\sigma_i}(s)\le V^\star(s)$ for all $s$, and in
particular
\[
V_i(\sigma_i\mid h^t;g_{-i})\ \le\ W_{\sigma_i}(s_0)\ \le\ V^\star(s_0)
\ =\ V_i(\sigma_i^\star\mid h^t;g_{-i}).
\]
Thus $\sigma_i^\star$ is a best response. The final displayed equality of suprema follows
because an optimal policy exists within $\mathcal F_i^\kappa(h^t)$.
\end{proof}

\subsection{A checkable route to on-path elimination under bounded memory}

Assumption~\ref{ass:finite_menu_kl}(2) requires only that every wrong retained hypothesis
becomes negligible relative to the true one along the realized path. This appendix-level
condition is itself stronger than necessary if two retained labels are continuation-payoff-
equivalent from reached histories, but it is a convenient route to posterior concentration for
our sampled-label PS-BR proof. Under bounded memory, recurrent-state KL separation provides
a transparent sufficient route to this eventual on-path elimination.

\begin{lemma}[State-frequency decomposition of on-path KL averages]
\label{lem:kl_state_frequency}
Fix player $i$, $\kappa\in\mathbb N$, and $f_{-i},g_{-i}\in\mathcal F_{-i}^\kappa$.
For a realized path $z$, define $s^t(z)=\mathrm{suffix}_\kappa(h^t(z))$ and empirical state
frequencies
\[
\hat\pi_T^z(s)\ :=\ \frac{1}{T}\sum_{t=1}^T \mathbf 1\{s^t(z)=s\},\qquad s\in\mathsf S_\kappa.
\]
Then for every $T$ and every $z$,
\[
\frac{1}{T}\sum_{t=1}^T 
D_{\mathrm{KL}}\!\Big(f_{-i}(h^t(z))\ \Big\|\ g_{-i}(h^t(z))\Big)
\;=\;
\sum_{s\in\mathsf S_\kappa} \hat\pi_T^z(s)\,
D_{\mathrm{KL}}\!\Big(f_{-i}(s)\ \Big\|\ g_{-i}(s)\Big).
\]
In particular, for any fixed state $s$,
\[
\liminf_{T\to\infty}\ \frac{1}{T}\sum_{t=1}^T 
D_{\mathrm{KL}}\!\Big(f_{-i}(h^t(z))\ \Big\|\ g_{-i}(h^t(z))\Big)
\ \ge\
\Big(\liminf_{T\to\infty}\hat\pi_T^z(s)\Big)\cdot
D_{\mathrm{KL}}\!\Big(f_{-i}(s)\ \Big\|\ g_{-i}(s)\Big).
\]
\end{lemma}

\begin{proof}
If $f_{-i},g_{-i}\in\mathcal F_{-i}^\kappa$, then for each $t$ we have
$f_{-i}(h^t(z))=f_{-i}(s^t(z))$ and $g_{-i}(h^t(z))=g_{-i}(s^t(z))$ by
Definition~\ref{def:kappa_memory_strategy}. Therefore,
\[
\frac{1}{T}\sum_{t=1}^T 
D_{\mathrm{KL}}\!\Big(f_{-i}(h^t(z))\ \Big\|\ g_{-i}(h^t(z))\Big)
=
\frac{1}{T}\sum_{t=1}^T 
D_{\mathrm{KL}}\!\Big(f_{-i}(s^t(z))\ \Big\|\ g_{-i}(s^t(z))\Big).
\]
Grouping the sum by the value of $s^t(z)$ yields the stated decomposition. The inequality
follows by lower bounding the sum by a single state's contribution and taking $\liminf$.
\end{proof}

\begin{corollary}[A sufficient condition for on-path elimination under bounded memory]
\label{cor:bounded_memory_implies_kl_sep}
Fix player $i$ and suppose $\mathcal S_{-i}\subseteq \mathcal F_{-i}^\kappa$.
Fix $g_{-i}\in\mathcal S_{-i}\setminus\{f_{-i}\}$ and define
\[
\tau_i(g_{-i})
:=
\inf\{t\ge 1:\ g_{-i}(h^t)(a_{-i}^t)=0\},
\]
with the convention $\inf\emptyset=\infty$.
Assume that $\mu^f$-a.s.\ one of the following holds:
\begin{enumerate}[label=(\alph*),leftmargin=*]
\item \textit{(Hard refutation)} $\tau_i(g_{-i})<\infty$.
\item \textit{(Common-support recurrent separation)} On the survival event $\{\tau_i(g_{-i})=\infty\}$,
common support holds on the visited suffix states, and there exists
$s\in\mathsf S_\kappa$ such that
\[
D_{\mathrm{KL}}\!\bigl(f_{-i}(s)\|g_{-i}(s)\bigr)>0,
\qquad
\liminf_{T\to\infty}\hat\pi_T^z(s)\ge \rho_i(g_{-i})>0.
\]
\end{enumerate}
Then Assumption~\ref{ass:finite_menu_kl}(2) holds for this $g_{-i}$.
\end{corollary}

\begin{proof}
If case (a) holds, then along any such path $z$ the denominator likelihood
$L_{i,T}(g_{-i};z)$ becomes zero from time $\tau_i(g_{-i})(z)+1$ onward while
$L_{i,T}(f_{-i};z)>0$, so
\[
\frac{L_{i,T}(g_{-i};z)}{L_{i,T}(f_{-i};z)}=0
\]
for all sufficiently large $T$. Hence Assumption~\ref{ass:finite_menu_kl}(2) is immediate.

Now consider case (b). On the survival event $\{\tau_i(g_{-i})=\infty\}$, define
\[
X_t(z):=\log\frac{f_{-i}(h^t(z))(a_{-i}^t(z))}{g_{-i}(h^t(z))(a_{-i}^t(z))}.
\]
Because the menu and suffix-state space are finite and common support holds on the visited
states, the positive probabilities that can appear in the denominator are bounded away from
zero on this event. Thus $(X_t)$ is uniformly bounded there. Let
\[
Y_t:=X_t-\mathbb E_{\mu^f}[X_t\mid \mathcal F_t].
\]
Then $(Y_t,\mathcal F_t)$ is a bounded martingale-difference sequence, so by the martingale
strong law,
\[
\frac{1}{T}\sum_{t=1}^T Y_t\longrightarrow 0
\qquad\text{$\mu^f$-a.s. on }\{\tau_i(g_{-i})=\infty\}.
\]
Moreover, on the survival event,
\[
\mathbb E_{\mu^f}[X_t\mid \mathcal F_t]
=
D_{\mathrm{KL}}\!\Big(f_{-i}(h^t)\ \Big\|\ g_{-i}(h^t)\Big).
\]
Therefore,
\[
\frac{1}{T}\log\frac{L_{i,T}(f_{-i};z)}{L_{i,T}(g_{-i};z)}
=
\frac{1}{T}\sum_{t=1}^{T-1}X_t
=
\frac{1}{T}\sum_{t=1}^{T-1}D_{\mathrm{KL}}\!\Big(f_{-i}(h^t)\ \Big\|\ g_{-i}(h^t)\Big)+o(1).
\]
Lemma~\ref{lem:kl_state_frequency} and the positive-frequency assumption on state $s$ imply
\[
\liminf_{T\to\infty}\frac{1}{T}\log\frac{L_{i,T}(f_{-i};z)}{L_{i,T}(g_{-i};z)}
\ge
\rho_i(g_{-i})\cdot D_{\mathrm{KL}}\!\bigl(f_{-i}(s)\|g_{-i}(s)\bigr)
>0
\]
$\mu^f$-a.s.\ on $\{\tau_i(g_{-i})=\infty\}$.
Hence
\[
\log\frac{L_{i,T}(g_{-i};z)}{L_{i,T}(f_{-i};z)}\longrightarrow -\infty,
\qquad\text{so}\qquad
\frac{L_{i,T}(g_{-i};z)}{L_{i,T}(f_{-i};z)}\longrightarrow 0
\]
$\mu^f$-a.s.\ on $\{\tau_i(g_{-i})=\infty\}$.
Therefore Assumption~\ref{ass:finite_menu_kl}(2) holds in case (b) as well.
\end{proof}

All statements in Sections~\ref{sec:rr}--\ref{sec:zero_shot} are formulated on the full
history space $H$ and therefore apply without modification when the realized profile $f$ (and/or
the menu strategies in the appendix-level retained-menu identification condition,
Assumption~\ref{ass:finite_menu_kl}) lie in $\mathcal F^\kappa$.
Relative to the main-text concentration-up-to-equivalence assumption
(Assumption~\ref{ass:equiv_class_conc}), the appendix records stronger sufficient retained-menu
identification conditions (Assumption~\ref{ass:finite_menu_kl} and the deterministic hard-refutation condition
Assumption~\ref{ass:det_menu_sep_app}) and adds two finite-state tools:
(i) best responses to $\kappa$-memory opponents can be taken to be stationary Markov
(Lemma~\ref{lem:markov_best_response}); and
(ii) on-path elimination can be verified by recurrent-state KL separation under common support or by finite-time hard refutation
(Lemma~\ref{lem:kl_state_frequency} and Corollary~\ref{cor:bounded_memory_implies_kl_sep}).
Thus bounded memory is not required for the main concentration result, but it gives a transparent finite-state route for checking stronger sufficient conditions for the appendix-level PS-BR identification argument.

\clearpage

\clearpage

\section{Implementation details of the strategy-level PS-BR planner}
\label{app:impl_strategy_psbr}

This appendix details the implementation used in our experiments. At each round, an agent samples a latent opponent strategy from its inference based on the previous history, evaluates candidate self-strategies by rollout, and plays the current action
induced by the best rollout-value strategy.

\subsection{Opponent strategy sampling}

Fix player $i$ at round $t$ with local history
$h_i^t=((a_i^1,a_{-i}^1),\ldots,(a_i^{t-1},a_{-i}^{t-1}))$.
For opponent-strategy inference, the implementation rewrites this to the opponent-view history
\[
\tilde h_{-i}^t=((a_{-i}^1,a_i^1),\ldots,(a_{-i}^{t-1},a_i^{t-1})),
\]
so each tuple is ordered as \emph{(opponent action, your action)}.
The opponent strategy inference is performed \emph{once per real decision round} (with configured label-sampling temperature) and then held
fixed across all $K$ rollout samples used to evaluate candidate self-strategies at that round.
Inference supports two modes:
\begin{itemize}[leftmargin=*]
    \item \textbf{\texttt{llm-label} (default):}
    construct an in-context prompt containing the game rules, observed history, and the
    allowed strategy labels (with short descriptions), then ask the model to output
    \emph{exactly one label}. Parsing is label-constrained; if parsing fails repeatedly, a
    deterministic label fallback is used.
    \item \textbf{\texttt{likelihood}:}
    infer from a hand-coded likelihood over the menu (described below), with no model call.
\end{itemize}

\paragraph{\texttt{llm-label} mode details.}

In \texttt{llm-label} mode, if the model call itself fails, the implementation falls back to
\texttt{likelihood} mode for that decision round.

The template used in code is:
\begin{quote}
\small
\begin{verbatim}
{rules_text}
Observed action history tuple format: (opponent action, your action).
Infer the opponent strategy from the FIRST action in each tuple.
Round 1: {opp_action_1}, {self_action_1}
Round 2: {opp_action_2}, {self_action_2}
...

You are inferring the opponent strategy in repeated {game_name}.
Observed rounds so far: {observed_rounds}.
Objective: sample one opponent strategy label according to your
posterior belief over allowed labels.
Estimate that posterior using ALL observed rounds 
(do not ignore older rounds), and focus on recent patterns.
The opponent may change strategy over time; if you detect a shift, 
prioritize the most recent consistent behavior while still 
accounting for earlier rounds.
Internally assign a compatibility score from 0 to 100 to every 
allowed label, convert them into relative posterior weights, and
sample exactly one final label from those weights.
Output rule: do NOT output scores, reasoning, or ranking.
Respond with exactly one label only.

**Output only the label.**

Allowed labels:
- {label_1}: {description_1}
- {label_2}: {description_2}
...
\end{verbatim}
\end{quote}
where \texttt{game\_name} is the active repeated-game name (e.g., BoS, PD,
Promo, Samaritan's dilemma, or Lemons), and
\texttt{observed\_rounds}=t-1.

When collusive-prior guidance is enabled (\texttt{--collusive-mode}), the prompt appends a strong-prior line. In our code this prior is \texttt{mad0} for Promo opponent 1 and \texttt{mad1} for Promo opponent 2.

\paragraph{\texttt{Likelihood}-mode details.}
To score strategy $s$, the implementation evaluates history under the
opponent's perspective
$\tilde h_{-i}^t=((a_{-i}^1,a_i^1),\ldots,(a_{-i}^{t-1},a_i^{t-1}))$:
\[
\log L_t(s)
=
\sum_{u=1}^{t-1}
\log\!\left(
\mathbf{1}\{a_{-i}^u=J\}p_s^u
+ \mathbf{1}\{a_{-i}^u=F\}(1-p_s^u)
\right),
\]
with clipping to $[10^{-6},1-10^{-6}]$ for numerical stability.
Given temperature $\tau>0$ (implemented as $\tau=\max\{\texttt{sample\_temperature},10^{-5}\}$),
weights are
\[
w_t(s)\propto \exp\!\left(\frac{\log L_t(s)}{\tau}\right),
\]
and one opponent strategy is sampled from this categorical distribution.

\subsection{Rollout value and strategy selection}

Given a sampled opponent strategy $s_{-i}$, for every candidate self-strategy
$s_i\in M_g$, the planner rolls out from round $t$ to $\bar t$, where
\[
\bar t=
\begin{cases}
\min\{T,\ t+H-1\}, & H>0,\\
T, & H=0,
\end{cases}
\]
$T$ is the game horizon, and $H$ is the planning horizon.

For rollout sample $m\in\{1,\dots,K\}$, at each simulated round $r$, actions are sampled
from the fixed opponent strategy $s_{-i}$ and the currently evaluated candidate $s_i$:
\[
\hat a_i^{r,m}\sim \mathrm{Bernoulli}\!\left(p_{s_i}^r\right),
\qquad
\hat a_{-i}^{r,m}\sim \mathrm{Bernoulli}\!\left(p_{s_{-i}}^r\right),
\]
where $p_{s_i}^r$ and $p_{s_{-i}}^r$ are the round-$r$ probabilities of action $J$ induced
by $s_i$ and $s_{-i}$ under the simulated history prefix generated so far. The rollout value
for candidate $s_i$ against sampled opponent strategy $s_{-i}$ is
\[
V_i^{(m)}(s_i\mid s_{-i})
=
\sum_{r=t}^{\bar t}
\gamma^{\,r-t}
u_i(\hat a_i^{r,m},\hat a_{-i}^{r,m}),
\]
with discount $\gamma$.

The estimated value of strategy $s_i$ is
\[
\bar V_i(s_i\mid s_{-i})=\frac{1}{K}\sum_{m=1}^K V_i^{(m)}(s_i\mid s_{-i}),
\]
and the chosen strategy is
\[
s_i^\star \in \arg\max_{s_i} \bar V_i(s_i\mid s_{-i}),
\]
with deterministic hash-based tie-breaking when needed.
The executed action at real round $t$ is then sampled from $s_i^\star$ at the current history.

\begin{algorithm}[t]
\caption{Strategy-level PS-BR loop for two-player games}
\label{alg:strategy_psbr_loop}
\begin{algorithmic}[1]
\Require game $g$, total rounds $T$, menu $M_g$, samples $K$, horizon $H$, discount $\gamma$, temperature $\tau$, inference mode $\in\{\texttt{llm-label},\texttt{likelihood}\}$
\State Initialize $h^1\gets\emptyset$, $x_1^1\gets(h^1,\emptyset)$, $x_2^1\gets(h^1,\emptyset)$, $C_1\gets0$, and $C_2\gets0$
\For{$t=1,\dots,T$}
    \For{$i\in\{1,2\}$}
        \State Let $x_i^t=(h^t,r_i^{1:t-1})$ be player $i$'s current local history
        \State Construct opponent-view history $\tilde h_{-i}^t$ by swapping tuple order in the public history $h^t$
        \State Infer one strategy label $s_{-i}\in M_g$ from rules, history $\tilde h_{-i}^t$
        \ForAll{$s_i\in M_g$}
            \For{$k=1,\dots,K$}
                \State $V_i^{(k)}(s_i\mid s_{-i})\gets \mathrm{RolloutValue}(g,i,s_i,s_{-i},x_i^t,t,T,H,\gamma)$
            \EndFor
            \State $\bar V_i(s_i\mid s_{-i})\gets \frac{1}{K}\sum_{k=1}^K V_i^{(k)}(s_i\mid s_{-i})$
        \EndFor
        \State $s_i^\star\gets \arg\max_{s_i\in M_g}\bar V_i(s_i\mid s_{-i})$ \Comment{deterministic tie-break}
        \State Sample real action $a_i^t$ from strategy $s_i^\star$ at history $x_i^t$
    \EndFor
    \State Sample realized rewards $(r_1^t,r_2^t)$ from the environment payoff law at $(a_1^t,a_2^t)$
    \State $C_1\gets C_1+r_1^t$ and $C_2\gets C_2+r_2^t$
    \State Set $h^{t+1}\gets(h^t,(a_1^t,a_2^t))$
    \State Set $x_1^{t+1}\gets(h^{t+1},r_1^{1:t})$ and $x_2^{t+1}\gets(h^{t+1},r_2^{1:t})$
\EndFor
\end{algorithmic}
\end{algorithm}

For Experiment~3, the environment payoff law in Algorithm~\ref{alg:strategy_psbr_loop} is the known Gaussian noise family centered at the true mean matrix. On the player's own side, player $i$ additionally samples $\tilde m_i\sim\pi_i^t(\cdot\mid x_i^t)$, rollout values are computed under $\tilde m_i$ in place of the true $u_i$, and player $i$'s local information history stores only $(h^t,r_i^{1:t-1})$; in particular, the update step above never reveals or conditions on $r_{-i}^{1:t-1}$.

\clearpage

\section{Social chain-of-thought prompting (SCoT)}
\label{app:scot}

This appendix shows that SCoT can be viewed as a special case of PS-BR.
The \emph{social chain-of-thought} prompting intervention of \citet{akata2025playing}
is a particularly simple two-stage ``predict-then-act'' instance of PS-BR.

\subsection{SCoT as a two-stage ``predict-then-act'' operator}

In \citet{akata2025playing}, SCoT is implemented by \emph{prompt-chaining} in each
round of a repeated game:
\begin{enumerate}[leftmargin=*]
    \item \textit{Prediction prompt (belief elicitation).} Given the public history $h^t$, the model is asked
    to predict the opponent's next move (or, more generally, to describe what the other player will do next).
    \item \textit{Action prompt (best response to the elicited belief).} The model is then asked to choose its
    action given the predicted opponent move, typically phrased as ``given your prediction, what is best for you to do now?''
\end{enumerate}
This ``separate belief report, then act'' structure forces an explicit theory-of-mind step before action
selection, and empirically improves coordination in some repeated games.

\subsection{Mapping SCoT as a special case of PS-BR}

Fix agent $i$ at history $h^t$. Let $A_{-i}$ denote the opponents' joint action space, and define the
agent's \emph{posterior predictive} over opponents' next action as
\[
q_i^t(\cdot \mid h^t) \in \Delta(A_{-i}).
\]
In our paper's belief language, $q_i^t(\cdot\mid h^t)$ is the one-step marginal induced by the agent's
posterior predictive continuation belief $f_{-i}^{i,t}|_{h^t}$.

SCoT can then be expressed as the following generic operator:
\begin{enumerate}[leftmargin=*]
    \item \textit{Inference:} produce $\tilde a_{-i}^t$ as an imputation of the missing opponents' next action.
    Operationally, this is obtained by querying the model with the prediction prompt.
    \item \textit{Optimize given the imputation:} choose $a_i^t$ as an (approximate) best response to the
    imputed $\tilde a_{-i}^t$ (and the known payoffs), e.g.
    \[
    a_i^t \in \arg\max_{a_i \in A_i} u_i(a_i,\tilde a_{-i}^t)
    \quad\text{(myopic)}.
    \]
    More generally, one may replace $u_i$ by the continuation objective, i.e., choose $a_i^t$ (or a continuation strategy)
    that maximizes the discounted value conditional on $\tilde a_{-i}^t$ and the induced continuation play.
\end{enumerate}

Two special cases are worth separating because they clarify the relationship to PS-BR.

\paragraph{(i) Deterministic SCoT = point estimation.}
In the implementation studied by \citet{akata2025playing}, the model is often run in
a near-deterministic regime (e.g., decoding choices consistent with temperature $\approx 0$), so the prediction
step behaves like a point estimate (roughly ``MAP'' under the model's implicit predictive distribution). In this
view, SCoT is an inference-and-optimize heuristic that can still improve play by making the model's implicit prediction problem explicit.

\paragraph{(ii) Myopic PS-BR = sampling-based estimation.}
If instead the prediction prompt is decoded stochastically (e.g., sampling at nonzero temperature),
then $\tilde a_{-i}^t$ becomes a draw from the model's own predictive distribution:
\[
\tilde a_{-i}^t \sim q_i^t(\cdot\mid h^t).
\]

\section{Prompts}
\label{app:prompts}

\subsection{\textsc{Base} prompts}
\label{app:prompts_base}

In \textsc{Base}, each player's round-$t$ prompt is the direct action query augmented with the same strategy-label context used elsewhere:
\[
\text{rules text} + \text{compact history} + \text{``You are currently playing round }t\text{''} + \text{action query} + \text{strategy-label context}.
\]
The compact history prefix used in code is:
\begin{quote}
\small
\begin{verbatim}
Observed action history (your action, opponent action):
Round 1: <self_1>, <opp_1>
...
Round t-1: <self_{t-1}>, <opp_{t-1}>
\end{verbatim}
\end{quote}

\paragraph{Round-level action query templates (Base).}
\begin{itemize}[leftmargin=*]
    \item BoS:
\begin{quote}
\small
\begin{verbatim}
Q: Which Option do you choose,  J or  F?
A:
\end{verbatim}
\end{quote}
    \item PD (order randomized each round):
\begin{quote}
\small
\begin{verbatim}
Q: Which action do you choose, J or F?
A:
\end{verbatim}
\end{quote}
    \item Harmony:
\begin{quote}
\small
\begin{verbatim}
Q: Which action do you choose, C or D?
A:
\end{verbatim}
\end{quote}
    \item Promo:
\begin{quote}
\small
\begin{verbatim}
Q: Which action do you choose, R, P, or Z?
A:
\end{verbatim}
\end{quote}
    \item Samaritan (Helper prompt):
\begin{quote}
\small
\begin{verbatim}
Q: Which action do you choose, H or N?
A:
\end{verbatim}
\end{quote}
    \item Samaritan (Recipient prompt):
\begin{quote}
\small
\begin{verbatim}
Q: Which action do you choose, W or S?
A:
\end{verbatim}
\end{quote}
    \item Lemons (Seller prompt):
\begin{quote}
\small
\begin{verbatim}
Q: Which action do you choose, HQ or LQ?
A:
\end{verbatim}
\end{quote}
    \item Lemons (Buyer prompt):
\begin{quote}
\small
\begin{verbatim}
Q: Which action do you choose,  B or  D?
A:
\end{verbatim}
\end{quote}
\end{itemize}

Before the final ``A:'' token, code injects the following strategy-context block (same helper used in \textsc{Base} and the \textsc{SCoT} prediction prompt):
\begin{quote}
\small
\begin{verbatim}
In repeated <GameName>, a strategy maps prior history to a player's next action
(possibly probabilistically).
Allowed strategies:
- <label_1>: <short description>
- ...

Role mapping in this prompt:
- Player A is the other player.
- Player B is you.
Observed rounds so far: <t-1>.
Context: full history prefix up to round <t-1>.
Strongly expect Player A to play with strategy '<prior_label>'.   [if available]
Allowed action tokens: <tokens>.                                  [if available]
Output rule: do NOT output scores, reasoning, or ranking.
Respond with exactly one action only.
\end{verbatim}
\end{quote}

\subsection{\textsc{SCoT} prompts}
\label{app:prompts_scot}

\textsc{SCoT} uses two prompts per player per round. The Stage-1 prediction prompt receives the same strategy-label context block as \textsc{Base}; the Stage-2 action prompt then conditions on the elicited prediction.

\paragraph{Stage 1 (prediction prompt).}
The prediction queries are:
\begin{itemize}[leftmargin=*]
    \item BoS:
\begin{quote}
\small
\begin{verbatim}
Q: Which action do you predict the other player will choose, J or F?
A:
\end{verbatim}
\end{quote}
    \item PD (order randomized each round):
\begin{quote}
\small
\begin{verbatim}
Q: Which action do you predict the other player will choose, J or F?
A:
\end{verbatim}
\end{quote}
    \item Harmony:
\begin{quote}
\small
\begin{verbatim}
Q: Which action do you predict the other player will choose, C or D?
A:
\end{verbatim}
\end{quote}
    \item Promo:
\begin{quote}
\small
\begin{verbatim}
Q: Which action do you predict the other player will choose, R, P, or Z?
A:
\end{verbatim}
\end{quote}
    \item Samaritan (Helper predicts Recipient):
\begin{quote}
\small
\begin{verbatim}
Q: Which action do you predict the other player will choose, W or S?
A:
\end{verbatim}
\end{quote}
    \item Samaritan (Recipient predicts Helper):
\begin{quote}
\small
\begin{verbatim}
Q: Which action do you predict the other player will choose, action H or action N?
A:
\end{verbatim}
\end{quote}
    \item Lemons (Seller predicts Buyer):
\begin{quote}
\small
\begin{verbatim}
Q: Which Option do you predict the other player will choose, Option B or Option D?
A:
\end{verbatim}
\end{quote}
    \item Lemons (Buyer predicts Seller):
\begin{quote}
\small
\begin{verbatim}
Q: Which Option do you predict the other player will choose, Option HQ or Option LQ?
A:
\end{verbatim}
\end{quote}
\end{itemize}

As implemented, the Stage-1 prediction prompt is enriched with the same strategy-context block shown above.

\paragraph{Stage 2 (action prompt conditioned on Stage-1 prediction).}
After receiving prediction \texttt{<PRED>}, code uses:
\begin{itemize}[leftmargin=*]
    \item BoS:
\begin{quote}
\small
\begin{verbatim}
Q: Given that you think the other player will choose Option <PRED> in round <t>,
imagine the outcome for both of your possible actions (Option J and Option F),
compare which gives you a better result, and then choose.
Which Option do you think is the best to choose for you in this round, Option J or Option F?
Output only one letter: J or F.
A:
\end{verbatim}
\end{quote}
    \item PD (with randomized \texttt{<opt1>, <opt2>}):
\begin{quote}
\small
\begin{verbatim}
Q: Given that you think the other player will choose Option <PRED> in round <t>,
imagine the outcome for both of your possible actions (Option <opt1> and Option <opt2>),
compare which gives you a better result, and then choose.
Which Option do you think is the best to choose for you in this round, Option <opt1> or Option <opt2>?
Output only one letter: J or F.
A:
\end{verbatim}
\end{quote}
    \item Harmony:
\begin{quote}
\small
\begin{verbatim}
Q: Given that you think the other player will choose <PRED> in round <t>,
imagine the outcome for both of your possible actions (C and D),
compare which gives you a better result, and then choose.
Which action do you think is best for you in this round, C or D?
Output only one action: C or D.
A:
\end{verbatim}
\end{quote}
    \item Promo:
\begin{quote}
\small
\begin{verbatim}
Q: Given that you think the other player will choose <PRED> in round <t>,
imagine the outcome for your possible actions (R, P, and Z),
compare which gives you a better result, and then choose.
Which action do you think is best for you in this round, R, P, or Z?
Output only one action: R, P, or Z.
A:
\end{verbatim}
\end{quote}
    \item Samaritan (Helper):
\begin{quote}
\small
\begin{verbatim}
Q: Given that you think the other player will choose Option <PRED> in round <t>,
imagine the outcome for both of your possible actions (Option H and Option N),
compare which gives you a better result, and then choose.
Which Option do you think is best to choose for you in this round, Option H or Option N?
Output only one letter: H or N.
A:
\end{verbatim}
\end{quote}
    \item Samaritan (Recipient):
\begin{quote}
\small
\begin{verbatim}
Q: Given that you think the other player will choose Option <PRED> in round <t>,
imagine the outcome for both of your possible actions (Option W and Option S),
compare which gives you a better result, and then choose.
Which Option do you think is best to choose for you in this round, Option W or Option S?
Output only one letter: W or S.
A:
\end{verbatim}
\end{quote}
    \item Lemons (Seller):
\begin{quote}
\small
\begin{verbatim}
Q: Given that you think the other player will choose Option <PRED> in round <t>,
imagine the outcome for both of your possible actions (Option HQ and Option LQ),
compare which gives you a better result, and then choose.
Which Option do you think is best to choose for you in this round, Option HQ or Option LQ?
Output only one letter: HQ or LQ.
A:
\end{verbatim}
\end{quote}
    \item Lemons (Buyer):
\begin{quote}
\small
\begin{verbatim}
Q: Given that you think the other player will choose Option <PRED> in round <t>,
imagine the outcome for both of your possible actions (Option B and Option D),
compare which gives you a better result, and then choose.
Which Option do you think is best to choose for you in this round, Option B or Option D?
Output only one letter: B or D.
A:
\end{verbatim}
\end{quote}
\end{itemize}

\subsection{\textsc{PS-BR} prompts for known deterministic payoffs}
\label{app:prompts_psbr_known}

\textsc{PS-BR} does not query the LLM for direct action choice. Actions are produced by rollout-based strategy evaluation after sampling one opponent strategy per round. The prompt-facing LLM call is for \emph{opponent strategy-label inference} in \texttt{llm-label} mode.

\paragraph{Opponent strategy inference prompt (\texttt{llm-label}).}
At round $t$, for player $i$, history is rewritten to opponent view
\[
\tilde h_{-i}^t=((a_{-i}^1,a_i^1),\ldots,(a_{-i}^{t-1},a_i^{t-1})),
\]
so tuples are \texttt{(Player A action, Player B action)} with:
\begin{itemize}[leftmargin=*]
    \item Player A = opponent whose strategy is inferred.
    \item Player B = current decision-maker.
\end{itemize}
The prompt template is:
\begin{quote}
\small
\begin{verbatim}
You are inferring Player A's strategy (the opponent) in repeated <GameName>.
In a repeated-game setting, a strategy is a rule that maps prior history to the
player's next action (possibly probabilistically).
<rules_text>
Observed rounds so far: <t-1>.

Allowed labels:
- <label_1>: <description_1>
- ...

Observed action history tuple format: (Player A action, Player B action).
Player A is the opponent whose strategy label you must infer.
Player B is you (the decision-maker).
Context: full history prefix up to round <...>.             
Target: observed Player A action at round <...>.         
Choose the allowed label that makes this observed Player A target most compatible
with the context.
At round <...>, use this mapping:
Context history as (Player A, Player B), rounds <...>:
round <k>: Player A=<...>, Player B=<...>
Observed target Player A action at round <...>: <...>
Strongly expect Player A to play with strategy '<prior_label>'.  
Player A's strategy may have changed over time, so weigh recent rounds more heavily
than earlier rounds.           
Output rule: do NOT output scores, reasoning, or ranking.
Respond with exactly one label only.

**Output only the label.**
\end{verbatim}
\end{quote}

\paragraph{Likelihood mode (no prompt).}
If \texttt{--strategy-inference likelihood} is used, no LLM prompt is issued for strategy inference; the label is sampled from a hand-coded likelihood over the finite menu.

\subsection{\textsc{PS-BR} prompts for unknown stochastic payoffs}
\label{app:prompts_psbr_unknown}

Under the theorem-aligned implementation used for Experiment~3, \textsc{PS-BR} under
unknown stochastic payoffs still samples both an opponent strategy hypothesis and a
payoff hypothesis at each round before rollout-based strategy evaluation. The
opponent-strategy side is handled exactly as in the known deterministic-payoff case.
The payoff side is not open-ended JSON inference. Instead, Experiment~3 uses the
known-common-noise / unknown-mean construction from
Section~\ref{sec:unknown_payoffs}: player $i$
maintains a posterior over a finite menu $\mathcal M_{i,g}$ of candidate mean payoff
matrices under the Gaussian noise family with known variance $\sigma_g^2$.

\paragraph{Opponent strategy inference prompt (\texttt{llm-label}).}
The opponent strategy is inferred from the joint action history, exactly as in the known deterministic payoffs case. The prompt template remains identical to the one detailed in the previous subsection.

\paragraph{Finite-menu Gaussian payoff posterior (experiment configuration).}
At round $t$, player $i$ updates
\[
\pi_i^t(m\mid h^t,r_i^{1:t-1})
\propto
\pi_i^0(m)\prod_{s=1}^{t-1}\phi(r_i^s;m(a^s),\sigma_g^2),
\qquad
m\in\mathcal M_{i,g},
\]
where $\phi(\cdot;\mu,\sigma_g^2)$ is the Gaussian density and
$r_i^s\mid a^s\sim\mathcal N(m(a^s),\sigma_g^2)$ under candidate mean matrix $m$.
The implementation then samples one matrix label $\tilde m_i\sim \pi_i^t$ and evaluates
continuation strategies against the induced payoff kernel
\[
q_i^{\tilde m_i}(\cdot\mid a)=\mathcal N(\tilde m_i(a),\sigma_g^2).
\]

\paragraph{Product structure of the menu.}
Although the theorem-level menu $\mathcal M_{i,g}$ is finite but large, it has product
form over joint actions. With a product prior over the offsets $(k_a)_{a\in A}$ and the
Gaussian likelihood above, the posterior factorizes by joint action. Operationally, the
implementation therefore updates the discrete posterior for each action-specific offset
$k_a\in K$ separately and samples a full mean matrix by drawing one offset for each
joint action. This is exactly equivalent to sampling from the full finite menu, without
explicitly enumerating all of its elements.

\paragraph{Likelihood mode (experiment configuration).}
In the reported Experiment~3 runs, \texttt{--payoff-inference likelihood} is used. No
LLM prompt is issued for payoff inference; the sampled mean-matrix label is drawn from
the Gaussian posterior above. Opponent strategy inference is handled either by the
\texttt{llm-label} prompt described above or by the corresponding likelihood mode,
depending on the strategy-inference setting.

\paragraph{Heuristic prompt mode.}
An open-ended \texttt{json} payoff-table prompt can still be used as a heuristic
variant, but it is not the theorem-aligned implementation analyzed in
Section~\ref{sec:unknown_payoffs} and instantiated in Experiment~3.

\section{Computation of minimum continuation probabilities}\label{sec:deltaCalculation}

In Section \ref{sec:experiments}, our repeated treatments follow the standard random-termination implementation to model infinitely repeated game behavior: after each round, the supergame continues with probability \(\delta_g\) and otherwise terminates. In the experimental design, \(\delta_g\) is chosen to exceed the analytical support thresholds used in this appendix, while the matched finite-horizon control is chosen to match the expected supergame length.

For each benchmark game \(g\), define
\[
\overline{\delta}_g
:=
\inf \Bigl\{
\delta \in [0,1) :
\text{the benchmark cooperative equilibrium in game } g
\text{ is sequentially rational}
\Bigr\}.
\]
We now compute \(\overline{\delta}_g\) for the five benchmark environments in Section~7.1. For BoS, Promo, and Samaritan, the benchmark cooperative equilibria are exactly those specified there. For PD and Lemons, however, the empirical engineering code now uses contrite finite-punishment variants; the closed-form threshold calculations below retain analytically simpler hard-trigger surrogates for those two games. Thus the benchmark cooperative equilibria used in this appendix are: BoS coordination or turn-taking, the hard-trigger surrogate in PD, alternating promotions with finite punishment in Promo, the \((H,W)\) path with punishment in Samaritan, and the hard-trigger surrogate in Lemons.

\begin{proposition}
\label{prop:delta-thresholds}
For the benchmark cooperative equilibria in the five environments,
\[
\overline{\delta}_{\mathrm{BoS}} = 0,\qquad
\overline{\delta}_{\mathrm{PD}} = \frac{2}{5},\qquad
\overline{\delta}_{\mathrm{Promo}} = \sqrt{10}-3,\qquad
\overline{\delta}_{\mathrm{Samaritan}} = \frac{1}{2},\qquad
\overline{\delta}_{\mathrm{Lemons}} = \frac{1}{4}.
\]
Numerically,
\[
\overline{\delta}_{\mathrm{BoS}} = 0,\qquad
\overline{\delta}_{\mathrm{PD}} = 0.4,\qquad
\overline{\delta}_{\mathrm{Promo}} \approx 0.1623,\qquad
\overline{\delta}_{\mathrm{Samaritan}} = 0.5,\qquad
\overline{\delta}_{\mathrm{Lemons}} = 0.25.
\]
\end{proposition}

\begin{proof}
We check the one-shot deviation constraints for each benchmark cooperative equilibrium.

\paragraph{BoS.}
The stage payoffs are
\[
(J,J)=(10,7),\qquad (F,F)=(7,10),\qquad (J,F)=(F,J)=(0,0).
\]
In the benchmark cooperative paths, whether players stick to \((J,J)\), stick to \((F,F)\), or alternate between them, each prescribed on-path action is already a stage-game best response to the opponent’s prescribed action. A unilateral deviation from a coordinated outcome changes the deviator’s current payoff from \(10\) or \(7\) to \(0\). Hence no continuation incentive is needed. Therefore
\[
\overline{\delta}_{\mathrm{BoS}}=0.
\]
This is a degenerate case: BoS belongs in the benchmark set because it is one of the five environments in Section~7.1, but its cooperative paths are not supported by the shadow of the future in the same sense as the other four games. 

\paragraph{PD.}
For the analytical hard-trigger surrogate of the PD benchmark, the cooperative equilibrium is grim-trigger cooperation:
\[
(J,J)\ \text{every period until a deviation, then } (F,F)\ \text{forever}.
\]
The stage payoffs are
\[
(J,J)=(3,3),\qquad (F,J)=(5,-5),\qquad (J,F)=(-5,5),\qquad (F,F)=(0,0).
\]
If a player cooperates on path forever, the continuation value is
\[
3+3\delta+3\delta^2+\cdots=\frac{3}{1-\delta}.
\]
If she deviates once from \(J\) to \(F\), she gets \(5\) today and then \(0\) forever. Thus cooperation is sequentially rational iff
\[
\frac{3}{1-\delta}\ge 5,
\]
which gives
\[
\delta\ge \frac{2}{5}.
\]
Hence
\[
\overline{\delta}_{\mathrm{PD}}=\frac{2}{5}.
\]

\paragraph{Promo.}
The benchmark cooperative equilibrium alternates
\[
(P,R)\ \text{in odd periods},\qquad (R,P)\ \text{in even periods},
\]
and after a deviation the players play \(Z\) for two periods and then return to the alternating path. 

Let \(V_P\) be the continuation value for a player in a period in which she is assigned \(P\), and let \(V_R\) be the continuation value for a player in a period in which she is assigned \(R\). Using the stage payoffs
\[
(P,R)=4,\qquad (R,P)=-1,
\]
we have
\[
V_P=4+\delta V_R,\qquad
V_R=-1+\delta V_P.
\]
Solving,
\[
V_P=\frac{4-\delta}{1-\delta^2},
\qquad
V_R=\frac{4\delta-1}{1-\delta^2}.
\]

The only binding deviation is by the player assigned \(R\). If she obeys, she receives \(V_R\). If instead she deviates to \(P\), then against the opponent’s prescribed \(P\) she gets \(0\) today, then \(-2\) in each of the next two punishment periods, and then returns to the alternating path at a \(P\)-phase. So the deviation payoff is
\[
0-2\delta-2\delta^2+\delta^3 V_P.
\]
Thus sequential rationality requires
\[
V_R\ge -2\delta-2\delta^2+\delta^3V_P.
\]
Substituting the formulas for \(V_R\) and \(V_P\),
\[
\frac{4\delta-1}{1-\delta^2}
\ge
-2\delta-2\delta^2+\delta^3\frac{4-\delta}{1-\delta^2}.
\]
Multiplying by \(1-\delta^2>0\) and simplifying yields
\[
\delta^4+6\delta^3-2\delta^2-6\delta+1\le 0.
\]
Since
\[
\delta^4+6\delta^3-2\delta^2-6\delta+1
=
(\delta^2-1)(\delta^2+6\delta-1),
\]
and \(1-\delta^2>0\) for \(\delta\in[0,1)\), this is equivalent to
\[
\delta^2+6\delta-1\ge 0.
\]
The unique root in \([0,1)\) is
\[
\delta=\sqrt{10}-3.
\]
Therefore
\[
\overline{\delta}_{\mathrm{Promo}}=\sqrt{10}-3.
\]

\paragraph{Samaritan.}
The benchmark cooperative equilibrium is:
\[
(H,W)\ \text{every period};
\]
if the recipient ever shirks, switch forever to \((N,W)\); if, during punishment, the helper deviates by helping, switch forever to \((H,S)\).

The stage payoffs are
\[
(H,W)=(2,-1),\qquad (H,S)=(0,0),\qquad (N,W)=(1,-2),\qquad (N,S)=(-1,-3).
\]

First, the recipient must prefer \(W\) on the cooperative path. If she obeys forever, her value is
\[
\frac{-1}{1-\delta}.
\]
If she deviates to \(S\), she gets \(0\) today and then \((N,W)\) forever, so her deviation value is
\[
\delta\frac{-2}{1-\delta}.
\]
Thus
\[
\frac{-1}{1-\delta}\ge \delta\frac{-2}{1-\delta},
\]
which is equivalent to
\[
\delta\ge \frac{1}{2}.
\]

Second, along the punishment path \((N,W)\), the helper must prefer not to deviate back to \(H\). If she obeys punishment forever, her value is
\[
\frac{1}{1-\delta}.
\]
If she deviates once to \(H\), she gets \(2\) today and then triggers \((H,S)\) forever, which gives \(0\) thereafter. Thus
\[
\frac{1}{1-\delta}\ge 2,
\]
equivalently,
\[
\delta\ge \frac{1}{2}.
\]

Hence the binding threshold is
\[
\overline{\delta}_{\mathrm{Samaritan}}=\frac{1}{2}.
\]

\paragraph{Lemons.}
For the analytical hard-trigger surrogate of the Lemons benchmark, the cooperative equilibrium is:
\[
(HQ,B)\ \text{every period until a low-quality sale occurs,}
\]
after which the buyer switches forever to \(D\) and the seller then plays dominant \(LQ\).

The stage payoffs are
\[
(HQ,B)=(3,3),\qquad (LQ,B)=(4,-1),\qquad (HQ,D)=(-1,0),\qquad (LQ,D)=(0,0).
\]
The binding deviation is the seller’s temptation to replace \(HQ\) by \(LQ\) when the buyer is expected to buy. If the seller cooperates forever, her value is
\[
\frac{3}{1-\delta}.
\]
If she deviates once to \(LQ\), she gets \(4\) today and then the buyer boycotts forever, yielding \(0\) thereafter. Thus
\[
\frac{3}{1-\delta}\ge 4,
\]
which gives
\[
\delta\ge \frac{1}{4}.
\]
Therefore
\[
\overline{\delta}_{\mathrm{Lemons}}=\frac{1}{4}.
\]

Collecting the five thresholds proves the result.
\end{proof}

Table~\ref{tab:delta-thresholds} summarizes the thresholds. For the matched finite-horizon controls, a natural choice is
\[
H_g=\left\lceil \frac{1}{1-\overline{\delta}_g}\right\rceil,
\]
which matches the expected supergame length up to integer rounding, in the spirit of
\citet{bo2005cooperation}'s random-termination versus matched-finite comparison.

\begin{table}[t]
\centering
\caption{Minimum continuation probabilities for supporting the benchmark cooperative equilibria}
\label{tab:delta-thresholds}
\begin{tabular}{lcc}
\toprule
Game \(g\) & \(\overline{\delta}_g\) & \(H_g=\left\lceil \frac{1}{1-\overline{\delta}_g}\right\rceil\) \\
\midrule
BoS & \(0\) & \(1\) \\
PD & \(\frac{2}{5}\) & \(2\) \\
Promo & \(\sqrt{10}-3 \approx 0.1623\) & \(2\) \\
Samaritan & \(\frac{1}{2}\) & \(2\) \\
Lemons & \(\frac{1}{4}\) & \(2\) \\
\bottomrule
\end{tabular}
\end{table}

Two interpretive remarks are useful here. First, among the five benchmark games, only BoS has \(\overline{\delta}_g=0\), because its benchmark cooperative outcomes are already stage-game best responses. Second, the other four thresholds are genuine continuation-value thresholds: they are the minimum shadow-of-the-future levels needed to make the benchmark cooperative repeated-game path sequentially rational.

\section{Game-specific strategy menus}\label{appsec:game_menus}

For the analysis under the retained-menu identification assumption, we use sparse deterministic menus. Each menu contains one benchmark cooperative-equilibrium automaton together with seven deterministic constant or periodic heuristics that are not on-path observationally equivalent to it. Because the retained labels are deterministic, once a wrong label prescribes a different action from the benchmark at a reached public history, the realized action assigns that label zero likelihood from that date onward. Accordingly, we exclude labels such as \texttt{always\_cooperate}, \texttt{always\_help}, \texttt{always\_work}, \texttt{always\_buy}, \texttt{always\_hq}, and any alternative trigger label that shares the same cooperative path and punishment regime as the benchmark equilibrium label. The richer empirical rollout menu used for engineering experiments can remain broader; the menus below are the sparse appendix menus used for the identification argument. For the \texttt{alt\_*} and \texttt{*\_cycle} labels below, the phase convention is treated as part of the public automaton state.

\begin{proposition}[Verification for the sparse simulation menus]
\label{prop:simulation_menus_verify_equiv_conc}
For each game and player role, the sparse public-action menu specified in
this appendix satisfies Assumption~\ref{ass:det_menu_sep_app} for the benchmark equilibrium
automaton used in the corresponding simulation analysis. Consequently, by
Corollaries~\ref{cor:det_menu_sep_implies_equiv_class_conc} and
\ref{cor:det_menu_sep_implies_equiv_class_conc_private}, Assumptions~\ref{ass:equiv_class_conc}
and \ref{ass:equiv_class_conc_private} hold for those simulation menus.
\end{proposition}

\begin{proof}
Each listed menu is deterministic by construction, so the deterministic-menu part of
Assumption~\ref{ass:det_menu_sep_app} holds immediately. Menu grain of truth also holds by
construction, since the relevant benchmark equilibrium automaton is explicitly included in the
retained menu for each game and role.

It remains to verify on-path finite separation. Fix a game, a player role, and the benchmark
profile $f$ for the corresponding simulation instance, and let $f_{-i}$ denote the benchmark
opponent automaton in the retained menu. Along the realized path under $f$, every wrong retained
label $g_{-i}$ differs from $f_{-i}$ after finitely many public histories for an explicit
structural reason:
\begin{itemize}[leftmargin=*]
    \item In BoS, the benchmark is one of the eight listed deterministic periodic words, and every
    other retained label is a different periodic word. Hence there is a smallest round at which the
    two words prescribe different actions.
    \item In PD, the benchmark is \texttt{grim\_eq}, which prescribes $J$ at every reached history
    on the realized cooperative path. Each wrong retained label prescribes $F$ by round $1$, $2$,
    $3$, or $4$: \texttt{all\_F} at round $1$, \texttt{alt\_JF} and \texttt{alt\_FJ} by round $2$,
    \texttt{JJF\_cycle} and \texttt{JFF\_cycle} by round $3$, and \texttt{JJJF\_cycle} and
    \texttt{JFFF\_cycle} by round $4$.
    \item In Promo, the benchmark is one of the two alternating cartel labels. The other cartel
    label starts in the opposite phase and therefore differs at round $1$. The remaining six labels
    are constant or period-$3$ rules; each either starts with the wrong action or repeats an action
    (or uses $Z$) within the first three rounds, whereas the benchmark alternates between $P$ and
    $R$ without using $Z$ on the realized path.
    \item In the Samaritan helper menu, \texttt{helper\_eq} prescribes $H$ at every reached history
    on the realized cooperative path, while every wrong label prescribes $N$ by round $1$, $2$,
    $3$, or $4$. In the recipient menu, \texttt{recipient\_eq} prescribes $W$ at every reached
    history on the realized cooperative path, while every wrong label prescribes $S$ by round $1$,
    $2$, $3$, or $4$.
    \item In the Lemons seller menu, \texttt{seller\_eq} prescribes $HQ$ at every reached history
    on the realized cooperative path, while every wrong label prescribes $LQ$ by round $1$, $2$,
    $3$, or $4$. In the buyer menu, \texttt{buyer\_eq} prescribes $B$ at every reached history on
    the realized cooperative path, while every wrong label prescribes $D$ by round $1$, $2$, $3$,
    or $4$.
\end{itemize}
Let $t_g$ be the first such disagreement round. Because the menus are deterministic, the realized
benchmark action at $t_g$ has probability one under $f_{-i}$ and probability zero under $g_{-i}$.
Therefore $g_{-i}(h^{t_g}(z))\neq f_{-i}(h^{t_g}(z))$ on the realized path, so $g_{-i}$ is
refuted in finite time. This proves Assumption~\ref{ass:det_menu_sep_app}(2).
\end{proof}

\paragraph{(1) BoS menu.}
Use the same 8-label menu for both players.
\begin{itemize}[leftmargin=*]
    \item \texttt{coord\_J}: play $J$ every round.
    \item \texttt{coord\_F}: play $F$ every round.
    \item \texttt{alt\_JF}: play $J,F,J,F,\ldots$.
    \item \texttt{alt\_FJ}: play $F,J,F,J,\ldots$.
    \item \texttt{JJF\_cycle}: repeat $J,J,F$.
    \item \texttt{JFF\_cycle}: repeat $J,F,F$.
    \item \texttt{JJFF\_cycle}: repeat $J,J,F,F$.
    \item \texttt{JJJF\_cycle}: repeat $J,J,J,F$.
\end{itemize}
All eight BoS labels are deterministic and pairwise distinct as action words, so any wrong retained label is contradicted in finite time.

\paragraph{(2) PD menu.}
Use the same 8-label menu for both players.
\begin{itemize}[leftmargin=*]
    \item \texttt{grim\_eq}: play $J$ until the opponent first plays $F$; thereafter play $F$ forever.
    \item \texttt{all\_F}: play $F$ every round.
    \item \texttt{alt\_JF}: play $J,F,J,F,\ldots$.
    \item \texttt{alt\_FJ}: play $F,J,F,J,\ldots$.
    \item \texttt{JJF\_cycle}: repeat $J,J,F$.
    \item \texttt{JFF\_cycle}: repeat $J,F,F$.
    \item \texttt{JJJF\_cycle}: repeat $J,J,J,F$.
    \item \texttt{JFFF\_cycle}: repeat $J,F,F,F$.
\end{itemize}
There is only one prefix-then-absorbing trigger label, \texttt{grim\_eq}. Every other retained label is a constant or periodic word, so it cannot coincide forever with a realized grim path.

\paragraph{(3) Promo menu (actions: $R$ = regular, $P$ = promotion, $Z$ = punishment/price war).}
Use the same 8-label menu for both players.
\begin{itemize}[leftmargin=*]
    \item \texttt{cartel\_phase0}: cooperative path $P,R,P,R,\ldots$; after any deviation from the prescribed phase path, play $Z,Z$ for two periods, then return to phase 0.
    \item \texttt{cartel\_phase1}: cooperative path $R,P,R,P,\ldots$; after any deviation from the prescribed phase path, play $Z,Z$ for two periods, then return to phase 1.
    \item \texttt{all\_R}: play $R$ every round.
    \item \texttt{all\_P}: play $P$ every round.
    \item \texttt{all\_Z}: play $Z$ every round.
    \item \texttt{RRP\_cycle}: repeat $R,R,P$.
    \item \texttt{PPR\_cycle}: repeat $P,P,R$.
    \item \texttt{RRZ\_cycle}: repeat $R,R,Z$.
\end{itemize}
The two cooperative cartel labels are distinct at round 1, and none of the six alternatives shares their alternating path or their two-period punishment block.

\paragraph{(4) Samaritan menus (Helper actions: $H$ = Help, $N$ = No-help; Recipient actions: $W$ = Work, $S$ = Shirk).}
The game is asymmetric, so we use separate menus.

\textit{Helper menu.}
\begin{itemize}[leftmargin=*]
    \item \texttt{helper\_eq}: play $H$ until the recipient first plays $S$; thereafter play $N$ forever.
    \item \texttt{all\_N}: play $N$ every round.
    \item \texttt{alt\_HN}: play $H,N,H,N,\ldots$.
    \item \texttt{alt\_NH}: play $N,H,N,H,\ldots$.
    \item \texttt{HHN\_cycle}: repeat $H,H,N$.
    \item \texttt{HNN\_cycle}: repeat $H,N,N$.
    \item \texttt{HHHN\_cycle}: repeat $H,H,H,N$.
    \item \texttt{HNNN\_cycle}: repeat $H,N,N,N$.
\end{itemize}

\textit{Recipient menu.}
\begin{itemize}[leftmargin=*]
    \item \texttt{recipient\_eq}: start with $W$; continue $W$ through the cooperative phase and through punishment. If punishment has already begun and the helper later plays $H$ again, switch to $S$ forever.
    \item \texttt{all\_S}: play $S$ every round.
    \item \texttt{alt\_WS}: play $W,S,W,S,\ldots$.
    \item \texttt{alt\_SW}: play $S,W,S,W,\ldots$.
    \item \texttt{WWS\_cycle}: repeat $W,W,S$.
    \item \texttt{WSS\_cycle}: repeat $W,S,S$.
    \item \texttt{WWWS\_cycle}: repeat $W,W,W,S$.
    \item \texttt{WSSS\_cycle}: repeat $W,S,S,S$.
\end{itemize}
For each role, there is only one trigger-style equilibrium label. The other seven labels are constants or cycles, so they cannot remain observationally equivalent to that equilibrium label on any realized path.

\paragraph{(5) Lemons menus (Seller actions: $HQ$ = High-quality, $LQ$ = Low-quality; Buyer actions: $B$ = Buy, $D$ = Don't buy).}
This game is also asymmetric, so we use separate menus.

\textit{Seller menu.}
\begin{itemize}[leftmargin=*]
    \item \texttt{seller\_eq}: play $HQ$ until the buyer first plays $D$; thereafter play $LQ$ forever.
    \item \texttt{all\_LQ}: play $LQ$ every round.
    \item \texttt{alt\_HQ\_LQ}: play $HQ,LQ,HQ,LQ,\ldots$.
    \item \texttt{alt\_LQ\_HQ}: play $LQ,HQ,LQ,HQ,\ldots$.
    \item \texttt{HHQ\_LQ\_cycle}: repeat $HQ,HQ,LQ$.
    \item \texttt{HQ\_LQ\_LQ\_cycle}: repeat $HQ,LQ,LQ$.
    \item \texttt{HHHQ\_LQ\_cycle}: repeat $HQ,HQ,HQ,LQ$.
    \item \texttt{HQ\_LQ\_LQ\_LQ\_cycle}: repeat $HQ,LQ,LQ,LQ$.
\end{itemize}

\textit{Buyer menu.}
\begin{itemize}[leftmargin=*]
    \item \texttt{buyer\_eq}: play $B$ until the first observed $LQ$; thereafter play $D$ forever.
    \item \texttt{all\_D}: play $D$ every round.
    \item \texttt{alt\_BD}: play $B,D,B,D,\ldots$.
    \item \texttt{alt\_DB}: play $D,B,D,B,\ldots$.
    \item \texttt{BBD\_cycle}: repeat $B,B,D$.
    \item \texttt{BDD\_cycle}: repeat $B,D,D$.
    \item \texttt{BBBD\_cycle}: repeat $B,B,B,D$.
    \item \texttt{BDDD\_cycle}: repeat $B,D,D,D$.
\end{itemize}
Again, only one retained label has the reputation/boycott prefix-then-absorbing form for each role. All other retained labels are constants or cycles, so they are finitely refuted against that equilibrium label.

\section{Promo game}

\subsection{Promo game \citep{lal1990price}: alternating promotions with finite punishment}\label{appssec:promo}

Lal (1990) studies repeated price competition in a market with two identical ``national'' brands that have loyal consumers and a third ``local'' brand with little/no loyalty. The local brand disciplines prices in the switching segment, creating a tension for the national brands between (i) extracting rents from loyals via a high ``regular'' price and (ii) defending the switchers via temporary price cuts. A key result is that, even when the corresponding one-shot stage game has no Nash equilibrium, an \emph{alternating promotions} pattern -- only one national brand is on promotion in a given period and the roles alternate over time -- can arise as a \emph{pure-strategy Nash equilibrium} of the infinite-horizon discounted game, supported by a credible number of punishment periods.

To obtain a compact repeated-game benchmark, we discretize \cite{lal1990price}’s richer price-choice problem into three representative regimes per firm:
\begin{itemize}[leftmargin=*]
    \item \textit{Regular} ($R$): charge the high ``regular'' price
    \item  \textit{Promotion} ($P$): charge the low promotional price
    \item \textit{Punishment/price war} ($Z$): charge a very low price used only in punishment phases.
\end{itemize}

The resulting 3$\times$3 payoff matrix in Appendix~\ref{sec:experiments} is a reduced-form encoding of the ordinal incentive structure: a unilateral promotion against a regular-price rival yields the highest current-period gain (the ``temptation'' payoff); simultaneous promotions are less profitable than \emph{alternating} promotions; and outcomes involving $Z$ are jointly bad, standing in for the ``intense competition/price war'' phase used to deter deviations.

The canonical nontrivial Nash equilibrium is an \emph{alternating} path: play $(P,R)$ in odd rounds and $(R,P)$ in even rounds (or vice versa). After any deviation from the prescribed phase, switch to a punishment phase (e.g., $(Z,Z)$ for a fixed number of rounds) for a few periods and then return to the alternating path (as defined as \cite{abreu1988theory}), or revert permanently to a low-payoff punishment regime (grim trigger). For sufficiently patient players, the discounted loss from the punishment phase outweighs the one-shot deviation gain, making the alternating-promotions path incentive compatible.

\end{document}